\documentclass[twoside]{article}

\usepackage{stylefilearxiv}

\usepackage{hyperref}
\usepackage{url}

\RequirePackage{times}
\RequirePackage{fancyhdr}
\RequirePackage{xcolor}
\RequirePackage{algorithm}
\RequirePackage{algorithmic}
\RequirePackage{eso-pic}
\RequirePackage{forloop}
\RequirePackage{url}

\usepackage{microtype}
\usepackage{graphicx}
\usepackage{subfigure}
\usepackage{amsmath}
\usepackage{amssymb}
\usepackage{mathtools}
\usepackage{amsthm}
\usepackage{ifthen}
\usepackage{xargs}
\usepackage{amsfonts}
\mathtoolsset{showonlyrefs}
\usepackage{enumitem}
\usepackage{bbm}
\usepackage{upgreek}
\usepackage{stmaryrd}
\usepackage{xspace}
\usepackage{caption}

\allowdisplaybreaks

\newcommand{\thp}{\theta}
\newcommand{\thpone}{\theta^{(1)}}
\newcommand{\thptwo}{\theta^{(2)}}

\newcommand{\nset}{\mathbb{N}}
\newcommand{\nsetpos}{\mathbb{N}_{>0}}
\newcommand{\rset}{\mathbb{R}}
\newcommand{\rsetnonneg}{\mathbb{R}_+}
\newcommand{\rsetpos}{\mathbb{R}_+^*}
\newcommand{\1}[1]{\mathbbm{1}_{#1}}

\newcommand{\E}{\mathbb{E}}
\newcommand{\V}{\mathbb{V}}
\newcommand{\prob}{\mathbb{P}}
\newcommand{\repfunc}[1]{h_{#1}}
\newcommand{\grad}[1]{G_{#1}^\M}
\newcommand{\gradt}[1]{\bar{G}_{#1}^\M}

\newcommand{\set}[1]{\mathsf{#1}}
\newcommand{\alg}[1]{\mathcal{#1}}
\newcommand{\signmeas}[1]{\tilde{\mathsf{M}}(#1)}
\newcommand{\signmeasz}[1]{\tilde{\mathsf{M}}_0(#1)}
\newcommand{\signmeasp}[1]{\tilde{\mathsf{M}}^p(#1)}
\newcommand{\signmeaszp}[1]{\tilde{\mathsf{M}}_0^p(#1)}
\newcommand{\meas}[1]{\mathsf{M}(#1)}
\newcommand{\probmeas}[1]{\mathsf{M}_1(#1)}
\newcommand{\signmeasalg}[1]{\tilde{\mathcal{M}}(#1)}
\newcommand{\signmeaszalg}[1]{\tilde{\mathcal{M}}_0(#1)}
\newcommand{\signmeaspalg}[1]{\tilde{\mathcal{M}}^{\tensprod p}(#1)}
\newcommand{\signmeaszpalg}[1]{\tilde{\mathcal{M}}_0^{\tensprod p}(#1)}
\newcommand{\measalg}[1]{\mathcal{M}(#1)}
\newcommand{\probmeasalg}[1]{\mathcal{M}_1(#1)}

\newcommand{\parspace}{\mathsf{\Theta}}

\newcommand{\catdist}{\mathsf{cat}}

\newcommand{\bmf}[1]{\mathsf{F}(#1)}
\newcommand{\tensprod}{\varotimes}
\newcommand{\eqdef}{\vcentcolon=}
\newcommand{\ie}{\emph{i.e.}\xspace}
\newcommand{\eg}{\emph{e.g.}\xspace}

\newcommand{\colbo}{{COLBO}}
\newcommand{\iwae}{{\texttt{IWAE}}}
\newcommand{\vae}{{\texttt{VAE}}}
\newcommand{\siwae}{{\texttt{OSIWAE}}}
\newcommand{\smcsiwae}{{\texttt{SMC-OSIWAE}}}
\newcommand{\OSIWAE}{{\texttt{OSIWAE}}}

\newcommand{\RML}{\texttt{RML}}

\newcommand{\smc}{SMC}
\newcommand{\VSMC}{\texttt{VSMC}}
\newcommand{\OVSMC}{{\texttt{OVSMC}}}

\newcommand{\recf}[1]{\Phi_{#1}}
\newcommand{\rect}[1]{\Psi_{#1}}
\newcommand{\rectsmooth}[1]{\bar{\Psi}_{#1}}
\newcommand{\lip}[1]{\mathsf{Lip}(#1)}
\newcommand{\contr}{\ell}
\newcommand{\lyap}{\ell^\M}
\newcommand{\vfunc}[1]{V_{#1}^\M}
\newcommand{\likfunc}[1]{V_{#1}}
\newcommand{\gradlik}[1]{G_{#1}}

\newcommand{\norm}[1]{\left\lVert #1 \right\rVert}
\newcommand{\abs}[1]{\left\lvert #1 \right\rvert}
\newcommand{\tvnorm}[1]{\left\lVert #1 \right\rVert_{\mathsf{TV}}}
\newcommand{\tvnormsmall}[1]{\lVert #1 \rVert_{\mathsf{TV}}}

\newcommand{\hilbdist}{d}

\newcommand{\z}[2]{Z_{#1}^{#2}}
\newcommand{\epart}[2]{\xi_{#1}^{#2}}

\newcommand{\I}[2]{I_{#1}^{#2}}

\newcommand{\auxrv}{U}
\newcommand{\auxrvb}{u}
\newcommand{\tstat}[2]{\tau_{#1}^{#2}}
\newcommand{\wgt}[2]{\omega_{#1}^{#2}}
\newcommand{\wgtsum}[1]{\Omega_{#1}}
\newcommand{\wgtfunc}[1]{w_{#1}}

\newcommand{\tfunc}[2]{\varphi_{#2}^{#1}}
\newcommand{\res}{\rho^\mathsf{res}}
\newcommand{\bs}{\rho^\mathsf{bs}}
\newcommand{\mds}[2]{W_{#1}^{(#2)}}

\newcommand{\tz}[1]{\mathbf{T}_{#1}}

\newcommand{\tconst}[1]{\boldsymbol{\Upsilon}_{#1}}
\newcommand{\hidker}[1]{\mathbf{M}_{#1}}
\newcommand{\hidkertrue}{\mathbf{M}}
\newcommand{\hiddens}[1]{m_{#1}}
\newcommand{\hiddenstrue}{m}
\newcommand{\emker}[1]{\mathbf{G}_{#1}}
\newcommand{\emkertrue}{\mathbf{G}}
\newcommand{\skertrue}{\mathbf{S}}
\newcommand{\emdens}[1]{g_{#1}}
\newcommand{\emdenstrue}{g}

\newcommand{\propdens}[1]{r_{#1}}
\newcommand{\propker}[1]{\mathbf{R}_{#1}}
\newcommand{\kernel}[1]{\mathbf{#1}}
\newcommand{\lker}[2]{\mathbf{L}_{#1}\langle #2\rangle}
\newcommand{\lkert}[3]{\mathbf{L}_{#1}^{#2}\langle #3\rangle}
\newcommand{\lkertang}[2]{\tilde{\mathbf{L}}_{#1}\langle #2\rangle}
\newcommand{\gtilde}[2]{\tilde{\mathbf{F}}_{#1}^{#2}}
\newcommand{\htilde}[2]{\tilde{\mathbf{H}}_{#1}^{#2}}

\newcommand{\indmeas}{\nu}

\newcommand{\vardist}[2]{q_{#1}^{#2}}
\newcommand{\pfilt}[2]{p_{#1}^{#2}}
\newcommand{\filt}[2]{\phi_{#2}^{#1}}
\newcommand{\refm}{\lambda_{\alg{X}}}
\newcommand{\refg}{\lambda_{\alg{Y}}}
\newcommand{\tang}[2]{\psi_{#2}^{#1}}
\newcommand{\xinit}{\chi}

\newcommand{\mutilde}{\tilde{\boldsymbol{\mu}}}

\newcommand{\bq}{\tilde{b}}

\newcommand{\rhoq}{\varrho_\eq}

\newcommand{\cq}{c}
\newcommand{\cqt}{\tilde{c}}
\newcommand{\kq}{\kappa_1}
\newcommand{\kqt}{\kappa_2}
\newcommand{\eq}{\epsilon}
\newcommand{\bg}{\tilde{\kappa}_1}
\newcommand{\bgw}{\tilde{\kappa}_2}

\newcommand{\M}{M}
\newcommand{\lipz}[1]{L_{#1}}

\newcommand{\biasfunc}{\beta}
\newcommand{\crect}{c_{\rect{}}}
\newcommand{\ctang}{c_{\tang{}{}}}

\theoremstyle{plain}
\newtheorem{theorem}{Theorem}[section]
\newtheorem{proposition}[theorem]{Proposition}
\newtheorem{lemma}[theorem]{Lemma}
\newtheorem{corollary}[theorem]{Corollary}
\theoremstyle{definition}
\newtheorem{definition}[theorem]{Definition}
\newtheorem{assumption}[theorem]{Assumption}
\theoremstyle{remark}
\newtheorem{remark}[theorem]{Remark}

\usepackage[round]{natbib}

\begin{document}

\twocolumn[

\ourtitle{Recursive Learning of Asymptotic Variational Objectives}

\ourauthor{ Alessandro Mastrototaro \qquad Mathias Müller \qquad Jimmy Olsson}

\ouraddress{Department of Mathematics,
	KTH Royal Institute of Technology,
	Stockholm, Sweden
	\\\texttt{\{alemas, matmul, jimmyol\}@kth.se}}

\begin{abstract}
  General \emph{state-space models} (SSMs) are widely used in statistical machine learning and are among the most classical generative models for sequential time-series data. SSMs, comprising latent Markovian states, can be subjected to \emph{variational inference} (VI), but standard VI methods like the \emph{importance-weighted autoencoder} (IWAE) lack functionality for streaming data. To enable online VI in SSMs when the observations are received in real time, we propose maximising an IWAE-type variational lower bound on the asymptotic \emph{contrast function}, rather than the standard IWAE ELBO, using stochastic approximation. Unlike the \emph{recursive maximum likelihood }method, which directly maximises the asymptotic contrast, our approach, called online sequential IWAE (OSIWAE), allows for online learning of both model parameters and a Markovian recognition model for inferring latent states. By approximating filter state posteriors and their derivatives using sequential Monte Carlo (SMC) methods, we create a particle-based framework for online VI in SSMs. This approach is more theoretically well-founded than recently proposed \emph{online variational SMC} methods. We provide rigorous theoretical results on the learning objective and a numerical study demonstrating the method's efficiency in learning model parameters and particle proposal kernels.
\end{abstract}]

\section{Introduction}

The \emph{variational autoencoder} (\vae) \citep{kingma:welling:2014} is a foundational probabilistic method in machine learning, renowned for its capability to learn latent-variable models. Building on this framework, the \emph{importance-weighted autoencoder} (\iwae) \citep{burda:2016} was developed to enhance the performance of probabilistic inference by providing a tighter bound on the marginal likelihood compared to the standard {\vae} using a multi-sample objective. The {\vae} and the {\iwae} have found diverse applications across many fields, including, \eg, deep generative modelling, recommendation systems, biology, and image compression \citep{zhang:2020,bond:2021,gayoso2021joint,xu2022multi,daudel2023alpha,doucet:moulines:thin:2023}.
In the case where the observed data is modelled by \emph{state-space models} (SSMs), also known as general state-space \emph{hidden Markov models} (HMMs), which is the most classical probabilistic generative modelling framework for time series  \citep[see][]{cappe:moulines:ryden:2005}, \emph{variational sequential Monte Carlo} (\VSMC) \citep{maddison:2017,le:2018,naesseth:2018} can be seen as a further development of the {\iwae}, where the importance-sampling-based estimator of the likelihood used in the {\iwae} is replaced by the (still unbiased) likelihood estimator provided by \emph{sequential Monte Carlo} (\smc) \emph{methods}, also known as \emph{particle filters} \citep{doucet:defreitas:gordon:2001,chopin:papaspiliopoulos:2020}. {\VSMC} not only enables parameter estimation in SSMs, but also allows for acceleration of {\smc} by optimising the particle proposal distributions. Such adaptation, which enables reduced particle degeneracy and improved accuracy, is essential to make the {\smc} method better suited for complex and high-dimensional data-assimilation problems. An inherent, general problem of {\VSMC} is the difficulty to reparameterise the resampling operation of the particle filter and consequently to estimate the corresponding ELBO gradient with acceptable accuracy. This has prompted its creators to use an \emph{ad hoc} truncated version of the gradient, in which terms with high variance are simply excluded, leading to a bias that is difficult to control. Moreover, since the standard implementation of {\VSMC} is primarily designed for batch-processing scenarios, it is not suitable for processing streaming time-series data, which is the focus of this work. An attempt to adapt the {\VSMC} methodology to online scenarios was recently made by \cite{ovsmc}, but although this methodology has been shown to work well in some cases, it relies on a time-distributed stochastic gradient derived from its batch-oriented predecessors, and consequently suffers from the same lack of theoretical support. 

Thus, in this paper, we take a different approach and focus instead on maximising a lower bound on the \emph{asymptotic contrast function}, alternatively termed the \emph{log-likelihood rate}, given by the ergodic limit of the time-normalised log-likelihood function when the number of observations tends to infinity \citep[see, \emph{e.g.},][]{cappe:moulines:ryden:2005,tadic:doucet:2018}, in an online scenario in which the data is only made available sequentially in real time. More specifically, in the proposed algorithm, which we refer to as the \emph{online sequential importance-weighted autoencoder} (\siwae), a \emph{contrast lower bound} (\colbo), interpreted as an {\iwae}-type multi-sample variational objective, is maximised using a Robbins--Monro scheme with Markovian perturbations targeting the zeros of the {\colbo} gradient. This allows to learn, simultaneously, SSM parameters as well as a Markovian recognition model depending on the given data point and on the previous latent state. The latter can be used as an effective particle proposal or as a sequential encoder similar to the \emph{variational recurrent neural networks} introduced by \citet{chung:2015}. 
In support of {\siwae}, we present  theoretical results that establish the {\colbo} ergodic limit and furthermore provide $\mathcal{O}(M^{-1})$ bounds (where $M$ is the {\colbo} sample size) on the discrepancies between the {\colbo} and the asymptotic contrast function as well as between their gradients. 

Ideally, {\siwae} requires access to the flow of filter state posteriors and their derivatives, which are however intractable in general. By approximating these measures using particle filters, a practical version of the algorithm, referred to as {\smcsiwae}, is obtained. In particular, the filter derivative, or \emph{tangent filter}, requires the calculation of expectations of additive state functionals---a problem that has received much attention in the SMC literature \citep[see, \eg,][]{kitagawa:sato:2001,olsson:cappe:douc:moulines:2008,delmoral:doucet:singh:2010,olsson:westerborn:2017}. Hence, we propose an SMC-based {\siwae}, called {\smcsiwae}, incorporating the latest advancement in this line of research, namely the \emph{AdaSmooth} algorithm \citep{mastrototaro:olsson:alenlov:2021}, which allows for online approximation of the tangent filters with complexity and memory requirements that only grow linearly with the number of particles. Unlike {\OVSMC}, {\smcsiwae} has a solid theoretical underpinning. 
When it comes to learning the model parameters, {\smcsiwae} approaches the particle-based \emph{recursive maximum likelihood} \citep{legland:mevel:1997,poyiadjis:doucet:singh:2011,delmoral:doucet:singh:2015} with increasing $M$, but where a lower bound on the contrast function is maximised rather than the contrast function itself. On the other hand, when learning is restricted to proposal parameters only, {\smcsiwae} becomes very similar to {\OVSMC}.
In this way, {\smcsiwae} can be interpreted as a golden compromise between these two methods. Furthermore, the {\smcsiwae} updating formula sheds some light on the bias of {\OVSMC} (and thus of {\VSMC}), which can be expected to be significant when the observations are noisy relative to the latent state signal. This latter conclusion is also confirmed by our simulations.

Finally, we present numerical experiments to showcase the effectiveness of the algorithm in learning SSM parameters as well as close-to-optimal particle proposals. These experiments highlight {\siwae}'s advantages over its predecessors {\OVSMC} and particle-based {\RML}.

The paper is structured as follows. In Section~\ref{sec:background} we provide a brief overview of SSMs and introduce the concepts of the asymptotic contrast function and the {\RML} procedure. In Section~\ref{sec:method}, we introduce the general {\siwae} idea, and illustrate our theoretical results. In Section~\ref{sec:smc_algo} we describe the particle-based implementation of {\siwae}, {\smcsiwae}, and in Section~\ref{sec:experiments} we provide some some numerical illustrations of the latter. Details and proofs are found in the appendix. 

\section{Background}\label{sec:background}
\subsection{Model and notation}
An SSM is a bivariate Markov chain $(X_t,Y_t)_{t\ge 0}$ evolving on some measurable state space $(\set{X} \times \set{Y}, \alg{X}\tensprod\alg{Y})$, which is typically Euclidean and furnished with the Borel $\sigma$-field. Here the so-called \emph{state process} $(X_t)_{t \geq 0}$ is latent, or \emph{hidden}, and only partially observed through the \emph{observation process} $(Y_t)_{t \geq 0}$. For a given time horizon $T\in\nset$, the joint law of $(X_t, Y_t)_{t = 0}^T$ is   
$p_\thp(x_{0:T},y_{0:T})
	=\hiddens{0}(x_0)\emdens{\thp}(y_0\mid x_0)\prod_{t=1}^{T}\hiddens{\thp}(x_t\mid x_{t-1})\emdens{\thp}(y_t\mid x_t)$,
where $\hiddens{\thp}$ and $\emdens{\thp}$ are Markov transition densities w.r.t. some reference measures (typically Lebesgue) on $\alg{X}$ and $\alg{Y}$, respectively, and $\hiddens{0}$ is some probability density w.r.t. the same reference measure on $\alg{X}$. Here $x_{0:T} = (x_0, \ldots, x_T)$ is our generic notation for vectors and $\thp\in\parspace\subset\rset^p$,  $p\in\nsetpos$, is a parameter vector.  
Under 
this dynamics, the state process is itself Markov with transition density $\hiddens{\thp}$ and initial density $\hiddens{0}$. Furthermore, conditionally on $(X_t)_{t = 0}^T$, the observations $(Y_t)_{t = 0}^T$ are independent with marginals $\emdens{\thp}(\cdot \mid X_t)$, $0 \leq t \leq T$. 
In the practical application of SSMs, the latent states are usually inferred from the observations using the so-called \emph{joint-smoothing distributions} $\filt{\thp}{0:t}(x_{0:t}) \eqdef p_\thp(x_{0:t} \mid y_{0:t})$ or \emph{filter distributions} $\filt{\thp}{t}(x_t) \eqdef p_\thp(x_{t} \mid y_{0:t})$. 
These state posteriors are also of paramount importance when inferring the parameter $\thp$ using maximum likelihood estimation   \citep[see, \eg,][Chapters~10--11]{cappe:moulines:ryden:2005}.


\subsection{The asymptotic contrast function}
Given a batch $Y_{0:T}$ of observations,  
the maximum-likelihood estimator (MLE) of $\thp$ is the parameter $\hat{\thp} \in\parspace$ such that $\log p_{\hat{\thp}}(Y_{0:T})\ge \log p_{\thp}(Y_{0:T})$ for all $\thp\in\parspace$.
However, in this paper we focus on the more challenging situation where the data become available via a data stream $(Y_t)_{t \in \nset}$. The data is assumed to be generated by some SSM, which does not necessarily belong to the parametric family governed by $\thp$. In this case, it becomes increasingly costly to evaluate the log-likelihood, up to the point where reprocessing all the observations as soon as a new one is recorded becomes infeasible. 
We will therefore focus instead on maximising the so-called \emph{contrast function}  
$\contr : \parspace \ni \thp \mapsto  \lim_{t\to\infty} t^{-1}\log p_\thp(Y_{0:t})$ with respect to $\theta$. In order to establish the (a.s.) existence of this objective (see~Section~\ref{sec:theory}), one typically considers the \emph{extended Markov chain} $(X_{t+1},Y_{t+1},\filt{\thp}{t},\tang{\thp}{t})_{t\in\nset}$, where $\tang{\thp}{t}(x)=\nabla_\thp \filt{\thp}{t}(x)$ is the so-called \emph{tangent-filter} density. Indeed, the fact that this process is Markov follows from the Markovianity of the SSM and the existence of mappings $\recf{\thp}$ and $\rect{\thp}$ such that for all $t\in\nset$, $\filt{\thp}{t+1}=\recf{\thp}(\filt{\thp}{t},Y_{t+1})$ 
and $\tang{\thp}{t+1} = \rect{\thp}(\tang{\thp}{t}, \filt{\thp}{t}, Y_{t+1})$ (see Appendix~\ref{app:proofs} for details). This chain can be shown to be ergodic under certain mixing assumptions \citep[see, \eg,][]{legland:mevel:1997,douc:matias:2002,tadic:doucet:2005}, and we denote by $\Pi_\thp
$ its stationary distribution and by $\bar{\Pi}_\thp$ the marginal of $\Pi_\thp$ w.r.t. the observation and the filter (it should be remarked that the existence of $\Pi_\thp$ is a mathematically involved topic, and we refer to the previous references for discussions). Using the strong law of large numbers for Markov chains, this construction allows us to express the contrast function as an expectation under $\Pi_\thp$ according to 
\begin{multline} \label{eq:contrast:ergodic:limit}
\ell(\theta) = \lim_{t \to \infty}\frac{1}{t} \sum_{s = 0}^{t - 1} \log \iint g_\thp(Y_{s + 1} \mid x_{s + 1}) \\\times m_\thp(x_{s + 1} \mid x_s) \, dx_{s + 1} \, \filt{\thp}{s}(x_s)\, dx_s 
\\= \lim_{t \to \infty} \frac{1}{t}\sum_{s = 0}^{t - 1} \likfunc{\thp}(Y_{s + 1},\filt{\thp}{s}) 
= \iint \likfunc{\thp}(y,\filt{}{})\,\bar{\Pi}_\thp(dy,d\filt{}{})
\end{multline}
(a.s.), where we have defined 
\begin{equation}\label{eq:def_likfunc}
	\likfunc{\thp}(y,\filt{}{})\eqdef\log \iint \emdens{\thp}(y \mid x')\hiddens{\thp}(x' \mid x)\,dx' \,\filt{}{}(x)\, dx. 
\end{equation}
It has been demonstrated that if the data is assumed to be generated by some SSM in the parametric family of interest, specified by some `true' parameter $\thp^*\in\parspace$, then, under some identifiability assumptions, $\contr(\thp)$ is maximised by $\thp^*$ and the MLE tends 
tends almost surely to $\thp^*$ as $t$ tends to infinity (strong consistency); see, \eg, \citet[Theorems~1--2]{douc:matias:2002} or \citet[Section~12.4]{cappe:moulines:ryden:2005}. In {\RML}, the contrast function is maximised online using stochastic approximation \citep{robbins:monro:1951}. 
Arguing as in \eqref{eq:contrast:ergodic:limit}, it can be shown that 
$\nabla_{\thp}\contr(\thp)=\int \gradlik{\thp}(y,\filt{}{},\tang{}{})\,\tilde{\Pi}_\thp(dy,d\filt{}{},d\tang{}{})$,
where $\tilde{\Pi}_\thp
$ is the marginal of $\Pi_\thp$ w.r.t. the observation, the filter, and the tangent filter, and
\begin{multline}
	\gradlik{\thp}(y,\filt{}{},\tang{}{})\eqdef\frac{\iint\emdens{\thp}(y\mid x')\hiddens{\thp}(x' \mid x)\,dx' \tang{}{}(x)\,dx}{\iint \emdens{\thp}(y\mid x')\hiddens{\thp}(x' \mid x)\, dx' \filt{}{}(x)\,dx}
	\\+\frac{\iint\nabla_{\thp}\{\emdens{\thp}(y\mid x')\hiddens{\thp}(x'\mid x)\}\,dx'\filt{}{}(x)\,dx}{\iint \emdens{\thp}(y\mid x')\hiddens{\thp}(x' \mid x)\, dx' \filt{}{}(x)\,dx}.
\end{multline}
Now, letting $\nabla_{\thp}\contr(\thp)$ serve as the mean field of a stochastic approximation scheme with state-dependent Markov noise \citep[see, \emph{e.g.},][Case~2]{karimi:2019}, a recursive Robbins--Monro algorithm finding a stationary point of the constrast function is given by 
\begin{equation}\label{eq:rml_update}
	\thp_{t+1}\gets\thp_t+\gamma_{t+1}\gradlik{\thp_t}(y_{t+1},\filt{\thp_{0:t}}{t},\tang{\thp_{0:t}}{t}), \quad t \in \nset, 
\end{equation}
followed by the updates $\filt{\thp_{0:t+1}}{t+1}=\recf{\thp_{t+1}}(\filt{\thp_{0:t}}{t},y_{t+1})$ and $\tang{\thp_{0:t+1}}{t+1}= \rect{\thp_{t+1}}(\tang{\thp_{0:t}}{t}, \filt{\thp_{0.t}}{t}, y_{t+1})$, where $(\gamma_t)_{t \in \nsetpos}$ is a suitable-chosen sequence of step sizes. The recursion \eqref{eq:rml_update} is initialised by some guess $\thp_0$ with associated time-zero filter and tangent filter $\filt{\thp_0}{0}$ and $\tang{\thp_0}{0}$, respectively. Except in cases where the model is linear Gaussian, $\gradlik{\thp}$, $\recf{\thp}$, and $\rect{\thp}$ have no closed-form expressions, so the practical implementation of \eqref{eq:rml_update} generally requires these quantities to be approximated. For this purpose, SMC methods have proved particularly useful \citep[see, \eg,][]{poyiadjis:doucet:singh:2011, delmoral:doucet:singh:2015,olsson:westerborn:2018}, and we shall return to this in Section~\ref{sec:smc_algo}.


\section{The online sequential importance-weighted autoencoder (\siwae)}\label{sec:method}
In the following our goal is to determine a lower bound on the asymptotic contrast by following the principles of the {\iwae} and to design a stochastic-approximation scheme to maximise the same. By maximising a lower bound on the contrast function, we are able to learn not only the model parameters, but also a variational recognition model, which can be used, for example, as a particle proposal. This is not possible through standard {\RML}. Similar to {\RML}, our first algorithm, {\siwae}, outlined in Section~\ref{sec:method} will not be implementable in the general case. Therefore, in Section~\ref{sec:smc_algo} we will present a practical particle-based version, {\smcsiwae}.


\subsection{The {\siwae} algorithm}

We return to \eqref{eq:def_likfunc} and focus on the inner integral $\int \emdens{\thp}(y\mid x')\hiddens{\thp}(x'\mid x)\,dx'$ of $\likfunc{\thp}$, which represents the likelihood of a certain observation $y\in\set{Y}$ of the SSM given the latent state $x \in\set{X}$ at the previous time step. Given $x$ and $y$, we may be interested in inferring the latent state $x'$ at the next time step by determining the conditional distribution $p_\thp(x'\mid x,y)\propto \emdens{\thp}(y\mid x')\hiddens{\thp}(x'\mid x)$. This conditional distribution is of crucial importance in particle filters, since it corresponds to the \emph{locally optimal proposal} \citep[see, \eg,][for details]{cornebise:moulines:olsson:2008}. The optimal proposal allows the particles to be guided more efficiently than the naive \emph{bootstrap proposal} $\hiddens{\thp}(x' \mid x)$, which mutates the particles `blindly', without including information about the current observation \citep{gordon:salmond:smith:1993}. 
Alternatively, if the SSM is interpreted as a model for encoding a process on a high-dimensional space $\set{Y}$ into a space $\set{X}$ of significantly lower dimension, then $p_\thp(x' \mid x, y)$ provides a sequential encoder that takes into account both the observed data and the previously encoded state \citep[see, \eg,][]{chung:2015}. The optimal kernel can be determined in a closed form for only a few model types \citep[][Section~7.2.2.2]{doucet:godsill:andrieu:2000,cappe:moulines:ryden:2005}; thus, our 
goal is to learn an approximation $\propdens{\thp}(x' \mid x, y)$, referred to simply as the \emph{proposal}, of the optimal proposal using variational inference. 
Note that $\propdens{\thp}$ is parameterised by the same $\thp$ as the SSM. This notation covers both the case where the parameters of the proposal are a subset of the parameters of the SSM as well as the more interesting case where the proposal involves additional parameters. In the latter case, $\parspace$ is the product of a model and a proposal-only parameter space.

Now, let $\propdens{\thp}$ be such that for every $x\in\set{X}$ and $y\in\set{Y}$, 
$\{x' : \propdens{\thp}(x' \mid x, y)=0\} \subseteq \{x' : \hiddens{\thp}(x'\mid x)=0 \}$. 
Moreover, for a given probability density $\phi$ on $\set{X}$ and $y\in\set{Y}$, define the joint probability densities $\pfilt{\thp}{\filt{}{}}(x ,x', y)\eqdef\filt{}{}(x)\hiddens{\thp}(x'\mid x)\emdens{\thp}(y\mid x')$ and $\vardist{\thp}{\filt{}{}}(x,x'\mid y)\eqdef \filt{}{}(x)\propdens{\thp}(x'\mid x, y)$. Using this notation and definition \eqref{eq:def_likfunc}, write 
\begin{equation}
    \likfunc{\thp}(y,\filt{}{}) 
	=\log 
 \E_{\vardist{\thp}{\filt{}{}}(\cdot \mid y)}
 \left[\frac{\pfilt{\thp}{\filt{}{}}(X,X', y)}{\vardist{\thp}{\filt{}{}}(X,X'\mid y)}\right].
\end{equation}

From this expression we immediately see that we are in the framework of {\vae}s
\citep{kingma:welling:2014}, where $\pfilt{\thp}{\filt{}{}}$ and $\vardist{\thp}{\filt{}{}}$ are the \emph{generative} and  \emph{recognition models}, respectively, in the special case where the unobserved latent variable comprises two consecutive states. From here it is a short step to generalising the variational objective using the {\iwae} framework
\citep{burda:2016}
, yielding 
\begin{multline}
	\vfunc{\thp}(y,\filt{}{}) \\ \eqdef 
 \E_{\vardist{\thp}{\filt{}{}}(\cdot \mid y)^{\varotimes M}}
 \left[\log\left(\frac{1}{\M}\sum_{i=1}^{\M}\frac{\pfilt{\thp}{\filt{}{}}(X^i,X'^i, y)}{\vardist{\thp}{\filt{}{}}(X^i,X'^i\mid y)}\right)\right]
 \\
 \le \likfunc{\thp}(y,\filt{}{}),
\end{multline}
where $\M\in\nsetpos$ is a sample-size hyperparameter.
Here $(X^i, X'^i)_{i = 1}^\M$ are independent draws from $\vardist{\thp}{\filt{}{}}(\cdot \mid y)$. 
Thus, we may consider the asymptotic variational objective 
$\contr^\M : \Theta \ni \thp \mapsto \lim_{t \to \infty} t^{-1} \sum_{s = 0}^{t - 1} \vfunc{\thp}(Y_{s + 1},\filt{\thp}{s})$, whose (a.s.) existence is guaranteed by Proposition~\ref{prop:strong:limits} and which, since it  
bounds the contrast function from below, will be referred to as \emph{contrast lower bound} (\colbo). Note that the {\colbo} can be expressed as the ergodic limit $\contr^\M(\thp) = \int \vfunc{\thp}(y,\filt{}{})\,\bar{\Pi}_\thp(dy,d\filt{}{})$. 
Now, similarly to the {\RML}, we may proceed by constructing a stochastic-approximation scheme with state-dependent Markov noise targeting the zeros of $\nabla_{\thp}\contr^\M(\thp)$, which in turn coincides with the (a.s.) limit of $t^{-1} \sum_{s = 0}^{t - 1} \nabla_\thp \vfunc{\thp}(Y_{s + 1},\filt{\thp}{s})$; see again Proposition~\ref{prop:strong:limits}. In order to identify the stochastic update, we need to derive an explicit expression for $\nabla_{\thp}\vfunc{\thp}(Y_{t+1},\filt{\thp}{t})$ as a function of the states of the extended Markov chain. 
In this derivation, we will apply the reparameterisation trick \citep{kingma:welling:2014}, by assuming that 
there exists some auxiliary random variable $\auxrv$, taking on values in some measurable space $(\set{U}, \alg{U})$ and having distribution $\indmeas(\auxrvb)$ on $(\set{U}, \alg{U})$ (the latter not depending on $\thp$), and some function $\repfunc{\thp}$ on $\set{X} \times \set{Y} \times \set{U}$, parameterised by $\thp$ and differentiable with respect to the same for any given argument $(x, y, \auxrvb)$, such that  for every $(x, y) \in \set{X} \times \set{Y}$, the pushforward distribution $\indmeas \circ \repfunc{\thp}^{-1}(x, y, \cdot)$ coincides with that governed by $\propdens{\thp}(\cdot \mid x, y)$. Defining the weight function 
\begin{equation}
	\wgtfunc{\thp}(x,y,\auxrvb)\eqdef\frac{\emdens{\thp}(y\mid \repfunc{\thp}(x, y, \auxrvb)) \hiddens{\thp}(\repfunc{\thp}(x, y, \auxrvb)\mid x)}{\propdens{\thp}(\repfunc{\thp}(x,y, \auxrvb)\mid x, y)}
\end{equation}
allows us to write, for a given $y \in \set{Y}$,  
\begin{multline}
	\nabla_\thp \vfunc{\thp}(y,\filt{\thp}{t})
	= \nabla_\thp\iint \log\left(\frac{1}{\M}\sum_{i=1}^{\M}\wgtfunc{\thp}(x^i, y, \auxrvb^i)\right)
	\\\times\prod_{j=1}^{\M}\indmeas(\auxrvb^j) \filt{\thp}{t}(x^j)\,d\auxrvb^{1:\M}\,dx^{1:\M}
	\\
 =\E_{(\filt{\thp}{t}\tensprod\indmeas)^{\varotimes \M}} \left[\frac{\sum_{i=1}^{\M}\nabla_\thp\wgtfunc{\thp}(X^i, y, \auxrv^i)}{\sum_{i'=1}^{\M}\wgtfunc{\thp}(X^{i'}, y, \auxrv^{i'})}\right]
	\\
 +\sum_{j'=1}^{\M}\iint\log\left(\frac{1}{\M}\sum_{i=1}^{\M}\wgtfunc{\thp}(x^i, y, \auxrvb^i)\right)\nabla_\thp\filt{\thp}{t}(x^{j'})
 \\\times 
\prod_{\substack{j = 1 \\ j \neq j'}}^M
 \filt{\thp}{t}(x^j)\prod_{k=1}^{\M}\indmeas(\auxrvb^{k})\,d\auxrvb^{1:\M}\,dx^{1:\M},
\end{multline}
where, in the first term,  
$(X^i, \auxrv^i)_{i=1}^\M$ are i.i.d. with distribution $\filt{\thp}{t}(x)\indmeas(\auxrvb)$. 
Note that by symmetry, the terms of the outer sum are identical; hence, letting 
\begin{multline}
	\grad{\thp}(y,\filt{}{},\tang{}{})\eqdef\E_{(\filt{}{}\tensprod\indmeas)^{\varotimes \M}} \left[\frac{\sum_{i=1}^{\M}\nabla_\thp\wgtfunc{\thp}(X^i, y, \auxrv^i)}{\sum_{i'=1}^{\M}\wgtfunc{\thp}(X^{i'}, y, \auxrv^{i'})}\right]
	\\+\M\int\E_{(\filt{}{}\tensprod\indmeas)^{\varotimes (\M-1)}}\left[\log\left(\frac{1}{\M}\wgtfunc{\thp}(x, y, \auxrvb)\vphantom{\sum_{i=1}^{\M-1}\wgtfunc{\thp}}\right.\right.
	\\\left.\left.+\frac{1}{\M}\sum_{i=1}^{\M-1}\wgtfunc{\thp}(X^i, y, \auxrv^i)\right)\right]\tang{}{}(x)\indmeas(\auxrvb)\,d\auxrvb\,dx,
\end{multline}
we may write $\nabla_\thp\vfunc{\thp}(y_{t+1},\filt{\thp}{t})=\grad{\thp}(y_{t+1},\filt{\thp}{t},\tang{\thp}{t})$ and, consequently,  
$\nabla_{\thp}\contr^\M(\thp)=
\int \grad{\thp}(y,\filt{}{},\tang{}{})\,\tilde{\Pi}_\thp(dy,d\filt{}{},d\tang{}{})$. Thus, similar to \eqref{eq:rml_update}, we may find a zero of $\nabla_{\thp}\contr^\M(\thp)$ using the Robbins--Monro scheme 
\begin{equation}\label{eq:osiwae_update}
	\thp_{t+1}\gets\thp_t+\gamma_{t+1}\grad{\thp_t}(Y_{t+1},\filt{\thp_{0:t}}{t},\tang{\thp_{0:t}}{t}), \quad t \in \nset, 
\end{equation}
followed by the updates $\filt{\thp_{0:t+1}}{t+1}=\recf{\thp_{t+1}}(\filt{\thp_{0:t}}{t},Y_{t+1})$ and $\tang{\thp_{0:t+1}}{t+1}= \rect{\thp_{t+1}}(\tang{\thp_{0:t}}{t}, \filt{\thp_{0.t}}{t}, Y_{t+1})$, where $(\gamma_t)_{t \in \nsetpos}$ is a suitable-chosen sequence of step sizes. We refer to schedule \eqref{eq:osiwae_update}, which is initialised as the {\RML} \eqref{eq:rml_update}, as the 
\emph{online sequential importance-weighted auto-encoder} (\siwae).
As in the case of {\RML}, a practical implementation requires the approximations of $\grad{\thp}$, $\recf{\thp}$, and $\rect{\thp}$. This is the objective of Section~\ref{sec:smc_algo}, where we describe a practical version of the {\siwae} based on SMC methods. 

\subsection{Theoretical properties of the COLBO}\label{sec:theory}
All the results displayed below are established under strong mixing assumptions on the SSM and the data-generating process. Furthermore, $\hiddens{\thp}$, $\emdens{\thp}$, $\propdens{\thp}$, and their compositions with the reparameterisation function $h_\thp$ are assumed be differentiable in $\thp$ with bounded gradients. These assumptions are standard in the literature and point to applications where the state and parameter spaces are compact. All details and proofs are found in Appendix~\ref{app:proofs}. Our first result serves to define properly the objectives under consideration. 
\begin{proposition}
\label{prop:strong:limits}
	For all $\M\in\nsetpos$ there exist real-valued differentiable functions $\contr$ and $\contr^\M$ on $\parspace$ such that for all $\thp\in\parspace$, $\prob$-a.s., 
 	\begin{align}
		&\lim_{t \to \infty}\frac{1}{t} \sum_{s=0}^{t-1}\likfunc{\thp}(Y_{s+1},\filt{\thp}{s}) = \contr(\thp), 
  \\&\lim_{t \to \infty} \frac{1}{t}\sum_{s=0}^{t-1}\gradlik{\thp}(Y_{s+1},\filt{\thp}{s},\tang{\thp}{s}) =  \nabla_\thp\contr(\thp),
	\\ &\lim_{t \to \infty}\frac{1}{t}\sum_{s=0}^{t-1}\vfunc{\thp}(Y_{s+1},\filt{\thp}{s}) = \contr^\M(\thp), 
  \\&\lim_{t \to \infty} \frac{1}{t}\sum_{s=0}^{t-1}\grad{\thp}(Y_{s+1},\filt{\thp}{s},\tang{\thp}{s}) = \nabla_\thp\contr^\M(\thp).
	\end{align}
\end{proposition}
The next result characterizes the {\siwae} objective by establishing the relation between the asymptotic contrast function and the {\colbo}. It is a direct consequence of \citep[][Theorem~1]{burda:2016} and \citep[][Proposition~1]{nowozin:2018}.

\begin{proposition}
	For all $\thp \in \parspace$ and $\M\in\nsetpos$, $\contr(\thp)\ge \contr^{\M+1}(\thp)\ge \contr^\M(\thp)$. Moreover, $\contr(\thp)-\lyap(\thp)$ is $\mathcal{O}(\M^{-1})$ uniformly in $\thp$. 
\end{proposition}

Finally, we establish an $\mathcal{O}(M^{-1})$ bias between the stochastic gradients $\gradlik{\thp}$ and $\grad{\thp}$.

\begin{theorem}
\label{thm:biasgradmain}
	For all 
 $(y_t)_{t\in\nset}$, $\grad{\thp}(y_{t+1},\filt{\thp}{t},\tang{\thp}{t})-\gradlik{\thp}(y_{t+1},\filt{\thp}{t},\tang{\thp}{t})$ is $\mathcal{O}(\M^{-1})$ uniformly in $t$ and $\thp$.  
 In addition, $\nabla_\thp\contr^\M(\thp)-\nabla_\thp\contr(\thp)$ is $\mathcal{O}(\M^{-1})$ uniformly in $\thp$. 
\end{theorem}

\section{SMC-based {\siwae} (\smcsiwae)}\label{sec:smc_algo}
We now present an implementable version of the {\siwae} algorithm based on SMC methods. 
For this purpose, we first 
provide an alternative expression of $\grad{\thp}$, where the integral involving $\tang{\thp}{t}$ is expressed as an expectation of the \emph{complete-data score} $\nabla_\thp\log p_\thp(X_{0:t},Y_{0:t})$ under the joint-smoothing distribution $\filt{\thp}{0:t}$. The complete-data score is 
of additive form, allowing for sequential updates with constant complexity (see next section). The following lemma, whose proof is found in Appendix~\ref{app:tfunc}, summarises these properties.
First, for $t \in \nsetpos$ and $\thp \in \Theta$, define the functions $\tfunc{\thp}{0}(x_0)=\nabla_\thp\log\emdens{\thp}(y_0\mid x_0)$ and
\begin{multline}
    \tfunc{\thp}{t}(x_{t})\eqdef\int \nabla_\thp\log p_\thp(x_{0:t},y_{0:t})
    \\p_\thp(x_{0:t-1}\mid y_{0:t-1},x_{t})\,dx_{0:t-1}.
\end{multline}
Moreover, let, for every density $\phi$, function $\varphi$, and $y \in \set{Y}$, 
\begin{multline}\label{eq:def_gradt}
			\gradt{\thp}(y, \phi,\varphi)\eqdef \E_{(\phi\tensprod\indmeas)^{\varotimes \M}}  \left[\frac{\sum_{i=1}^{\M}\nabla_\thp\wgtfunc{\thp}(X^i, y, \auxrv^i)}{\sum_{i'=1}^{\M}\wgtfunc{\thp}(X^{i'}, y, \auxrv^{i'})}\right.
			\\+\M\log\left(\frac{1}{\M}\sum_{i=1}^{\M}\wgtfunc{\thp}(X^i, y, \auxrv^i)\right)
			\\\left.\times\left(\varphi(X^\M)-\E_{\phi}[\varphi(X)]\right)\vphantom{\frac{\sum_{i=1}^{\M}\nabla_\thp\wgtfunc{\thp}(X^i, y, \auxrv^i)}{\sum_{i'=1}^{\M}\wgtfunc{\thp}(X^{i'}, y, \auxrv^{i'})}}\right].
		\end{multline}
Then the following holds true. 
\begin{lemma} \label{lemma:tfunc} \ \\[-5mm]
 \begin{itemize}
		\item[(i)] There exists a mapping $\rectsmooth{\thp}$ such that for all $t\in\nset$,	$\tfunc{\thp}{t+1}=\rectsmooth{\thp}(\tfunc{\thp}{t},\filt{\thp}{t}, Y_{t+1})$. 
		\item[(ii)] 
		It holds that 
  $\gradt{\thp}(Y_{t+1},\filt{\thp}{t},\tfunc{\thp}{t})=\grad{\thp}(Y_{t+1},\filt{\thp}{t},\tang{\thp}{t})$.
	\end{itemize} 
\end{lemma}

Building on Lemma~\ref{lemma:tfunc}, the {\siwae} procedure may be reformulated by substituting $\grad{\thp}$ with $\gradt{\thp}$ and replacing the tangent-filter sequence by $(\tfunc{\thp_{0:t}}{t})_{t\in\nset}$. These updates are performed online as well according to $\tfunc{\thp_{0:t+1}}{t+1}= \rectsmooth{\thp_{t+1}}(\tfunc{\thp_{0:t}}{t}, \filt{\thp_{0.t}}{t}, Y_{t+1})$. The initialisation step involves computing $\filt{\thp_{0}}{0}$ and setting $\tfunc{\thp_0}{0}(x_0)=\nabla_\thp\log \emdens{\thp_0}(Y_0\mid x_0)$. Still, this idealised approach is impractical for direct implementation, why a particle-based version of the same is presented in the next section. 

\subsection{{\siwae} gradient-step approximation}\label{subsec:smcsiwae}
We assume that at each time $t\in\nset$ we have access to some weighted particle sample $(\epart{t}{i},\wgt{t}{i})_{i=1}^N$, $N\in\nsetpos$, whose associated weighted empirical measure approximates $\filt{\thp}{t}$. In addition, assume that we have access to some associated statistics $(\tstat{t}{i})_{i=1}^N$ such that $\tstat{t}{i}\simeq \tfunc{\thp}{t}(\epart{t}{i})$ for all $i$. We assume that $\sum_{i=1}^{N} \wgt{t}{i} f(\epart{t}{i}) / \wgtsum{t} \simeq \E_{\filt{\thp}{t}}[f(X)]$ and $\sum_{i=1}^{N} \wgt{t}{i} \tstat{t}{i}f(\epart{t}{i}) / \wgtsum{t} \simeq \E_{\filt{\thp}{t}}[f(X)\tfunc{\thp}{t}(X)]$, 
where $\wgtsum{t}\eqdef\sum_{i=1}^{N}\wgt{t}{i}$, for every measurable function $f$ such that these expectations are well defined. The sample $(\epart{t}{i}, \tau_t^i, \wgt{t}{i})_{i=1}^N$ will be produced using the so-called \texttt{AdaSmooth} algorithm proposed by \cite{mastrototaro:olsson:alenlov:2021} (see Appendix~\ref{app:adasmooth}). Given this sample, our goal is to approximate the expectation \eqref{eq:def_gradt} when the inputs $\filt{\thp}{t}$ and $\tfunc{\thp}{t}$ are replaced by their particle approximations. In the standard {\iwae}, when $\M$ is sufficiently large, a good estimate of the gradient is typically obtained by simply drawing $\M$ i.i.d. samples from the (reparameterised) recognition model.
Thus, in our case we would ideally need $\M$ i.i.d. samples from $\filt{\thp}{t}\tensprod \nu$ at the iteration $t$. However, since $\filt{\thp}{t}$ is generally intractable, we sample instead $M$ conditionally i.i.d. draws from the empirical distribution formed by a particle sample $(\epart{t}{i},\wgt{t}{i})_{i=1}^N$ targeting $\filt{\thp}{t}$. 
More precisely, write 
\begin{multline} \label{eq:g_bar:alt:form}
\gradt{\thp}(y, \phi_t^\thp,\varphi_t^\thp) \\
= \iint  \E_{(\phi_t^\thp \tensprod \nu)^{\tensprod (M - 1)}} \left[ \Gamma_1^\thp(x, u, X^{1:M - 1}, U^{1:M - 1}, y) \right] \\
+ \E_{(\phi_t^\thp \tensprod \nu)^{\tensprod (M - 1)}} \left[ \Gamma_2^\thp(x, u, X^{1:M - 1}, U^{1:M - 1}, y) \right] \\\times (\varphi_t^\thp(x) -\E_{\phi_t^\thp}[\varphi_t^\thp(X)])  \,\phi_t^\thp(x) \nu(u) \, dx \, du, 
\end{multline}
where 
\begin{align}
\lefteqn{\Gamma_1^\thp(x, u, x^{1:M - 1}, u^{1:M - 1}, y)} \\
&\eqdef \frac{\nabla_\thp\wgtfunc{\thp}(x, y, \auxrvb)+\sum_{i=1}^{\M-1}\nabla_\thp\wgtfunc{\thp}(x^i, y, \auxrvb^i)}{\wgtfunc{\thp}(x, y, \auxrvb)+\sum_{k=1}^{\M-1}\wgtfunc{\thp}(x^k, y, \auxrvb^{k})}, \\
\lefteqn{\Gamma_2^\thp(x, u, x^{1:M - 1}, u^{1:M - 1}, y)} \\ 
&\eqdef\M \log\left(\frac{1}{\M}\wgtfunc{\thp}(x, y,\auxrvb)
+\frac{1}{\M}\sum_{i=1}^{\M-1}\wgtfunc{\thp}(x^i, y, \auxrvb^{i})\right).
\end{align}
Now, estimating (i) the inner expectations of \eqref{eq:g_bar:alt:form} based on $M - 1$ independent draws $(\check{\xi}_t^i, \check{\upsilon}^i_t)_{i = 1}^{M - 1}$ generated as 
$$
(\check{\xi}_t^i, \check{\upsilon}^i_t) \sim \left( \sum_{i=1}^{N}\frac{\wgt{t}{i}}{\wgtsum{t}}\delta_{\epart{t}{i}} \right) \tensprod \nu,  
$$
\emph{i.e.}, by resampling pairs of particles and associated statistics in proportion to their weights and providing each resampled pair with a draw from $\nu$, and then (ii) the outer integral using samples $(\hat{\xi}_t^i, \hat{\tau}_t^i, \hat{\upsilon}_t^i)_{i = 1}^N$ drawn independently according to 
$$
(\hat{\xi}_t^i, \hat{\tau}_t^i, \hat{\upsilon}_t^i) \sim \left( \sum_{i=1}^{N}\frac{\wgt{t}{i}}{\wgtsum{t}}\delta_{(\epart{t}{i}, \tau_t^i)} \right) \tensprod \nu , 
$$
allows $\gradt{\thp}(y, \phi^\thp_t,\varphi^\thp_t)$ to be estimated by 
\begin{multline} \label{eq:g_bar:estimator}
 \bar{\Gamma}^\thp(\hat{\xi}_t^{1:N}, \hat{\upsilon}_t^{1:N}, \check{\xi}_t^{1:M - 1}, \check{\upsilon}_t^{1:M - 1},y) 
 \\\eqdef \frac{1}{N}\sum_{i=1}^{N}\left\{\vphantom{\frac{1}{N}\sum_{\ell=1}^{N}\hat{\tau}_t^\ell}\Gamma_1^\thp(\hat{\xi}_t^j, \hat{\upsilon}_t^j, \check{\xi}_t^{1:M - 1}, \check{\upsilon}_t^{1:M - 1}, y) \right.
	\\
    \left.+\Gamma_2^\thp(\hat{\xi}_t^j, \hat{\upsilon}_t^j, \check{\xi}_t^{1:M - 1}, \check{\upsilon}_t^{1:M - 1},y)\left( \hat{\tau}_t^i -\frac{1}{N}\sum_{\ell=1}^{N}\hat{\tau}_t^\ell \right)\right\}.
\end{multline}

\begin{algorithm}[htb]
	\caption{{\smcsiwae}}\label{algo:siwae_smc}
	\begin{algorithmic}[1]
		\REQUIRE $(\epart{t}{i},\tstat{t}{i},\wgt{t}{i})_{i=1}^N$, $Y_{t+1}$, $\thp_t$, step size $\gamma_{t+1}$.
        \STATE draw $(\check{\xi}_t^i, \check{\upsilon}_t^i)_{i=1}^{M-1}  \stackrel{\tiny{\mbox{i.i.d}}}{\sim} \left( \sum_{i=1}^{N}\frac{\wgt{t}{i}}{\wgtsum{t}}\delta_{\epart{t}{i}} \right) \tensprod \nu $
        \STATE draw $(\hat{\xi}_t^i, \hat{\tau}_t^i, \hat{\upsilon}_t^i)_{i=1}^N \stackrel{\tiny{\mbox{i.i.d}}}{\sim} \left( \sum_{i=1}^{N}\frac{\wgt{t}{i}}{\wgtsum{t}}\delta_{(\epart{t}{i}, \tau_t^i)} \right) \tensprod \nu $
        \STATE$\begin{aligned}[t]
			&\text{set }\thp_{t+1} \gets\thp_{t}
            \\&\hspace{4mm}+\gamma_{t+1}\bar{\Gamma}^{\thp_t}(\hat{\xi}_t^{1:N}, \hat{\upsilon}_t^{1:N}, \check{\xi}_t^{1:M - 1}, \check{\upsilon}_t^{1:M - 1}, Y_{t+1})
        \end{aligned}$
        \label{line:siwae_update_smc}
		\STATE run $(\epart{t+1}{i},\tstat{t+1}{i},\wgt{t+1}{i})_{i=1}^N$\\ $\hspace{8mm}\gets\texttt{AdaSmooth}((\epart{t}{i},\tstat{t}{i},\wgt{t}{i})_{i=1}^N, Y_{t+1}, \thp_{t+1})$
  \label{line:adasmooth}
		\RETURN $(\epart{t+1}{i},\tstat{t+1}{i},\wgt{t+1}{i})_{i=1}^N$, $\thp_{t+1}$
	\end{algorithmic}
\end{algorithm}
The procedure describing the practical {\siwae}, referred to as {\smcsiwae}, is displayed in Algorithm~\ref{algo:siwae_smc}, which also includes the online update of $(\epart{t}{i},\tstat{t}{i},\wgt{t}{i})_{i=1}^N$ via the \texttt{AdaSmooth} online particle smoother described in detail in Appendix~\ref{app:adasmooth}. As we mentioned earlier, in typical applications, $\thp=(\thpone, \thptwo)$, where $\thpone$ parameterises the SSM only, while $\thptwo$ parameterises the proposal. Now, note that by Theorem~\ref{thm:biasgradmain}, $\grad{\thp}$ converges to $\gradlik{\thp}$ as $\M$ tends to infinty, where the latter does not involve $\propdens{\thptwo}$; hence, the components of $\grad{\thp}$ corresponding to the gradient with respect to $\thptwo$ converge to zero. This becomes a problem when implementing the {\smcsiwae}, as the estimate of the gradient with respect to $\thptwo$ suffers from a low signal-to-noise ratio when $\M$ is moderately large \citep[we refer to][for a discussion]{rainforth:iwae:2018}, while it is always favourable to use a large $\M$ in  the model-parameter estimation. Thus, in practice we suggest to repeat twice Algorithm~\ref{algo:siwae_smc} (except for Line~\ref{line:adasmooth}, which is executed only once) at each iteration $t$: first with $\M$ small, typically equal to 5 or 10, and updating $\thptwo$ only, then with $\M$ large to update $\thpone$. Alternative solutions have been discussed by \citet{roeder:2017,tucker:2018,finke:thiery:2019}. 
It is interesting to note that that when dealing with $\thptwo$, since $\nabla_{\thptwo}\log p_{\thpone}(x_{0:t}, y_{0:t})=0$, $(\tstat{t}{i})_{i=1}^N$ are all zero, and so is the second term of \eqref{eq:g_bar:estimator}. In this case, the update of $\thptwo$ is similar to that performed by the {\OVSMC} method \citep[][Algorithm~2]{ovsmc}. However, {\OVSMC} updates the model parameters without the complete-data score term, which is a source of bias of {\OVSMC}. Therefore, although the two methods are derived from different starting points, they can be related. Still, {\OVSMC} lacks a clear asymptotic objective, relying on a hard-to-control truncation of its gradient, whereas {\siwae} aims to maximise a well-defined lower bound on the contrast function, at the price of having to update recursively the statistics $(\tstat{t}{i})_{i=1}^N$ (as discussed in Appendix~\ref{app:adasmooth}). 

\section{Numerical experiments}\label{sec:experiments}
In this section, we provide numerical simulations to illustrate the performance of the proposed \texttt{OSIWAE} algorithm in the contexts of parameter learning, optimal filtering, and proposal adaptation. All the experiments were performed on a MacBook Air M2 and used the \texttt{ADAM} optimiser \citep{kingma:ba:2015}. If not otherwise stated, a constant learning rate of $\gamma_t = 0.001$ was used.

\subsection{Multivariate linear Gaussian SSM} \label{ex:LGSSM}
We consider a 10-dimensional multivariate linear Gaussian SSM to provide an initial assessment of the performance of {\smcsiwae}. 
Formally, let the state and observation spaces be  \( \set{X} = \mathbb{R}^{d_x} \) and \( \set{Y} = \mathbb{R}^{d_y} \), respectively, with \( d_x = d_y = 10 \). The SSM is defined by the state transition density \(m_\theta(x_{t+1} \mid x_t) = N_{d_x}(x_{t+1}; A x_t, S_u S_u^\intercal)\) and the observation density \(g_\theta(y_t \mid x_t) = N_{d_y}(y_t; B x_t, S_v S_v^\intercal) \). 
Here \( A \in \mathbb{R}^{d_x \times d_x} \) and \( B \in \mathbb{R}^{d_y \times d_x} \) are the state transition and observation matrices, respectively. The matrices \( S_u \) and \( S_v \) are diagonal covariance matrices for the process and observation noises. We employ a Gaussian proposal distribution \( r_\theta(\cdot \vert x_t, y_{t+1}) \) with mean vector and diagonal covariance matrix parameterised by two distinct neural networks taking $x_t$ and $y_{t+1}$ as inputs.
We let the true matrices \( A \) and \( B \) be diagonal with entries sampled uniformly from \([0.5, 1]\) and generate a dataset of observations by simulating the SSM. We then apply the {\smcsiwae} algorithm to estimate these matrices while computing particle-based filter expectations of the latent states. We compare the performance of {\smcsiwae} with the \texttt{AdaSmooth}-based {\RML} method and the {\OVSMC} algorithm. Reference values for the optimal filtering were obtained by executing the Kalman filter for the true model dynamics. 

\begin{figure}[htb]
    \centering
    \hspace{-0.5cm}
    \includegraphics[width=0.49\textwidth]{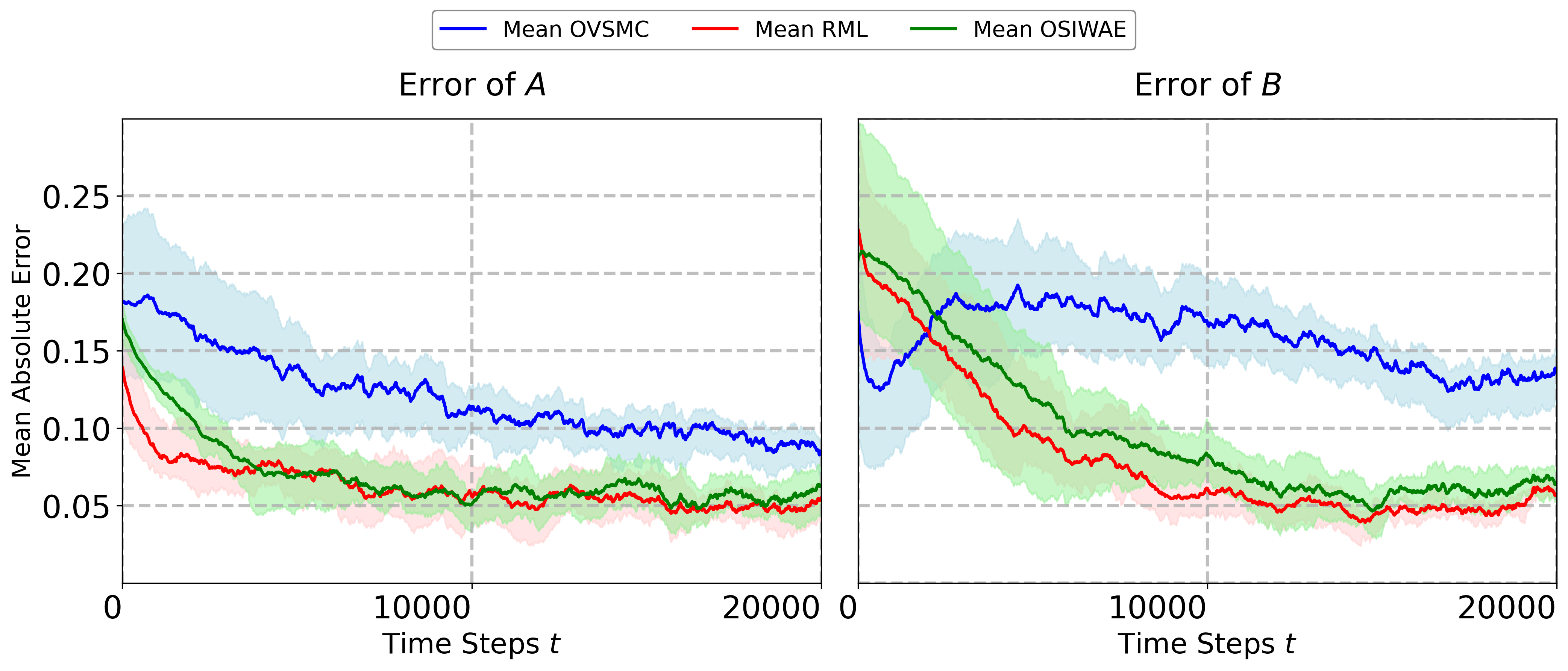}
    \caption{Parameter estimation errors over time for {\smcsiwae}, {\OVSMC}, and {\RML} in the scenario where \(S_u = 0.2I\) and \(S_v = 0.5I\). {\smcsiwae} and {\RML} used \(N = 1000\) particles and \(M = 1000\) importance samples, while {\OVSMC} used  \(N = 10000\) particles to ensure comparable computational complexity. The proposal distribution \(r_\theta\) (a 10-dimensional Gaussian distribution) was parameterised by two single-layer neural networks with 64 nodes each and ReLU activations and learned using \(L = 5\) particles. The error bounds are based on 30 independent runs of each algorithm. With our implementation, {\smcsiwae} took on average 42~min, {\OVSMC} 26~min, and {\RML} 46~min.}
    \label{fig:parameter_estimation}
\end{figure}


In Figure~\ref{fig:parameter_estimation}, we see that the parameter estimates produced using {\smcsiwae} converge faster and exhibit lower MAE compared to {\OVSMC}. This is explained by the fact that the stochastic gradient of {\OSIWAE} incorporates information from the past by estimating the complete-data score, which improves accuracy, especially when the observations are non-informative. It should be noticed that {\smcsiwae} is almost on par with {\RML}, although {\smcsiwae} simultaneously adapts the proposal while learning the model parameters.  

Hence, 
we next examine the filter-mean estimates produced by these algorithms as the parameters are being learned. We evaluate the MSEs of the filter-mean estimates of each algorithm with respect to the output of the Kalman filter executed for the true model parameters and display the result in Figure~\ref{fig:state_estimation}.
Clearly, after an initial phase, when both {\smcsiwae} and {\OVSMC} learn the proposal parameters and therefore perform worse than {\RML}, {\smcsiwae} shows a significantly better performance than its competitors in the long run. 




\begin{figure}[ht]
    \centering
    \includegraphics[width=0.7\columnwidth]{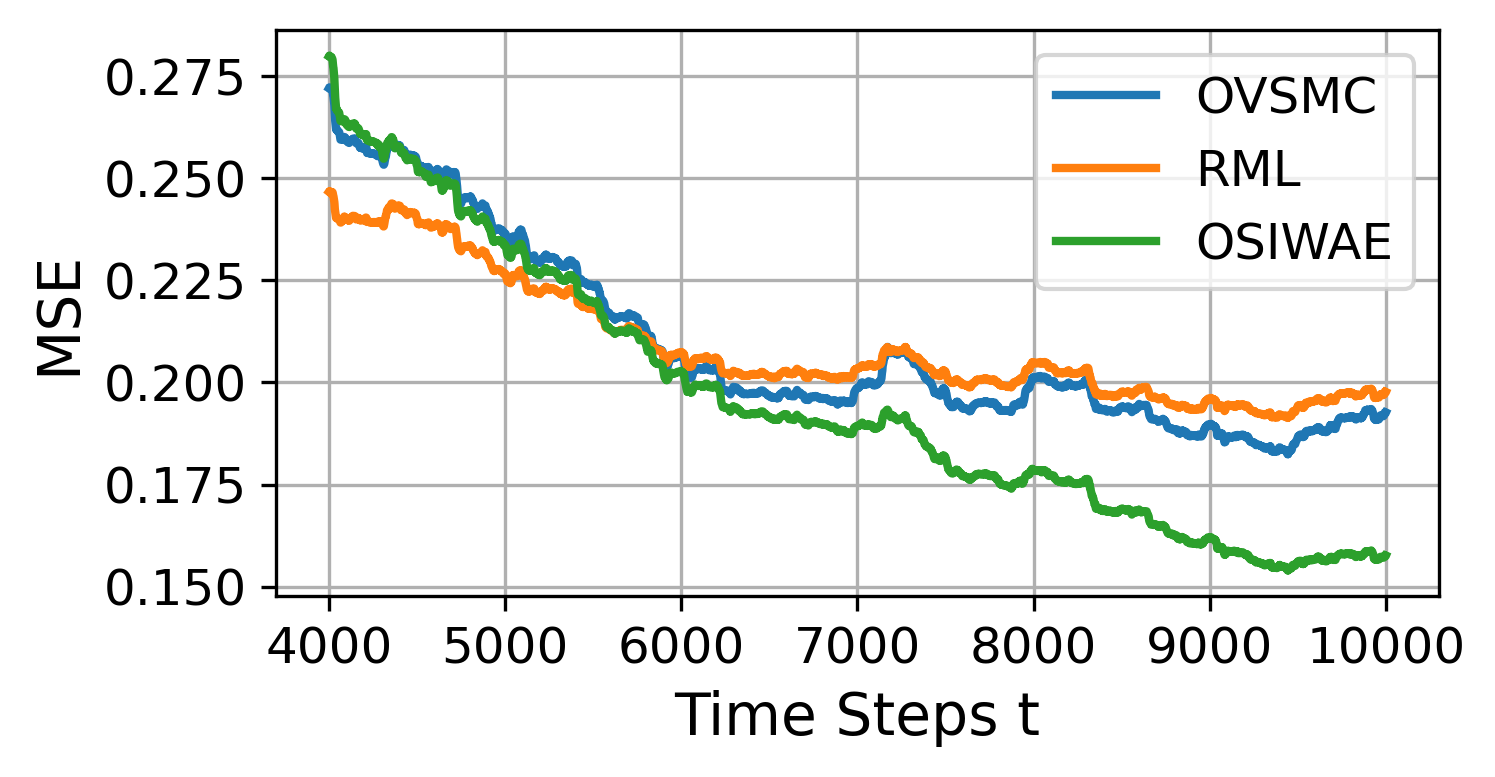}
    \caption{MSEs over time for {\OSIWAE}, {\OVSMC}, and {\RML} with respect to the Kalman filter (executed for true parameters) for the linear Gaussian model with \( S_u = 0.5I \) and \( S_v = 0.2I \). The values are plotted as moving averages with a window of 3000 time steps. For all methods, the MSEs are based on 50 independent runs on the same data and different starting values of \( A \) and \( B \).}
    \label{fig:state_estimation}
\end{figure}


\subsection{Simultaneous localisation and mapping}\label{ex:SLAM}

The \emph{simultaneous localisation and mapping} (SLAM) problem is fundamental in robotics and requires online inference; see, \eg, \citet{dissanayake2001solution,thrun2005probabilistic}. The goal is to jointly estimate the trajectory of a robot and the positions of $L\in\nsetpos$ unknown landmarks based on noisy observations. In this context, the latent states are the positions of the robot in a two-dimensional landscape at different time steps. These positions are partially observed through a vector of pairs indicating the distance and the angle with respect to the landmarks. We let $\theta=(\theta^1,\dots,\theta^L)$ be the positions of the landmarks, where $\theta^i=(\theta_1^i,\theta_2^i)\in\rset^2$. The robot's motion is modeled as a bivariate random walk with covariance matrix $\sigma_{\text{motion}}^2 I_2$. Here we have $Y_t = (Y_{t}^1,\dots,Y_t^L)$, where $Y_t^i$ is a tuple indicating a noisy measurement of the distance and the angle of the robot with respect to landmark $i$, for $i \in \{1,\dots L\}$. More specifically, $Y_t^i = h(X_t, \theta^i) + \sigma_{\text{obs}}V_t^i$, where $(V_t^1)_{t \in \nset}, \ldots, (V_t^L)_{t \in \nset}$ are sequences of bivariate i.i.d standard Gaussian random variables. The measurement function is such that $h(x, \theta^i) = ( \| \theta^i - x \|,\ \mbox{atan2}( \theta_2^i - x_2,\ \theta_1^i - x_1) )$ for all $x=(x_1,x_2)\in\rset^2$.

In this experiment, we aim to learn the unknown positions of the landmarks while sequentially estimating the position of the robot. We assume that the noise parameters are known. Figure \ref{fig:landmark_position} shows that after an initial phase where {\smcsiwae} is adapting the proposal, the landmark estimation becomes clearly better compared to both {\RML} and {\OVSMC}, and our algorithm is able to estimate the exact locations more precisely. On the right, we see that {\smcsiwae} is also able to first train the proposal in some environment and, when used in another one, the learning curve is more accurate than {\RML} and {\OVSMC} from the beginning.

\begin{figure}[ht]
    \centering   
    \includegraphics[width=1.\columnwidth]{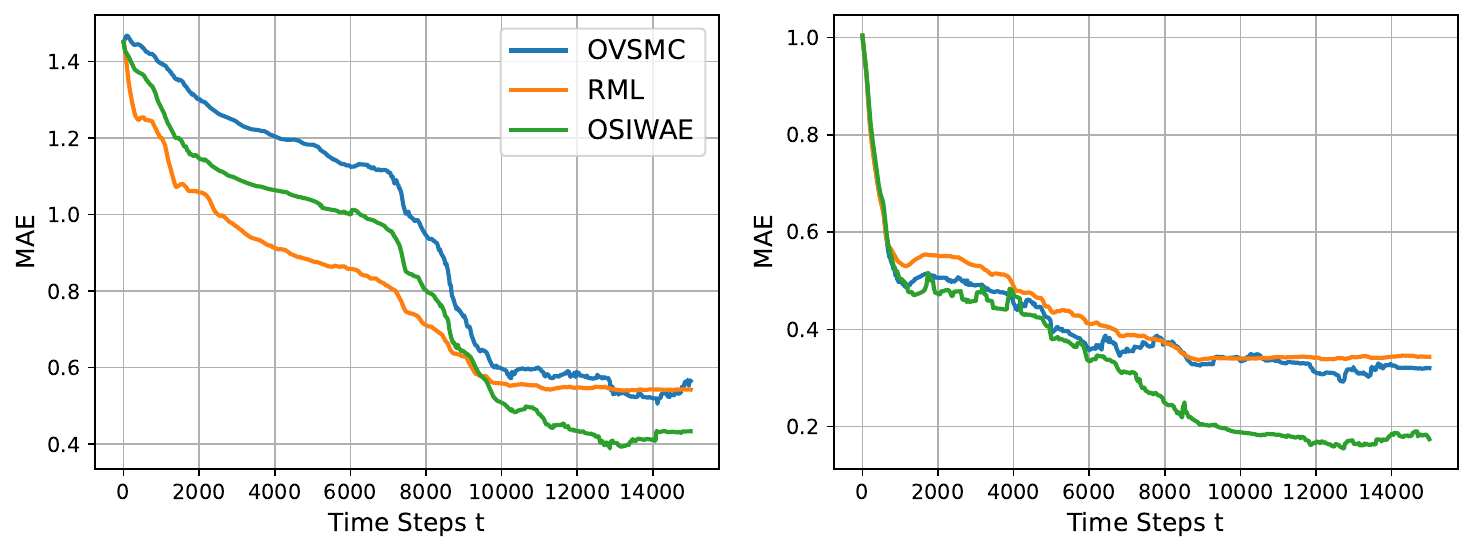}
    \caption{Average MAE of the estimated positions of $L = 8$ landmarks over time using {\OSIWAE}, {\RML}, and {\OVSMC} in a SLAM scenario with \( \sigma_{\text{motion}}^2 = 0.2 \) and \( \sigma_{\text{obs}}^2 = 0.1 \).
    The proposal distribution \( r_\theta(\cdot \mid x_t, y_{t+1}) \) in both {\OSIWAE} and {\OVSMC} is learned via two distinct neural networks, each with one hidden layer of 128 nodes. All three methods use \( N = 1000 \) particles and {\OSIWAE} uses $M = 1000$.
    Left panel: All three algorithms run on the same data, without any prior learning. 
    Right panel: A training run is first performed using {\smcsiwae} on a different data record to learn the proposal distribution; afterwards, all three algorithms are applied to the same data.  
    }
    \label{fig:landmark_position}
\end{figure}

\subsection{Growth Model}\label{Growth}

Finally, we consider the so-called \emph{growth model} 
\cite{kitagawa:1987},
which is a standard benchmark model in particle filtering due to the highly nonlinear latent process. The state dynamics is given by $X_t = a_{t-1}(X_{t-1}) + \sigma_u U_{t-1}$, where $a_{t-1}(x) = \alpha_0 x + \alpha_1 x / (1 + x^2) + \alpha_2 \cos(1.2(t - 1))$, and the observations process satisfies $Y_t = b X_t^2 + \sigma_v V_t$, where $(U_t)_{t\in\nset}$ and $(V_t)_{t\in\nset}$ are i.i.d standard Gaussian random variables. 
Here, we first generated data with  \(\alpha_0 = 0.5\), \(\alpha_1 = 25\), \(\alpha_2 = 8\), \(\sigma_u^2 = 10\), \(b = 0.05\), and \(\sigma_v^2 = 1\); then we used {\smcsiwae} to estimate $\alpha_0$, $b$, and $\sigma_u$ and simultaneously adapted the particle filter proposal. The interesting aspect of this model is that under certain parameterisations---like the one given---the locally optimal proposal is bimodal, with one dominating mode and the other one almost negligible. In these scenarios, the bootstrap proposal tends to be too diffuse, resulting in many wasted samples. Thus, we run the {\smcsiwae} algorithm to estimate the unknown parameters of the model while simultaneously learning a better proposal distribution. We design a family of proposals that integrate new parameters with the ones of the model. This is done by letting again $\propdens{\thp}(\cdot\mid x_t, y_{t+1})$ be Gaussian with mean and variance parameterised by neural networks; however, at each time $t\in\nset$, in addition to the new observation $y_{t+1}$, we input the mean $a_{t}(x_t)$ instead of the current state $x_t$.


Figure~\ref{fig:kernel_3}  illustrates the progression of the proposal distribution. Initially, after a few thousand iterations, the learned proposal starts to approximate the locally optimal kernel, despite the model parameters not yet being fully learned. As the optimal kernel converges to reflect the true parameters, our Gaussian proposal accurately matches the dominant mode. In contrast, the prior kernel of the standard bootstrap particle filter remains overly dispersed.


\begin{figure}[htb]
    \centering
    \includegraphics[width=0.49\textwidth]{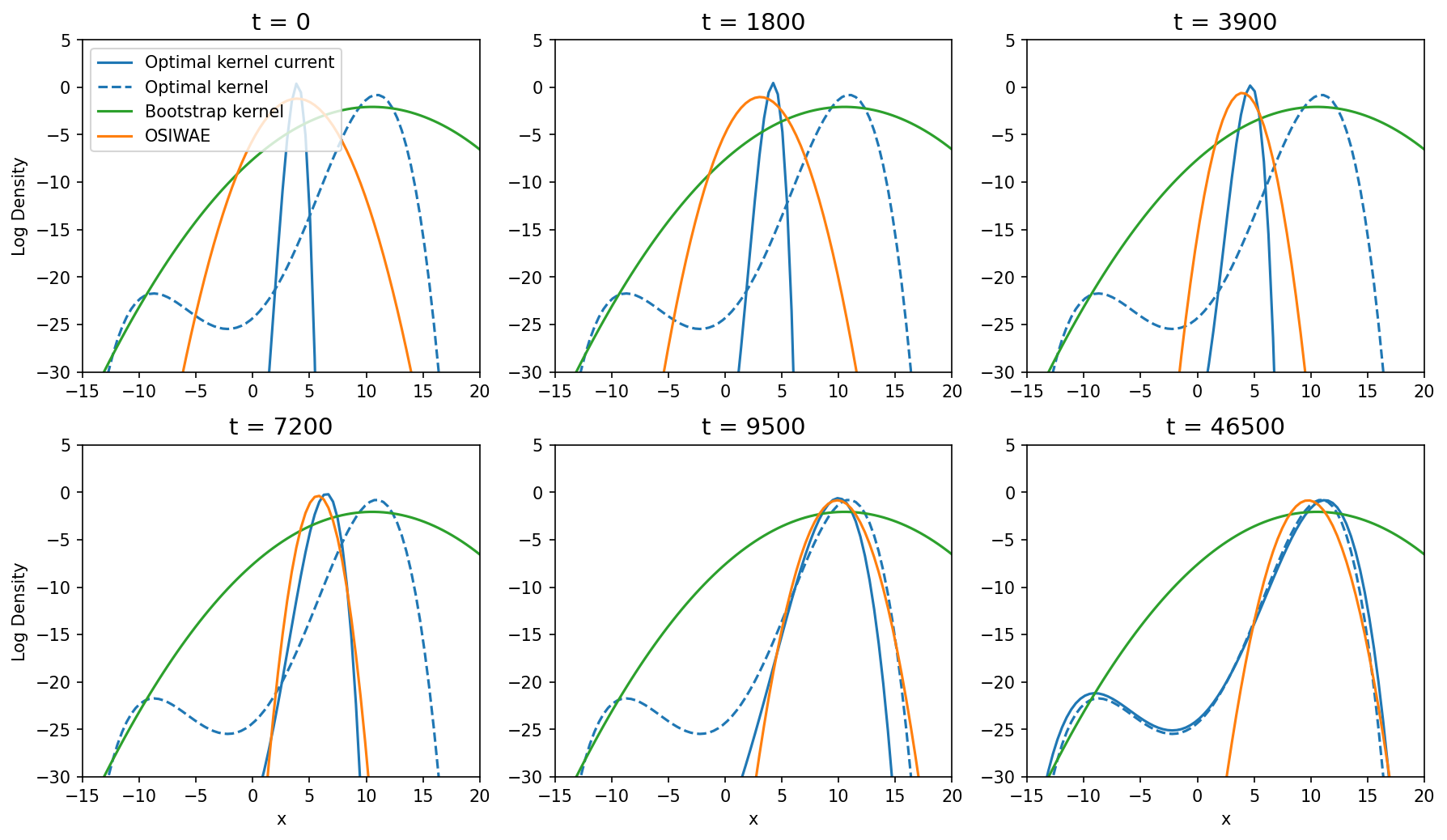}
    \caption{Log-densities of the learned proposal, the optimal kernel parameterised by the current parameter fit as well as the true parameters, and the prior kernel with true parameters. {\smcsiwae} uses 1000 particles and $M = 1000$. The Gaussian proposal $\propdens{\thp}$ is parameterised by two distinct neural networks, with one hidden layer of 12 nodes each, modelling the mean and the variance of the same. In each plot, \(x_t = 0.1\) and \(y_{t+1} = 6 \).}
    \label{fig:kernel_3}
\end{figure}

\section{Conclusion}

We have introduced {\siwae}, a method for recursively optimising an asymptotic {\iwae}-type variational objective in SSMs. {\siwae} is equipped with theoretical results describing its objective and its inherent $\mathcal{O}(M^{-1})$ bias with respect to the asymptotic contrast. By using particle methods, we obtain a practically implementable version, {\smcsiwae}, which can be viewed as an extension of particle-based {\RML} that also allows online training of the particle proposal. Our algorithm also sheds theoretical light on the recently proposed {\OVSMC}, which lacks theoretical underpinnings due to the \emph{ad hoc} truncation of its target gradient. As future research, we intend to provide {\smcsiwae} with a theoretical analysis akin to that of \cite{tadic:doucet:2018} for particle-based {\RML}.

\section*{Acknowledgments}
This work is supported by the Swedish Research Council, grant 2018-05230,
and by the Wallenberg AI, Autonomous Systems and Software Program (WASP) \emph{Online learning in dynamical generative models}.

\onecolumn

\appendix

\section{Proofs of Section~\ref{sec:theory}}\label{app:proofs}
In this appendix we present detailed proofs of the results discussed in the main body of the article. We begin with an overview of the structure of the appendix structure to facilitate navigation.
\begin{itemize}
	\item In Section~\ref{subsec:notation} we introduce the notation, which includes some measure-theoretic formalism that is not present in the main body of paper.
	\item In Section~\ref{subsec:forgetting} we prove the exponential forgetting of filter and tangent-filter measures (Proposition~\ref{prop:forg_filt} and Propostion~\ref{prop:forg_tang}), as a special case of the results of \citet{tadic:doucet:2005}.
	\item In Section~\ref{subsec:ergodicity} we introduce the extended Markov chain, comprising the data generating process, the filter measure associated to the observations and its gradient, and its Markov kernel $\tz{\thp}$. In addition, for this chain, we establish its ergodicity and a strong law of large numbers for a class of objective functions (Proposition~\ref{prop:ergod} and Proposition~\ref{prop:slln}).
	\item In Section~\ref{subsec:mean_field}, using the previously established strong law of large number, we define the objective functions $\lyap(\thp)$ and $\contr(\thp)$ as well as their  gradients (Proposition~\ref{prop:all_conv_slln}, which proves Proposition~3.1). 
	\item In Section~\ref{subsec:bias} we study the bias of $\lyap(\thp)$ and $\nabla \lyap(\thp)$ with respect to $\contr(\thp)$ and $\nabla \contr(\thp)$  (Corollary~\ref{prop:biascontr} and Corollary~\ref{cor:biasmf}, which prove Proposition~3.2 and Theorem~3.3, respectively).
	\item In Section~\ref{subsec:aux}, we prove some auxiliary lemmas that are used in previous sections.
\end{itemize}

\subsection{Notation}\label{subsec:notation}
\label{sec:notation}
We let $\rsetnonneg$ and $\rsetpos$ be the sets of nonnegative and positive real numbers, respectively. For $m\le n\in\nset$, we denote $x_{m:n}\eqdef(x_m,x_{m+1},\dots,x_{n-1},x_n)$ or, alternatively, $x^{m:n}\eqdef(x^m,x^{m+1},\dots,x^{n-1},x^n)$, depending on the specific case. If $m>n$, then by convention $x_{m:n}=x^{m:n}=\emptyset$, $\prod_{i=m}^{n}=1$, and $\sum_{i=m}^{n}=0$. For any vector $x_{1:d}\in\rset^d$, $d\in\nsetpos$, we indicate with $\norm{\cdot}$ the maximum norm, \ie, $\norm{x_{1:d}}=\max_{i \in \{1,\dots,d\}}\abs{x_i}$. For some general state space $(\set{S}, \alg{S})$ we let $\bmf{\alg{S}}$ be the set of real Borel-measurable functions on $\set{S}$ and $\1{\set{S}}\in \bmf{\set{S}}$ be the constant function equal to one on the whole $\set{S}$. We let $\meas{\alg{S}}$ be the set of finite measures on $\alg{S}$ and $\probmeas{\alg{S}}\subset\meas{\alg{S}}$ the set of probability measures. The set of finite signed measures on $\alg{S}$ is denoted by $\signmeas{\alg{S}}\supset\meas{\alg{S}}$. 
For $\mu\in\signmeas{\alg{S}}$, we denote by $\abs{\mu}=\mu^++\mu^-$ its total variation, where $ \mu^+\in\meas{\alg{S}}$ and $\mu^-\in\meas{\alg{S}}$ are the positive and negative parts of $\mu$, respectively, \emph{i.e.}, $\mu = \mu^+-\mu^-$. We let $\tvnorm{\mu}=\abs{\sigma}(\set{S})$ be the total variation norm of $\mu$. We denote by $\probmeasalg{\alg{S}}$, $\measalg{\alg{S}}$ and $\signmeasalg{\alg{S}}$ the sigma-fields of $\probmeas{\alg{S}}$, $\meas{\alg{S}}$, and $\signmeas{\alg{S}}$, respectively, induced by the total variation norm. For every integer $p \in \nsetpos$, we also define the product space 
\begin{equation}\label{eq:signmeasp}
	(\signmeasp{\alg{X}}, \signmeaspalg{\alg{X}})=(\underbrace{\signmeas{\alg{X}}\times\cdots\times \signmeas{\alg{X}}}_{p \text{ times}},\underbrace{\signmeasalg{\alg{X}}\tensprod\cdots\tensprod \signmeasalg{\alg{X}}}_{p \text{ times}}).
\end{equation}
For every $f\in\bmf{\alg{S}}$ and $\boldsymbol{\mu}=(\mu_1,\dots,\mu_p)\in\signmeasp{\alg{S}}$, we denote
\begin{equation}
	\boldsymbol{\mu}f=\int f(s)\,\boldsymbol{\mu}(ds)=(\mu_1f,\dots,\mu_pf)\in\rset^p,
\end{equation}
and, by convention, we still denote with $\tvnorm{\boldsymbol{\mu}}$ the maximum norm of the vector of total variation norms, 
\emph{i.e.},
\begin{equation}
	\tvnorm{\boldsymbol{\mu}}=\max_{i \in \{1,\dots,p\}}\tvnorm{\mu_i}.
\end{equation}
Given some state spaces $(\set{S},\alg{S})$, $(\set{S}',\alg{S}')$, and $(\set{S}'',\alg{S}'')$ and two kernels $\kernel{K}_1:\set{S}\times \alg{S}'\to \rsetnonneg$ and $\kernel{K}_2:\set{S}'\times \alg{S}''\to \rsetnonneg$, we may define new product Markov kernels by, first, the tensor product $\kernel{K}_1\tensprod\kernel{K}_2:\set{S}\times (\alg{S}'\tensprod \alg{S}'')\to \rsetnonneg$ given by, for 
$s \in \set{S}$
and $f\in\bmf{\alg{S}'\tensprod\alg{S}''}$,
\begin{equation}
	(\kernel{K}_1\tensprod\kernel{K}_2)f(s)= \iint f(s',s'') 
    \,\kernel{K}_1(s,ds')\,\kernel{K}_2(s',ds'')
\end{equation}
and, second, the standard product $\kernel{K}_1\kernel{K}_2:\set{S}\times \alg{S}''\to \rsetnonneg$ given by, for 
$s \in \set{S}$
and $f\in\bmf{\alg{S}''}$,
\begin{equation}
	\kernel{K}_1\kernel{K}_2f(s)= \iint f(s'')\, 
 \kernel{K}_1(s,ds')\,\kernel{K}_2(s',ds'').
\end{equation}
Similarly, for $\mu\in\meas{\alg{S}}$ and  $\kernel{K}:\set{S}\times \alg{S}'\to \rsetnonneg$, we define $\mu\tensprod\kernel{K}\in\meas{\alg{S}\tensprod\alg{S}'}$ and $\mu\kernel{K}\in\meas{\alg{S}'}$ such that for $f_1\in\bmf{\alg{S}\tensprod\alg{S}'}$ and $f_2\in\bmf{\alg{S}'}$ we have
\begin{align}
	(\mu\tensprod\kernel{K})f_1&= \iint f_1(s,s') 
	\mu(ds)\,\kernel{K}(s,ds'),
	\\\mu\kernel{K}f_2&= \int f_2(s') 
	\int\mu(ds)\,\kernel{K}(s,ds').
\end{align}
Moreover, for $\mu\in\meas{\alg{S}}$ and $\mu'\in\meas{\alg{S}'}$, we denote by $\mu\tensprod\mu'\in \meas{\alg{S}\tensprod\alg{S}'}$ the standard measure product given by, for  $f\in\bmf{\alg{S}\tensprod\alg{S}'}$,
\begin{equation}
	(\mu\tensprod\mu') f=\iint f(s,s')\,\mu(ds)\,\mu'(ds')
\end{equation}
and by $\mu^{\tensprod k}\in \meas{\alg{S}^{\tensprod k}}$, $k\in\nsetpos$, the generalised product given by, for $f\in\bmf{\alg{S}^{\tensprod k}}$, 
\begin{equation}
	\mu^{\tensprod k}f=\idotsint f(s_1,\dots,s_k)\, \prod_{m = 1}^k \mu(ds_m).
\end{equation}
We assume that all random variables are defined on a common probability space $(\set{\Omega},\alg{F},\prob)$ and consider state and observation spaces $\set{X}\subseteq\rset^{d_x}$ and $\set{Y}\subseteq\rset^{d_y}$, respectively, where $(d_x,d_y)\in\nsetpos^2$. The SSM under consideration is a bivariate Markov chain $(X_t, Y_t)_{t \in \mathbb{N}}$ evolving on $(\set{X} \times \set{Y}, \alg{X} \tensprod \alg{Y})$ according to a dynamics governed by a parametric model, with parameter $\thp=(\thp^1,\dots,\thp^p)\in\parspace\subseteq\rset^p$, $p\in\nsetpos$, where $\parspace$ is some parameter space. The Markov transition kernel of the model is
\begin{equation}\label{def:sker}
	\kernel{S}_\thp : (\set{X}\times\set{Y}) \times (\alg{X} \tensprod \alg{Y}) \ni ((x,y), A) \mapsto \iint \1{A}(x',y') \, \hidker{\thp}(x,dx') \,\emker{\thp}(x', dy'), 
\end{equation}
where we have introduced the Markov kernels
\begin{align}
	\hidker{\thp} : \set{X} \times \alg{X} &\ni (x, A) \mapsto \int  \1{A}(x') \, \hiddens{\thp}(x, x') \, \refm(dx'), \\
	\emker{\thp} : \set{X} \times \alg{Y} &\ni (x, B) \mapsto \int  \1{B}(y) \, \emdens{\thp}(x, y) \, \refg(dy),
\end{align}
with $\hiddens{\thp}: \set{X} \times \set{X} \to \rsetnonneg$ and $\emdens{\thp}:\set{X}\times\set{Y}\to\rsetnonneg$ being the state and emission transition densities with respect to the reference measures $\refm \in \meas{\alg{X}}$ and $\refg \in \meas{\alg{Y}}$. Here we have slightly modified the notation of the main paper, by using the short-hand notation $\hiddens{\thp}(x, x') = \hiddens{\thp}(x' \mid x)$ and $\emdens{\thp}(x, y) = \emdens{\thp}(y \mid x)$, and allowing a general reference measure instead of the Lebesgue measure.
The chain is initialised according to $\xinit\varotimes\emker{\thp}:\alg{X} \tensprod \alg{Y} \ni A \mapsto \int_A \xinit(dx) \,\emker{\thp}(x, dy)$, where $\xinit$ is some probability measure on $(\set{X}, \alg{X})$ having density $m_0(x)$ with respect to $\refm$. 

Given a sequence $(y_t)_{t\ge 0}$ of observations, we define, for each $t \in \nset$, the filter measure $\filt{\thp}{t}\in\probmeas{\alg{X}}$ which satisfies, for every $f\in\bmf{\alg{X}}$,
\begin{equation}\label{eq:filt_def}
	\filt{\thp}{t} f \eqdef \frac{\idotsint  f(x_t) \hiddens{0}(x_0)\emdens{\thp}( x_0, y_0)\prod_{t'=1}^{t}\hiddens{\thp}( x_{t'-1}, x_{t'})\emdens{\thp}(x_{t'}, y_{t'})\,\refm(dx_0)\cdots\refm(dx_{t})}{\idotsint \hiddens{0}(x_0)\emdens{\thp}( x_0, y_0)\prod_{t'=1}^{t}\hiddens{\thp}( x_{t'-1}, x_{t'})\emdens{\thp}(x_{t'}, y_{t'})\,\refm(dx_0)\cdots\refm(dx_{t})}.
\end{equation}
The corresponding filter derivative, or tangent filter, is given by $\tang{\thp}{t}f= \nabla_\thp\filt{\thp}{t} f$. 

We let $\propker{\thp}: \set{X} \times \set{Y} \times \alg{X} \to [0,1]$ be some \emph{proposal kernel}, parameterised by $\thp \in \parspace$ as well and having transition density $\propdens{\thp}:\set{X}\times\set{Y}\times\set{X}\to \rsetnonneg$ with respect to $\refm$. This proposal is assumed to be such that for every $(x, y, A) \in \set{X} \times \set{Y} \times \alg{X}$, 
$$
\propker{\thp}((x, y), A) = 0 \Rightarrow \int \1{A}(x') \emdens{\thp}(x', y) \, \hidker{\thp}(x, dx') = 0. 
$$ 
In order to express the {\siwae} samples as explicit differentiable functions of $\thp$, the proposal is assumed to be reparameterisable. More precisely, we assume that there exist some state-space $(\set{U},\alg{U})$, an easily samplable probability measure $\indmeas\in\probmeas{\alg{U}}$, not depending on $\thp$, and a function $\repfunc{\thp}: \set{X} \times \set{Y} \times \set{U}\to \set{X}$ such that for all $(x,y)\in\set{X}\times\set{Y}$ and $\thp \in \parspace$, it holds that $\int f(\repfunc{\thp}(x, y, \auxrvb)) \, \indmeas(d \auxrvb) = \int f(x') \, \propker{\thp} ((x,y),dx')$ for all bounded real-valued measurable functions $f$ on $\set{X}$; in other words, the pushforward distribution $\indmeas \circ \repfunc{\thp}^{-1}(x, y, \cdot)$ coincides with $\propker{\thp}((x, y), \cdot)$. 

On the basis of the proposal kernel, we redefine the reparameterised weight function 
\begin{equation}
	\wgtfunc{\thp}(x,y,\auxrvb)\eqdef\frac{ \hiddens{\thp}( x,\repfunc{\thp}(x, y, \auxrvb))\emdens{\thp}(\repfunc{\thp}(x, y, \auxrvb),y)}{\propdens{\thp}(x, y,\repfunc{\thp}(x,y, \auxrvb))},
\end{equation} 
for all $(x,y,\auxrvb)\in \set{X}\times\set{Y}\times \set{U}$ such that $\propdens{\thp}(x, y,\repfunc{\thp}(x,y, \auxrvb)) > 0$.

\subsection{Exponential forgetting of the filter and its derivative}\label{subsec:forgetting}

Our analysis relies on the following assumptions.
\begin{assumption}\label{assum:eq}
	There exists $\eq \in(0,1)$ such that for every $\thp\in \parspace$, $(x,x') \in \set{X}^2$, and $y\in\set{Y}$, 
	$$\eq\le \hiddens{\thp}(x,x')\le\eq^{-1},\qquad \eq\le\emdens{\thp}(x,y)\le\eq^{-1}.
	$$ 
\end{assumption}
Under Assumption~\ref{assum:eq} we define 
\begin{equation} \label{eq:def:rho}
    \rhoq\eqdef(1-\eq^4)/(1+\eq^4)\in(0,1).
\end{equation}
\begin{assumption}\label{assum:bg}
	There exists $\bg\in[1\infty)$ such that for every $\thp \in \parspace$, $(x,x')\in\set{X}^2$, and $y\in\set{Y}$,
	\begin{align}
  \norm{\nabla_\thp \hiddens{\thp}(x,x')} \vee \norm{\nabla_\thp \emdens{\thp}(x,y)}
  \le\bg.
	\end{align}
	
\end{assumption}
The strong mixing Assumptions~\ref{assum:eq}--\ref{assum:bg} are standard in the literature and point to applications where the state and parameter spaces are compact.
\begin{remark} \label{rem:mu:probability}
	Note that by Assumption~\ref{assum:eq} it follows that the reference measure $\refm$ is a finite measure on $(\set{X}, \alg{X})$. Thus, without loss of generality we may assume that $\refm$ is a probability measure. 
\end{remark}
\begin{definition}
	For $y\in\set{Y}$, let the Markov kernel $\lker{\thp}{y}:\set{X}\times \alg{X}\to \rsetnonneg$ and the signed kernel $\lkertang{\thp}{y}:\set{X}\times \alg{X}\to \rset^p$ be defined by, for 
 $x\in\set{X}$ and $f\in \bmf{\alg{X}}$,
	\begin{align}
        \lker{\thp}{y}f(x)
  &\eqdef \int f(x')\emdens{\thp}(x',y)\hiddens{\thp}(x,x')\,\refm(dx')
	\end{align}
	and
	\begin{equation}
        \lkertang{\thp}{y}f(x) \eqdef  \int f(x')\nabla_\thp\left\{\emdens{\thp}(x',y)\hiddens{\thp}(x,x')\right\}\,
        \refm(dx').
	\end{equation}
\end{definition}
\begin{definition}\label{def:recursions}
	For $\thp\in\parspace$ 
  let the mappings $\recf{\thp}:\meas{\alg{X}}\times \set{Y}\to \meas{\alg{X}}$ and $\rect{\thp}:\meas{\alg{X}}\times \signmeasp{\alg{X}}\times \set{Y}\to \signmeasp{\alg{X}}$ be given by 
	\begin{equation}
		\recf{\thp}(\mu, y)f=\frac{\mu\lker{\thp}{y}f}{\mu\lker{\thp}{y}\1{\set{X}}}
	\end{equation}
	and
	\begin{equation}\label{eq:tang_rec}
		\rect{\thp}(\mu,\mutilde,y)f=\frac{\mutilde\lker{\thp}{y}f-\mutilde\lker{\thp}{y}\1{\set{X}}\recf{\thp}(\mu, y)f+\mu\lkertang{\thp}{y}f-\mu\lkertang{\thp}{y}\1{\set{X}}\recf{\thp}(\mu, y)f}{\mu\lker{\thp}{y}\1{\set{X}}}.
	\end{equation}
	Moreover, for $\thp\in\parspace$ and a sequence $(y_t)_{t \in \nsetpos}$ in $\set{Y}$ 
 we may define recursively, for $t \in \nsetpos$, 
 the composite mappings 
 $\recf{\thp}^t:\meas{\alg{X}}\times \set{Y}^{t}\to \meas{\alg{X}}$ and $\rect{\thp}^t:\meas{\alg{X}}\times \signmeas{\alg{X}}\times \set{Y}^{t}\to \signmeasp{\alg{X}}$
 by
	\begin{align}
		\recf{\thp}^t(\mu,y_{1:t})&\eqdef\recf{\thp}(\recf{\thp}^{t-1}(\mu,y_{1:t-1}),y_{t}), \\
		\rect{\thp}^t(\mu,\mutilde,y_{1:t})
  &\eqdef\rect{\thp}(\recf{\thp}^{t-1}(\mu,y_{1:t-1}),\rect{\thp}^{t-1}(\mu,\mutilde,y_{1:t-1}),y_{t}). 
	\end{align}
	By convention we let $\recf{\thp}^{0}(\mu,y)=\recf{\thp}^{0}(\mu)=\mu$ and $\rect{\thp}^{0}(\mu,\mutilde, y)=\rect{\thp}^{0}(\mutilde)=\mutilde$.  
\end{definition}

The following lemma relates the compositions of $\recf{\thp}$ and $\rect{\thp}$.
\begin{lemma}\label{lemma:filt_tang}
	Assume that $\mu^\thp\in \meas{\alg{X}}$ is absolutely continuous with respect to the reference measure $\refm$, with a density being differentiable with respect to the parameter $\thp$. Let $\mutilde^\thp\in\signmeasp{\alg{X}}$ be the tangent-filter measure of $\mu^\thp$.
 Then for every $t\in \nsetpos$, $y_{1:t} \in \set{Y}^t$, and $f\in\bmf{\alg{X}}$,  
	\begin{equation}
		\nabla_{\thp}\recf{\thp}^t(\mu^\thp,y_{1:t})f=\rect{\thp}^t(\mu^\thp,\mutilde^\thp,y_{1:t})f.
	\end{equation}
\end{lemma}
\begin{proof}
	We proceed by induction, proving first the claim for in the base case $t=1$. We immediately see that
	\begin{multline}
    \label{eq:ind_recf_rect}
		\nabla_{\thp}\recf{\thp}(\mu^\thp,y_{1})f=\nabla_{\thp}\frac{\mu^\thp\lker{\thp}{y_1}f}{\mu^\thp\lker{\thp}{y_1}\1{\set{X}}} 
		\\=\frac{\mutilde^\thp\lker{\thp}{y_1}f-\mutilde^\thp\lker{\thp}{y_1}\1{\set{X}}\recf{\thp}(\mu^\thp, y_1)f+\mu^\thp\lkertang{\thp}{y_1}f-\mu^\thp\lkertang{\thp}{y_1}\1{\set{X}}\recf{\thp}(\mu^\thp, y_1)f}{\mu^\thp\lker{\thp}{y_1}\1{\set{X}}}
		=\rect{\thp}(\mu^\thp,\mutilde^\thp,y_{1})f.
    \end{multline}
	Now assume that the claim is true for some $t \in \nsetpos$ and let $\mu_t^\thp= \recf{\thp}^t(\mu^\thp,y_{1:t})$ and $\mutilde_t^\thp = \rect{\thp}^t(\mu^\thp,\mutilde^\thp,y_{1:t})$, so that, by the induction hypothesis, $\nabla_{\thp}\mu_t^\thp f=\mutilde_t^\thp f$. Thus,
	\begin{equation}
		\nabla_{\thp}\recf{\thp}^{t+1}(\mu^\thp,y_{1:t+1})f=\nabla_\thp\recf{\thp}(\mu_t^\thp,y_{t+1})f=\rect{\thp}(\mu_t^\thp,\mutilde_t^\thp,y_{t+1})f,
	\end{equation}
	where we used, first, the recursive definition for $\recf{\thp}^{t+1}$, then the induction hypothesis together with \eqref{eq:ind_recf_rect}. The proof is completed by noting that 
	\begin{equation}
		\rect{\thp}(\mu_t^\thp,\mutilde_t^\thp,y_{t+1})f= \rect{\thp}(\recf{\thp}^t(\mu^\thp,y_{1:t}),\rect{\thp}^t(\mu^\thp,\mutilde^\thp,y_{1:t}),y_{t+1})f=\rect{\thp}^{t+1}(\mu^\thp,\mutilde^\thp,y_{1:t+1})f.
	\end{equation}
\end{proof}

\begin{remark}
	An immediate consequence of Lemma~\ref{lemma:filt_tang} is that for $t \in \nset$, it holds $\filt{\thp}{t}=\recf{\thp}(\filt{\thp}{t-1}, y_{t})=\recf{\thp}^{t}(\filt{\thp}{0}, y_{1:t})$ and $\tang{\thp}{t}=\rect{\thp}(\filt{\thp}{t-1},\tang{\thp}{t-1}, y_{t})=\rect{\thp}^{t}(\filt{\thp}{0},\tang{\thp}{0}, y_{1:t})$.
\end{remark}

In the following, let $d$ 
be the Hilbert distance between elements in $\meas{\alg{X}}$, defined as
	\begin{equation}
		\hilbdist(\mu,\mu')\eqdef\log  \frac{\sup_{B\in\alg{X}:\mu'(B)>0} \mu(B)/\mu'(B)}{\inf_{B\in\alg{X}:\mu'(B)>0} \mu(B)/\mu'(B)},
	\end{equation}
	where is it assumed there exist $0<a\le b$ such that $a\mu(B)\le \mu'(B)\le b\mu(B)$ for all $B\in\alg{X}$; otherwise $\hilbdist(\mu,\mu')=\infty$. 
\begin{lemma}\label{lemma:hilbdist}
	Let Assumption~\ref{assum:eq} hold. 
    Then for every $(\mu,\mu') \in\probmeas{\alg{X}}^2$, $y\in\set{Y}$, and $\thp\in\parspace$, 
	\begin{align}
		&\tvnorm{\mu-\mu'}\le \frac{2}{\log 3}\hilbdist(\mu,\mu'),\label{eq:hilb1}
		\\&\hilbdist(\mu \lker{\thp}{y},\mu' \lker{\thp}{y})\le \eq^{-4} \tvnorm{\mu-\mu'},\label{eq:hilb2}
		\\&\hilbdist(\mu \lker{\thp}{y},\mu' \lker{\thp}{y})\le \frac{1-\eq^4}{1+\eq^4}\hilbdist(\mu,\mu').\label{eq:hilb3}
	\end{align}
\end{lemma}
\begin{proof}
	The proof of \eqref{eq:hilb1} can be found in \cite[Lemma~1]{atar:zeitouni:1997}, while the proofs of \eqref{eq:hilb2} and  \eqref{eq:hilb3} can be found in \cite[Lemma~3.4 and Proposition~3.9(i)]{legland:oudjane:2004}.
\end{proof}
We are now ready to state a first---now classical---result on the exponential forgetting of the filter.
\begin{proposition}[Forgetting of the filter]\label{prop:forg_filt}
	Let Assumption~\ref{assum:eq} hold. Then there exists $\kq>1$, depending on $\eq$ only, such that for every $t\in\nset$, $y_{1:t} \in \set{Y}^t$, $\thp\in\parspace$, and $(\mu,\mu')\in \probmeas{\alg{X}}^2$, 
	\begin{equation}
		\tvnorm{\recf{\thp}^{t}(\mu, y_{1:t})-\recf{\thp}^{t}(\mu', y_{1:t})}\le \kq\rhoq^t\tvnorm{\mu-\mu'},
	\end{equation}
	  where $\rhoq$ is defined in  \eqref{eq:def:rho}
\end{proposition}
\begin{proof}
	For all $t \in \nset$, let 
 $\mu_t\eqdef \recf{\thp}^{t}(\mu, y_{1:t})$ and $\mu_t'\eqdef \recf{\thp}^{t}(\mu', y_{1:t})$. Note that the Hilbert distance is invariant under multiplication by	positive scalars, \ie, $\hilbdist(\mu_{t+1},\mu_{t+1}')=\hilbdist(\mu_t \lker{\thp}{y_{t+1}},\mu_t' \lker{\thp}{y_{t+1}})$. Hence, Lemma~\ref{lemma:hilbdist} implies that 
	\begin{align}
		&\tvnorm{\mu_t-\mu_t'}\le \frac{2}{\log 3}\hilbdist(\mu_t,\mu_t'),
		\\&\hilbdist(\mu_{t+1},\mu_{t+1}')\le \eq^{-4} \tvnorm{\mu_t-\mu_t'},
		\\&\hilbdist(\mu_{t+1},\mu_{t+1}')\le \frac{1-\eq^4}{1+\eq^4}\hilbdist(\mu_{t},\mu_{t}'),
	\end{align}
	which in turn implies that 
	\begin{equation}
		\tvnorm{\mu_t-\mu_t'}\le \frac{2}{\log 3}\left(\frac{1-\eq^4}{1+\eq^4}\right)^{t-1}\hilbdist(\mu_1,\mu_1')
		\le \frac{2\eq^{-4}}{\log 3}\left(\frac{1-\eq^4}{1+\eq^4}\right)^{t-1}\tvnorm{\mu-\mu'}.
	\end{equation}
	The proof is now concluded recalling the definition of $\rhoq$ and letting $\kq\eqdef 2\eq^{-4}\rhoq^{-1}/\log 3$.
\end{proof}

The following lemmas are instrumental for the proof of the forgetting of the tangent filter established in Proposition~\ref{prop:forg_tang}.
\begin{lemma}\label{lemma:bound_lkertang}
	Let Assumptions~\ref{assum:eq}--\ref{assum:bg} hold. Then for all $\thp\in\parspace$, $y\in\set{Y}$, and all $\tilde{\mu}\in\signmeas{\alg{X}}$ such that $\tvnorm{\tilde{\mu}}<\infty$, 
	\begin{equation}
		\tvnorm{\tilde{\mu}\lkertang{\thp}{y}}\le 2\bg\eq^{-1}\tvnorm{\tilde{\mu}}.
	\end{equation}
\end{lemma}
\begin{proof}
	Let $\mu^+$ and $\mu^-$ be the positive and negative parts of $\tilde{\mu}$, respectively. Then
	\begin{equation}
		\tvnorm{\tilde{\mu}\lkertang{\thp}{y}}
		\le\tvnorm{\mu^+\lkertang{\thp}{y}}+\tvnorm{\mu^-\lkertang{\thp}{y}}.
	\end{equation}
	Using Assumptions~\ref{assum:eq}--\ref{assum:bg}, note that $\mu^\pm\lkertang{\thp}{y}\in\signmeasp{\alg{X}}$ and 
	\begin{equation}
		\tvnorm{\mu^\pm\lkertang{\thp}{y}}
		\le \int \mu^\pm(dx)\,\left(\hiddens{\thp}(x,x')\norm{\nabla_{\thp}\emdens{\thp}(x',y)}+
		\emdens{\thp}(x',y)\norm{\nabla_{\thp}\hiddens{\thp}(x,x')}\right)\,\refm(dx')
		\le 2\bg \eq^{-1}\tvnorm{\mu^\pm},
	\end{equation}
	which implies
	\begin{equation}
		\tvnorm{\tilde{\mu}\lkertang{\thp}{y}}\le2\bg \eq^{-1}(\tvnorm{\mu^+}+\tvnorm{\mu^-})=2\bg \eq^{-1}\tvnorm{\tilde{\mu}}.
	\end{equation}
\end{proof}

\begin{definition}\label{def:lkert_gtile_htilde}
	For every $\thp\in\parspace$ and $(y_t)_{t\in\nsetpos}\in\set{Y}$, define recursively  
	\begin{equation}
		\lkert{\thp}{t+1}{y_{1:t+1}}=\lkert{\thp}{t}{y_{1:t}}\lker{\thp}{y_{t+1}}, \quad t\in\nset,
	\end{equation}
	with the convention $\lkert{\thp}{0}{y_{0}}f(x)=f(x)$. In addition, for every $\thp\in\parspace$, define the mappings $\gtilde{\thp}{t}: \meas{\alg{X}}\times\signmeasp{\alg{X}}\times \set{Y}^t\to\signmeasp{\alg{X}}$, $t\in\nset$, and $\htilde{\thp}{t}: \meas{\alg{X}}\times\set{Y}^t\to\signmeasp{\alg{X}}$,  $t\in\nsetpos$, by $\gtilde{\thp}{0}(\mu,\mutilde, y_{0}) \eqdef \mutilde$ and, for $f\in\bmf{\alg{X}}$,
	\begin{align}
		\gtilde{\thp}{t}(\mu,\mutilde, y_{1:t})f&\eqdef\frac{(\mutilde\lkert{\thp}{t}{y_{1:t}}-\mutilde\lkert{\thp}{t}{y_{1:t}}\1{\set{X}}\recf{\thp}^{t}(\mu,y_{1:t}))f}{\mu \lkert{\thp}{t}{y_{1:t}}\1{\set{X}}}
		\\\htilde{\thp}{t}(\mu, y_{1:t})f&\eqdef\frac{(\recf{\thp}^{t-1}(\mu,y_{1:t-1})\lkertang{\thp}{y_t}-\recf{\thp}^{t-1}(\mu,y_{1:t-1})\lkertang{\thp}{y_t}\1{\set{X}}\recf{\thp}^{t}(\mu,y_{1:t}))f}{\recf{\thp}^{t-1}(\mu,y_{1:t-1})\lker{\thp}{y_t}\1{\set{X}} }.
	\end{align}
\end{definition}

\begin{lemma}\label{lemma:gtil_htil}
	For every $\thp\in\parspace$, $(y_t)_{t \in \nsetpos}\in\set{Y}$, $\mu\in\probmeas{\alg{X}}$, and $\mutilde\in\signmeasp{\alg{X}}$ it holds that 
	\begin{equation}
		\rect{\thp}^t(\mu,\mutilde, y_{1:t}) = \gtilde{\thp}{t}(\mu,\mutilde,y_{1:t})+\sum_{s=1}^{t}\gtilde{\thp}{t-s}(\mu_s,\htilde{\thp}{s}(\mu, y_{1:s}), y_{s+1:t}), \quad t\in\nset.
	\end{equation}

\end{lemma}
\begin{proof}
	The base case $t=0$ is trivially true. Now we assume that the claim is true for some $t \in \nset$ and proceed by induction. Let $\mu_t = \recf{\thp}^t(\mu,y_{1:t})\in\probmeas{\alg{X}}$ and $\mutilde_{t}=\rect{\thp}^t(\mu,\mutilde, y_{1:t})\in\signmeasp{\alg{X}}$, for $t\ge 0$. We write
	\begin{multline}
		\mutilde_{t+1} = \rect{\thp}(\mu_t, \mutilde_t, y_{t+1})=\frac{\mutilde_t\left(\lker{\thp}{y_{t+1}}-\lker{\thp}{y_{t+1}}\1{\set{X}}\mu_{t+1}\right)}{\mu_t\lker{\thp}{y_{t+1}}\1{\set{X}}}
		+\frac{\mu_t\lkertang{\thp}{y_{t+1}}-\mu_t\lkertang{\thp}{y_{t+1}}\1{\set{X}}\mu_{t+1}}{\mu_t\lker{\thp}{y_{t+1}}\1{\set{X}}}
		\\=\frac{\mutilde_t\left(\lker{\thp}{y_{t+1}}-\lker{\thp}{y_{t+1}}\1{\set{X}}\mu_{t+1}\right)}{\mu_t\lker{\thp}{y_{t+1}}\1{\set{X}}}+\htilde{\thp}{t+1}(\mu, y_{1:t+1}),
	\end{multline}
	where we used, first, the recursion \eqref{eq:tang_rec} first and, second, Definition~\ref{def:lkert_gtile_htilde} of $\htilde{\thp}{t}$. Now, we apply the induction hypothesis to $\mutilde_t$, noticing first that for every $\tilde{\boldsymbol{\nu}}\in\signmeasp{\alg{X}}$ and $0\le s<t$,
	\begin{align}
		\lefteqn{\frac{\gtilde{\thp}{t-s}(\mu_s, \tilde{\boldsymbol{\nu}}, y_{s+1:t})\left(\lker{\thp}{y_{t+1}}-\lker{\thp}{y_{t+1}}\1{\set{X}}\mu_{t+1}\right)}{\mu_t\lker{\thp}{y_{t+1}}\1{\set{X}}}} \hspace{10mm}
		\\&= \frac{\left(\tilde{\boldsymbol{\nu}}\lkert{\thp}{t-s}{y_{s+1:t}}-\tilde{\boldsymbol{\nu}}\lkert{\thp}{t-s}{y_{s+1:t}}\1{\set{X}}\recf{\thp}^{t-s}(\mu_s,y_{s+1:t})\right)\left(\lker{\thp}{y_{t+1}}-\lker{\thp}{y_{t+1}}\1{\set{X}}\mu_{t+1}\right)}{\mu_s \lkert{\thp}{t-s}{y_{s+1:t}}\1{\set{X}}\mu_t\lker{\thp}{y_{t+1}}\1{\set{X}}}
		\\&=\frac{\tilde{\boldsymbol{\nu}}\lkert{\thp}{t+1-s}{y_{s+1:t+1}}-\tilde{\boldsymbol{\nu}}\lkert{\thp}{t+1-s}{y_{s+1:t+1}}\1{\set{X}}\recf{\thp}^{t+1-s}(\mu_s,y_{s+1:t+1})}{\mu_s \lkert{\thp}{t+1-s}{y_{s+1:t+1}}\1{\set{X}}}
		\\&\hspace{1cm}-\frac{\tilde{\boldsymbol{\nu}}\lkert{\thp}{t-s}{y_{s+1:t}}\1{\set{X}}(\mu_{t}\lker{\thp}{y_{t+1}}- \mu_t\lker{\thp}{y_{t+1}}\1{\set{X}}\mu_{t+1})}{\mu_s \lkert{\thp}{t-s}{y_{s+1:t}}\1{\set{X}}\mu_t\lker{\thp}{y_{t+1}}\1{\set{X}}}
		\\&=\gtilde{\thp}{t+1-s}(\mu_s, \tilde{\boldsymbol{\nu}}, y_{s+1:t+1})-\frac{\tilde{\boldsymbol{\nu}}\lkert{\thp}{t-s}{y_{s+1:t}}\1{\set{X}}(\mu_{t+1}- \mu_{t+1})}{\mu_s \lkert{\thp}{t-s}{y_{s+1:t}}\1{\set{X}}}
		=\gtilde{\thp}{t+1-s}(\mu_s, \tilde{\boldsymbol{\nu}}, y_{s+1:t+1}).
	\end{align}
	Thus, using the previous relation and the induction step for $\mutilde_t$ we obtain  
	\begin{align}
		\mutilde_{t+1} &=\frac{\mutilde_t\left(\lker{\thp}{y_{t+1}}-\lker{\thp}{y_{t+1}}\1{\set{X}}\mu_{t+1}\right)}{\mu_t\lker{\thp}{y_{t+1}}\1{\set{X}}}+\htilde{\thp}{t+1}(\mu, y_{1:t+1})
		\\&=\gtilde{\thp}{t}(\mu,\mutilde, y_{1:t})	\frac{\lker{\thp}{y_{t+1}}-\lker{\thp}{y_{t+1}}\1{\set{X}}\mu_{t+1}}{\mu_t\lker{\thp}{y_{t+1}}\1{\set{X}}}
		\\&\hspace{1cm}+\sum_{s=1}^{t}\gtilde{\thp}{t-s}(\mu_s,\htilde{\thp}{s}(\mu, y_{1:s}), y_{s+1:t})	\frac{\lker{\thp}{y_{t+1}}-\lker{\thp}{y_{t+1}}\1{\set{X}}\mu_{t+1}}{\mu_t\lker{\thp}{y_{t+1}}\1{\set{X}}} +\htilde{\thp}{t+1}(\mu, y_{1:t+1})
		\\&=\gtilde{\thp}{t+1}(\mu,\mutilde, y_{1:t+1})+\sum_{s=1}^{t}\gtilde{\thp}{t+1-s}(\mu_s,\htilde{\thp}{s}(\mu, y_{1:s}), y_{s+1:t+1})+\htilde{\thp}{t+1}(\mu, y_{1:t+1})
		\\&=\gtilde{\thp}{t+1}(\mu,\mutilde, y_{1:t+1})+\sum_{s=1}^{t+1}\gtilde{\thp}{t+1-s}(\mu_s,\htilde{\thp}{s}(\mu, y_{1:s}), y_{s+1:t+1}),
	\end{align}
	which proves the claim.
\end{proof}
\begin{lemma}\label{lemma:gtil_forg}
	Let Assumption~\ref{assum:eq} hold. Then for every $t\in\nset$, 
 $y_{1:t} \in \set{Y}^t$, 
 $\thp\in\parspace$, $(\mu, \mu')\in\probmeas{\alg{X}}^2$, and $(\mutilde,\mutilde') \in\signmeasp{\alg{X}}^2$, 
	\begin{align}
		\text{(i) }&\tvnorm{\gtilde{\thp}{t}(\mu,\mutilde,y_{1:t})}\le2\eq^{-4}\kq\rhoq^t\tvnorm{\mutilde},
		\\\text{(ii) }&\tvnorm{\gtilde{\thp}{t}(\mu,\mutilde,y_{1:t})-\gtilde{\thp}{t}(\mu,\mutilde',y_{1:t})}\le2\eq^{-4}\kq\rhoq^t\tvnorm{\mutilde-\mutilde'},
		\\\text{(iii) }&\tvnorm{\gtilde{\thp}{t}(\mu,\mutilde,y_{1:t})-\gtilde{\thp}{t}(\mu',\mutilde,y_{1:t})}\le2\eq^{-8}\kq\rhoq^t \tvnorm{\mu-\mu'}\tvnorm{\mutilde}. 
	\end{align}
\end{lemma}
\begin{proof}
	We first assume that $p=1$, this is, $\mutilde=\tilde{\mu}\in\signmeas{\alg{X}}$, which implies that $\gtilde{\thp}{t}(\mu,\tilde{\mu},y_{1:t})\in\signmeas{\alg{X}}$ as well.	Let $\mu_{t}=\recf{\thp}^t(\mu,y_{1:t})$, and denote by $\mu^+$ and $\mu^-$ the positive and negative parts of $\tilde{\mu}$, respectively. If $\tvnorm{\mu^\pm}>0$, then we let $\mu_0^\pm=(\mu^\pm \1{\set{X}})^{-1}\mu^\pm$ and $\mu_t^\pm= \recf{\thp}^t(\mu_0^\pm,y_{1:t})$; otherwise we let $\{\mu_t^\pm\}_{t\ge 0}$ be a sequence of trivial measures. We start with part (i) and write
	\begin{align}
		\gtilde{\thp}{t}&(\mu,\tilde{\mu}, y_{1:t})= (\mu \lkert{\thp}{t}{y_{1:t}}\1{\set{X}})^{-1}(\tilde{\mu}\lkert{\thp}{t}{y_{1:t}}-\tilde{\mu}\lkert{\thp}{t}{y_{1:t}}\1{\set{X}}\mu_{t})
		\\&=\frac{(\mu^+\1{\set{X}}\mu_0^+-\mu^-\1{\set{X}}\mu_0^-)\lkert{\thp}{t}{y_{1:t}}-(\mu^+\1{\set{X}}\mu_0^+-\mu^-\1{\set{X}}\mu_0^-)\lkert{\thp}{t}{y_{1:t}}\1{\set{X}}\mu_{t}}{\mu \lkert{\thp}{t}{y_{1:t}}\1{\set{X}}}
		\\&=\frac{\mu^+\1{\set{X}}\mu_0^+\lkert{\thp}{t}{y_{1:t}}-\mu^-\1{\set{X}}\mu_0^-\lkert{\thp}{t}{y_{1:t}}-\mu^+\1{\set{X}}\mu_0^+\lkert{\thp}{t}{y_{1:t}}\1{\set{X}}\mu_{t}+\mu^-\1{\set{X}}\mu_0^-\lkert{\thp}{t}{y_{1:t}}\1{\set{X}}\mu_{t}}{\mu \lkert{\thp}{t}{y_{1:t}}\1{\set{X}}}
		\\&=\frac{\mu^+\1{\set{X}}\mu_0^+\lkert{\thp}{t}{y_{1:t}}\1{\set{X}}\mu_{t}^+-\mu^-\1{\set{X}}\mu_0^-\lkert{\thp}{t}{y_{1:t}}\1{\set{X}}\mu_{t}^-}{\mu \lkert{\thp}{t}{y_{1:t}}\1{\set{X}}}
		-\frac{\mu^+\1{\set{X}}\mu_0^+\lkert{\thp}{t}{y_{1:t}}\1{\set{X}}\mu_{t}-\mu^-\1{\set{X}}\mu_0^-\lkert{\thp}{t}{y_{1:t}}\1{\set{X}}\mu_{t}}{\mu \lkert{\thp}{t}{y_{1:t}}\1{\set{X}}}
		\\&=(\mu \lkert{\thp}{t}{y_{1:t}}\1{\set{X}})^{-1}\left(\mu^+\1{\set{X}}\mu_0^+\lkert{\thp}{t}{y_{1:t}}\1{\set{X}}(\mu_{t}^+-\mu_{t})+\mu^-\1{\set{X}}\mu_0^-\lkert{\thp}{t}{y_{1:t}}\1{\set{X}}(\mu_{t}-\mu_{t}^-)\right).
	\end{align}
	Now, using Assumption~\ref{assum:eq} we obtain that
	\begin{equation}
		\frac{\mu_0^\pm\lkert{\thp}{t}{y_{1:t}}\1{\set{X}}}{\mu \lkert{\thp}{t}{y_{1:t}}\1{\set{X}}}=\frac{\mu_0^\pm\lker{\thp}{y_1}\lkert{\thp}{t-1}{y_{2:t}}\1{\set{X}}}{\mu \lker{\thp}{y_1}\lkert{\thp}{t-1}{y_{2:t}}\1{\set{X}}}\le \frac{\eq^{-2}\refm\lkert{\thp}{t-1}{y_{2:t}}\1{\set{X}}}{\eq^{2}\refm\lkert{\thp}{t-1}{y_{2:t}}\1{\set{X}}}=\eq^{-4};
	\end{equation}
	thus
	\begin{multline}
		\tvnorm{\gtilde{\thp}{t}(\mu,\tilde{\mu}, y_{1:t})}
		=\tvnorm{\frac{\mu_0^+\lkert{\thp}{t}{y_{1:t}}\1{\set{X}}}{\mu \lkert{\thp}{t}{y_{1:t}}\1{\set{X}}}\mu^+\1{\set{X}}(\mu_{t}^+-\mu_{t})+\frac{\mu_0^-\lkert{\thp}{t}{y_{1:t}}\1{\set{X}}}{\mu \lkert{\thp}{t}{y_{1:t}}\1{\set{X}}}\mu^-\1{\set{X}}(\mu_{t}-\mu_{t}^-)}
		\\\le \eq^{-4}\left(\mu^+\1{\set{X}}\tvnorm{\mu_{t}^+-\mu_{t}}+\mu^-\1{\set{X}}\tvnorm{\mu_{t}-\mu_{t}^-}\right).
	\end{multline}
	Using Proposition~\ref{prop:forg_filt}, we obtain
	\begin{multline}
		\tvnorm{\gtilde{\thp}{t}(\mu,\tilde{\mu}, y_{1:t})}\le \eq^{-4}\kq\rhoq^t\left(\mu^+\1{\set{X}}\tvnorm{\mu_{0}^+-\mu_{0}}+\mu^-\1{\set{X}}\tvnorm{\mu_{0}-\mu_{0}^-}\right)
		\\\le 2\eq^{-4}\kq\rhoq^t\left(\mu^+\1{\set{X}}+\mu^-\1{\set{X}}\right)=2\eq^{-4}\kq\rhoq^t\tvnorm{\tilde{\mu}},
	\end{multline}
	where we used the fact that the total variation norm of the difference of two probability measures is at most two. In order to generalise to the case $\mutilde=(\tilde{\mu}_1,\dots,\tilde{\mu}_p)\in\signmeasp{\alg{X}}$, for $p>1$, we simply notice that
	\begin{equation}
		\tvnorm{\gtilde{\thp}{t}(\mu,\mutilde,y_{1:t})}=\max_{i \in \{1,\dots,p\}}\tvnorm{\gtilde{\thp}{t}(\mu,\tilde{\mu}_i,y_{1:t})}\le2\eq^{-4}\kq\rhoq^t\max_{i \in \{1,\dots,p\}}\tvnorm{\tilde{\mu}_i}=2\eq^{-4}\kq\rhoq^t\tvnorm{\mutilde}.
	\end{equation}
	Having proved (i), (ii) follows immediately since $\gtilde{\thp}{t}(\mu,\mutilde,y_{1:t})-\gtilde{\thp}{t}(\mu,\mutilde',y_{1:t})=\gtilde{\thp}{t}(\mu,\mutilde-\mutilde',y_{1:t})$. To prove (iii), let $\mu_t=\recf{\thp}^t(\mu,y_{1:t})$ and $\mu_t'=\recf{\thp}^t(\mu',y_{1:t})$ and write
	\begin{align}
		\lefteqn{\gtilde{\thp}{t}(\mu,\tilde{\mu}, y_{1:t})-\gtilde{\thp}{t}(\mu',\tilde{\mu}, y_{1:t})} \hspace{10mm} \\
  &= \frac{\tilde{\mu}\lkert{\thp}{t}{y_{1:t}}-\tilde{\mu}\lkert{\thp}{t}{y_{1:t}}\1{\set{X}}\mu_{t}}{\mu \lkert{\thp}{t}{y_{1:t}}\1{\set{X}}}-\frac{\tilde{\mu}\lkert{\thp}{t}{y_{1:t}}-\tilde{\mu}\lkert{\thp}{t}{y_{1:t}}\1{\set{X}}\mu_{t}'}{\mu' \lkert{\thp}{t}{y_{1:t}}\1{\set{X}}}
		\\&=-\frac{(\mu-\mu')\lkert{\thp}{t}{y_{1:t}}\1{\set{X}}\tilde{\mu}\lkert{\thp}{t}{y_{1:t}}}{\mu\lkert{\thp}{t}{y_{1:t}}\1{\set{X}}\mu' \lkert{\thp}{t}{y_{1:t}}\1{\set{X}}}-\tilde{\mu}\lkert{\thp}{t}{y_{1:t}}\1{\set{X}}\left(\frac{\mu_t}{\mu\lkert{\thp}{t}{y_{1:t}}\1{\set{X}}}-\frac{\mu_t'}{\mu' \lkert{\thp}{t}{y_{1:t}}\1{\set{X}}}\right)
		\\&=-\frac{(\mu-\mu')\lkert{\thp}{t}{y_{1:t}}\1{\set{X}}}{\mu\lkert{\thp}{t}{y_{1:t}}\1{\set{X}}}\frac{\tilde{\mu}\lkert{\thp}{t}{y_{1:t}}}{\mu' \lkert{\thp}{t}{y_{1:t}}\1{\set{X}}}
		-\tilde{\mu}\lkert{\thp}{t}{y_{1:t}}\1{\set{X}}\frac{\mu' \lkert{\thp}{t}{y_{1:t}}\1{\set{X}}(\mu_t-\mu_t')-(\mu-\mu') \lkert{\thp}{t}{y_{1:t}}\1{\set{X}}\mu_t'}{\mu\lkert{\thp}{t}{y_{1:t}}\1{\set{X}}\mu' \lkert{\thp}{t}{y_{1:t}}\1{\set{X}}}
		\\&= -\frac{(\mu-\mu')\lkert{\thp}{t}{y_{1:t}}\1{\set{X}}}{\mu\lkert{\thp}{t}{y_{1:t}}\1{\set{X}}}\left(\frac{\tilde{\mu}\lkert{\thp}{t}{y_{1:t}}}{\mu' \lkert{\thp}{t}{y_{1:t}}\1{\set{X}}}-\frac{\tilde{\mu}\lkert{\thp}{t}{y_{1:t}}\1{\set{X}}\mu_t'}{\mu' \lkert{\thp}{t}{y_{1:t}}\1{\set{X}}}\right)
		-\tilde{\mu}\lkert{\thp}{t}{y_{1:t}}\1{\set{X}}\frac{\mu_t-\mu_t'}{\mu\lkert{\thp}{t}{y_{1:t}}\1{\set{X}}}
		\\&= -\frac{(\mu-\mu')\lkert{\thp}{t}{y_{1:t}}\1{\set{X}}}{\mu\lkert{\thp}{t}{y_{1:t}}\1{\set{X}}}\gtilde{\thp}{t}(\mu',\tilde{\mu}, y_{1:t})
		-\frac{\tilde{\mu}\lkert{\thp}{t}{y_{1:t}}\1{\set{X}}}{\mu\lkert{\thp}{t}{y_{1:t}}\1{\set{X}}}(\mu_t-\mu_t').
	\end{align}
	Now, using Assumption~\ref{assum:eq} we obtain that
	\begin{equation}
		\frac{\abs{\tilde{\mu}\lkert{\thp}{t}{y_{1:t}}\1{\set{X}}}}{\mu\lkert{\thp}{t}{y_{1:t}}\1{\set{X}}}=\frac{\abs{\tilde{\mu}\lker{\thp}{y_1}\lkert{\thp}{t-1}{y_{1:t}}\1{\set{X}}}}{\mu \lker{\thp}{y_1}\lkert{\thp}{t-1}{y_{1:t}}\1{\set{X}}}\le \frac{\eq^{-2}\abs{\tilde{\mu}\1{\set{X}}}\refm\lkert{\thp}{t-1}{y_{1:t}}\1{\set{X}}}{\eq^{2}\refm\lkert{\thp}{t-1}{y_{1:t}}\1{\set{X}}}\le\eq^{-4}\tvnorm{\tilde{\mu}}
	\end{equation}
	and, similarly, $(\mu\lkert{\thp}{t}{y_{1:t}}\1{\set{X}})^{-1}\abs{(\mu-\mu')\lkert{\thp}{t}{y_{1:t}}\1{\set{X}}}\le \eq^{-4}\tvnorm{\mu-\mu'}$. Therefore, 
	\begin{align}
  \lefteqn{\tvnorm{\gtilde{\thp}{t}(\mu,\tilde{\mu}, y_{1:t})-\gtilde{\thp}{t}(\mu',\tilde{\mu}, y_{1:t})}} \hspace{10mm} \\
		&\le \frac{\abs{(\mu-\mu')\lkert{\thp}{t}{y_{1:t}}\1{\set{X}}}}{\mu\lkert{\thp}{t}{y_{1:t}}\1{\set{X}}}\tvnorm{\gtilde{\thp}{t}(\mu',\tilde{\mu}, y_{1:t})}+\frac{\abs{\tilde{\mu}\lkert{\thp}{t}{y_{1:t}}\1{\set{X}}}}{\mu\lkert{\thp}{t}{y_{1:t}}\1{\set{X}}}\tvnorm{\mu_t-\mu_t'}
		\\&\le \eq^{-4}\left(\tvnorm{\mu-\mu'}\tvnorm{\gtilde{\thp}{t}(\mu',\tilde{\mu}, y_{1:t})}+\tvnorm{\tilde{\mu}}\tvnorm{\mu_t-\mu_t'}\right)
		\\&\le  \eq^{-4}\left(\eq^{-4}\kq\rhoq^t\tvnorm{\tilde{\mu}}\tvnorm{\mu-\mu'}+ \kq\rhoq^t\tvnorm{\mu-\mu'}\tvnorm{\tilde{\mu}}\right) \\
		&\le 2\eq^{-8}\kq\rhoq^t\tvnorm{\tilde{\mu}}\tvnorm{\mu-\mu'},
	\end{align}
	where we in the penultimate inequality used (i) and Proposition~\ref{prop:forg_filt}. Finally, (iii) is proven letting again $\mutilde=(\tilde{\mu}_1,\dots,\tilde{\mu}_p)\in\signmeasp{\alg{X}}$, for $p>1$, and writing 
	\begin{align}
		\tvnorm{\gtilde{\thp}{t}(\mu,\mutilde,y_{1:t})-\gtilde{\thp}{t}(\mu',\mutilde,y_{1:t})}&=\max_{i \in \{1,\dots,p\}}\tvnorm{\gtilde{\thp}{t}(\mu,\tilde{\mu}_i,y_{1:t})-\gtilde{\thp}{t}(\mu',\tilde{\mu}_i,y_{1:t})}
		\\
    &\le2\eq^{-8}\kq\rhoq^t \tvnorm{\mu-\mu'}\max_{i \in \{1,\dots,p\}}\tvnorm{\tilde{\mu}_i} \\
    &=2 \eq^{-8}\kq\rhoq^t \tvnorm{\mu-\mu'}\tvnorm{\mutilde}.
	\end{align}
\end{proof}
\begin{lemma}\label{lemma:htil_forg}
	Let Assumptions~\ref{assum:eq}--\ref{assum:bg} hold. Then for every $t\in\nset$, $y_{1:t} \in \set{Y}^t$,  
 $\thp\in\parspace$, and $(\mu, \mu')\in\probmeas{\alg{X}}^2$.
	\begin{align}
		\text{(i) }&\tvnorm{\htilde{\thp}{t}(\mu,y_{1:t})}\le4\bg \eq^{-3},
		\\\text{(ii) }&\tvnorm{\htilde{\thp}{t}(\mu,y_{1:t})-\htilde{\thp}{t}(\mu',y_{1:t})}\le 10\bg\eq^{-7}\kq\rhoq^{t-1}\tvnorm{\mu-\mu'}. 
	\end{align}
\end{lemma}
\begin{proof}
	Let $\mu_{t}=\recf{\thp}^t(\mu,y_{1:t})$ and $\mu_{t}'=\recf{\thp}^t(\mu',y_{1:t})$. Then, using Definition~\ref{def:lkert_gtile_htilde}, 
	\begin{equation}
		\htilde{\thp}{t}(\mu, y_{1:t})=\frac{\mu_{t-1}\lkertang{\thp}{y_t}-
			\mu_{t-1}\lkertang{\thp}{y_t}\1{\set{X}}\mu_{t}}{\mu_{t-1}\lker{\thp}{y_t}\1{\set{X}} },
	\end{equation}
	so we may write 
	\begin{equation}
		\tvnorm{\htilde{\thp}{t}(\mu, y_{1:t})}\le\frac{\tvnormsmall{\mu_{t-1}\lkertang{\thp}{y_t}}+
			\tvnormsmall{\mu_{t-1}\lkertang{\thp}{y_t}}\tvnorm{\mu_{t}}}{\mu_{t-1}\lker{\thp}{y_t}\1{\set{X}} }
		\le2\eq^{-2}\tvnorm{\mu_{t-1}\lkertang{\thp}{y_t}}\le 4\bg \eq^{-3},
	\end{equation}
	where we have used Lemma~\ref{lemma:bound_lkertang}. This establishes (i). In order to prove (ii), let $\mu_{t}=\recf{\thp}^t(\mu,y_{1:t})$ and $\mu_{t}'=\recf{\thp}^t(\mu',y_{1:t})$ and write
	\begin{align}
		\htilde{\thp}{t}(\mu, y_{1:t})-\htilde{\thp}{t}(\mu', y_{1:t})
		&=(\mu_{t-1}\lkertang{\thp}{y_t}-
		\mu_{t-1}\lkertang{\thp}{y_t}\1{\set{X}}\mu_{t})\frac{(\mu_{t-1}'-\mu_{t-1})\lker{\thp}{y_t}\1{\set{X}} }{\mu_{t-1}\lker{\thp}{y_t}\1{\set{X}} \mu_{t-1}'\lker{\thp}{y_t}\1{\set{X}} }
		\\&\hspace{1cm}+\frac{\mu_{t-1}\lkertang{\thp}{y_t}-
			\mu_{t-1}\lkertang{\thp}{y_t}\1{\set{X}}\mu_{t}}{\mu_{t-1}'\lker{\thp}{y_t}\1{\set{X}} }-\frac{\mu_{t-1}'\lkertang{\thp}{y_t}-
			\mu_{t-1}'\lkertang{\thp}{y_t}\1{\set{X}}\mu_{t}'}{\mu_{t-1}'\lker{\thp}{y_t}\1{\set{X}} }
		\\&=(\mu_{t-1}\lkertang{\thp}{y_t}-
		\mu_{t-1}\lkertang{\thp}{y_t}\1{\set{X}}\mu_{t})\frac{(\mu_{t-1}'-\mu_{t-1})\lker{\thp}{y_t}\1{\set{X}} }{\mu_{t-1}\lker{\thp}{y_t}\1{\set{X}} \mu_{t-1}'\lker{\thp}{y_t}\1{\set{X}} }
		\\&\hspace{1cm}+
		\frac{(\mu_{t-1}-\mu_{t-1}')\lkertang{\thp}{y_t}+(\mu_{t-1}'-	\mu_{t-1})\lkertang{\thp}{y_t}\1{\set{X}}\mu_{t}+
		\mu_{t-1}'\lkertang{\thp}{y_t}\1{\set{X}}(\mu_{t}'-\mu_{t})}{\mu_{t-1}'\lker{\thp}{y_t}\1{\set{X}}}.
	\end{align}
	This yields the bound 
	\begin{align}
		\lefteqn{\tvnorm{\htilde{\thp}{t}(\mu,y_{1:t})-\htilde{\thp}{t}(\mu',,y_{1:t})}}\hspace{10mm} \\
		&\le \left(\tvnorm{\mu_{t-1}\lkertang{\thp}{y_t}}+\tvnorm{\mu_{t-1}\lkertang{\thp}{y_t}}\tvnorm{\mu_{t}}\right)\frac{\abs{(\mu_{t-1}'-\mu_{t-1})\lker{\thp}{y_t}\1{\set{X}}} }{\mu_{t-1}\lker{\thp}{y_t}\1{\set{X}} \mu_{t-1}'\lker{\thp}{y_t}\1{\set{X}} }
		\\&\hspace{2mm}+\frac{\tvnormsmall{(\mu_{t-1}-\mu_{t-1}')\lkertang{\thp}{y_t}}+\tvnormsmall{(\mu_{t-1}'-	\mu_{t-1})\lkertang{\thp}{y_t}}\tvnorm{\mu_{t}}+
		\tvnormsmall{\mu_{t-1}'\lkertang{\thp}{y_t}}\tvnorm{\mu_{t}'-\mu_{t}}}{\mu_{t-1}'\lker{\thp}{y_t}\1{\set{X}}}
		\\&\le 4\bg\eq^{-1}\frac{\eq^{-2}\tvnormsmall{\mu_{t-1}'-\mu_{t-1}}}{\eq^4}+\eq^{-2}\left(4\bg\eq^{-1}\tvnorm{\mu_{t-1}-\mu_{t-1}'}+2\bg\eq^{-1}\tvnorm{\mu_{t}-\mu_{t}'}\right)
		\\&\le 8\bg\eq^{-7}\tvnorm{\mu_{t-1}-\mu_{t-1}'}+2\bg\eq^{-3}\tvnorm{\mu_{t}-\mu_{t}'}
		\\&\le 8\bg\eq^{-7}\kq\rhoq^{t-1}\tvnorm{\mu-\mu'}+ 2\bg\eq^{-3}\kq\rhoq^t\tvnorm{\mu-\mu'} \\
		&\le 10\bg\eq^{-7}\kq\rhoq^{t-1}\tvnorm{\mu-\mu'},
	\end{align}
	which finally proves (ii).
\end{proof}

The following lemma, which is a direct consequence of the previous results, will be useful later.
\begin{lemma}\label{lemma:bound_tang}
	Let Assumptions~\ref{assum:eq}--\ref{assum:bg} hold. Then for every $\thp\in\parspace$ and $y_0\in\set{Y}$, 
	\begin{equation}
		\tvnorm{\tang{\thp}{0}}\le2\eq^{-1}\bg.
	\end{equation}
		Moreover, there exist constants $\crect > 0$ and $\ctang>0$ depending only on $\eq$ and $\bg$, such that for every $t\in\nset$, $y_{1:t} \in \set{Y}^t$, $\thp\in\parspace$, $\mu\in\probmeas{\alg{X}}$, and $\mutilde\in\signmeasp{\alg{X}}$, 
	\begin{equation}
		\tvnorm{\rect{\thp}^t(\mu, \mutilde,y_{1:t})}\le \crect(\rhoq^t\tvnorm{\mutilde}+1). 
	\end{equation}
 	In particular, letting $\tang{\thp}{t}\eqdef\rect{\thp}(\filt{\thp}{0}, \tang{\thp}{0}, y_{1:t})$,
	\begin{equation}
		\tvnorm{\tang{\thp}{t}}\le\crect(\rhoq^t\tvnorm{\tang{\thp}{0}}+1)\le\ctang. 
	\end{equation}
\end{lemma}
\begin{proof}
	Using Lemmas~\ref{lemma:gtil_htil}--\ref{lemma:htil_forg}, we may obtain the bound 
	\begin{align}
		\tvnorm{\rect{\thp}^t(\mu, \mutilde,y_{1:t})}
		&\le \tvnorm{\gtilde{\thp}{t}(\mu, \mutilde,y_{1:t})}+\sum_{s=1}^{t}\tvnorm{\gtilde{\thp}{t-s}(\recf{\thp}(\mu,y_{1:s}),\htilde{\thp}{s}(\mu, y_{1:s}), y_{s+1:t})}
		\\&\le 2\eq^{-4}\kq\rhoq^t\tvnorm{\mutilde}+ \sum_{s=1}^{t}2\eq^{-4}\kq\rhoq^{t-s}\tvnorm{\htilde{\thp}{s}(\mu, y_{1:s})} \\
		&\le 2\eq^{-4}\kq\rhoq^t\tvnorm{\mutilde}+ 8\eq^{-7}\bg\kq\sum_{s=0}^{t-1}\rhoq^{s}
		\\
		&\le 8\eq^{-7}\bg\kq(1-\rhoq)^{-1}\left(\rhoq^t\tvnorm{\mutilde}+1\right).
	\end{align} 
	Now, for all $\thp\in\parspace$, 
	\begin{multline}
		\tvnorm{\tang{\thp}{0}}= \int\norm{\nabla_{\thp}\left(\frac{ \hiddens{0}(x)\emdens{\thp}(x,y)}{\int\hiddens{0}(x')\emdens{\thp}(x',y)\,\refm(dx') }\right)}\,\refm(dx)
		\\\le \frac{\int \hiddens{0}(x)\norm{\nabla_\thp\emdens{\thp}(x,y)}\,\refm(dx)}{\int\hiddens{0}(x')\emdens{\thp}(x',y)\,\refm(dx') }+\frac{\int \hiddens{0}(x)\emdens{\thp}(x,y)\,\refm(dx)}{\left(\int\hiddens{0}(x')\emdens{\thp}(x',y)\,\refm(dx')\right)^2 }\int\hiddens{0}(x')\norm{\nabla_\thp\emdens{\thp}(x',y)}\,\refm(dx') 
		\\\le 2\eq^{-1}\int \hiddens{0}(x)\norm{\nabla_{\thp}\emdens{\thp}(x,y)}\,\refm(dx)\le 2\eq^{-1}\bg.
	\end{multline}
	Hence,
	\begin{equation}
	\tvnorm{\tang{\thp}{t}}=\tvnorm{\rect{\thp}^t(\filt{\thp}{0}, \tang{\thp}{0},y_{1:t})}\le8\eq^{-7}\bg\kq(1-\rhoq)^{-1}(2\eq^{-1}\bg+1).
	\end{equation}
    The proof is concluded by letting $\crect\eqdef8\eq^{-7}\bg\kq(1-\rhoq)^{-1}$ and $\ctang\eqdef\crect (2\eq^{-1}\bg+1)$.
\end{proof}

\begin{proposition}[exponential forgetting of the tangent filter]\label{prop:forg_tang}
	Let Assumptions~\ref{assum:eq}--\ref{assum:bg} hold.  Then there exists $\kqt>1$, depending only on $\eq$ and $\bg$, such that for every $t\in\nset$, $y_{1:t} \in \set{Y}^t$, $\thp\in\parspace$, $(\mu,\mu')\in \probmeas{\alg{X}}^2$, and $(\mutilde,\mutilde')\in \signmeasp{\alg{X}}^2$, 
	\begin{equation}
		\tvnorm{\rect{\thp}^{t}(\mu,\mutilde, y_{1:t})-\rect{\thp}^{t}(\mu',\mutilde', y_{1:t})}
		\le\eq^{-4}\kq\rhoq^t\tvnorm{\mutilde-\mutilde'}+\kqt \rhoq^{t} \tvnorm{\mu-\mu'}(\tvnorm{\mutilde'}+1)t. 
	\end{equation}
\end{proposition}
\begin{proof}
	Using Lemma~\ref{lemma:gtil_htil} and Lemma~\ref{lemma:gtil_forg}, we may write 
	\begin{align}
		\lefteqn{\tvnorm{	\rect{\thp}^{t}(\mu,\mutilde, y_{1:t})-\rect{\thp}^{t}(\mu',\mutilde', y_{1:t}) }} \hspace{10mm}
		\\&\le\tvnorm{\gtilde{\thp}{t}(\mu,\mutilde,y_{1:t})-\gtilde{\thp}{t}(\mu',\mutilde,y_{1:t})}+\tvnorm{\gtilde{\thp}{t}(\mu',\mutilde,y_{1:t})-\gtilde{\thp}{t}(\mu',\mutilde',y_{1:t})}
		\\&\hspace{1cm}+\sum_{s=1}^{t}\tvnorm{\gtilde{\thp}{t-s}(\recf{\thp}^s(\mu, y_{1:s}),\htilde{\thp}{s}(\mu, y_{1:s}), y_{s:t})-\gtilde{\thp}{t-s}(\recf{\thp}^s(\mu', y_{1:s}),\htilde{\thp}{s}(\mu, y_{1:s}), y_{s:t})}
		\\&\hspace{1cm}+\sum_{s=1}^{t}\tvnorm{\gtilde{\thp}{t-s}(\recf{\thp}^s(\mu', y_{1:s}),\htilde{\thp}{s}(\mu, y_{1:s}), y_{s:t})-\gtilde{\thp}{t-s}(\recf{\thp}^s(\mu', y_{1:s}),\htilde{\thp}{s}(\mu', y_{1:s}), y_{s:t})}
		\\&\le 2\eq^{-8}\kq\rhoq^t \tvnorm{\mu-\mu'}\tvnorm{\mutilde'}+2\eq^{-4}\kq\rhoq^t\tvnorm{\mutilde-\mutilde'}
		\\&\hspace{1cm}+\sum_{s=1}^{t-1} 2\eq^{-8}\kq\rhoq^{t-s} \tvnorm{\recf{\thp}^s(\mu, y_{1:s})-\recf{\thp}^s(\mu', y_{1:s})}\tvnorm{\htilde{\thp}{s}(\mu, y_{1:s})}
		\\&\hspace{1cm}+\sum_{s=1}^{t-1}2\eq^{-4}\kq\rhoq^{t-s} \tvnorm{\htilde{\thp}{s}(\mu, y_{1:s})-\htilde{\thp}{s}(\mu', y_{1:s})}
		+\tvnorm{\htilde{\thp}{t}(\mu, y_{1:t})-\htilde{\thp}{t}(\mu', y_{1:t})}.
	\end{align}
	Now, applying Proposition~\ref{prop:forg_filt} and Lemma~\ref{lemma:htil_forg} yields 
	\begin{align}
		\lefteqn{\tvnorm{	\rect{\thp}^{t}(\mu,\mutilde, y_{1:t})-\rect{\thp}^{t}(\mu',\mutilde', y_{1:t}) }} \hspace{-10mm} \\
		&\le		2\eq^{-8}\kq\rhoq^t \tvnorm{\mu-\mu'}\tvnorm{\mutilde'}+\eq^{-4}\kq\rhoq^t\tvnorm{\mutilde-\mutilde'}
		\\&\hspace{1cm}+\sum_{s=1}^{t-1} 2\eq^{-8}\kq\rhoq^{t-s} \kq\rhoq^s\tvnorm{\mu-\mu'}4\bg \eq^{-3}
		\\&\hspace{1cm}+\sum_{s=1}^{t-1}\eq^{-4}\kq\rhoq^{t-s} 10\bg\eq^{-7}\kq\rhoq^{s-1}\tvnorm{\mu-\mu'}+10\bg\eq^{-7}\kq\rhoq^{t-1}\tvnorm{\mu-\mu'}
		\\&=\eq^{-4}\kq\rhoq^t\tvnorm{\mutilde-\mutilde'}+\tvnorm{\mu-\mu'}\left(2\eq^{-8}\kq\rhoq^t \tvnorm{\mutilde'}+ (t-1)8\eq^{-11}\bg\kq\rhoq^t\right.
		\\&\hspace{1cm}\left.+(t-1)10\bg\eq^{-11}\kq^2\rhoq^{t-1}+10\bg\eq^{-7}\kq\rhoq^{t-1}\right)
		\\&=\eq^{-4}\kq\rhoq^t\tvnorm{\mutilde-\mutilde'}+2\eq^{-8}\kq\rhoq^t \tvnorm{\mu-\mu'}
		\\&\hspace{1cm}\times\left(\tvnorm{\mutilde'}+ (t-1)4\eq^{-3}\bg+(t-1)5\bg\eq^{-3}\kq\rhoq^{-1}+5\bg\eq\rhoq^{-1}\right)
		\\&\le \eq^{-4}\kq\rhoq^t\tvnorm{\mutilde-\mutilde'}+2\eq^{-8}\kq\rhoq^t \tvnorm{\mu-\mu'}(\tvnorm{\mutilde'}+t9\bg\eq^{-3}\kq\rhoq^{-1})
		\\&\le \eq^{-4}\kq\rhoq^t\tvnorm{\mutilde-\mutilde'}+18\eq^{-11}\kq^2\bg\rhoq^{t-1} \tvnorm{\mu-\mu'}(\tvnorm{\mutilde'}+1)t.
	\end{align}
	Finally, the proof is completed by defining $\kqt\eqdef18\eq^{-11}\kq^2\bg\rhoq^{-1}$.
\end{proof}

\subsection{Construction and ergodicity of the extended chain}\label{subsec:ergodicity}
In this section we construct the extended Markov chain, which includes the data-generating SSM, the filter, and the tangent filter, and thus evolves on the product space 
\begin{equation}
	(\set{Z}, \alg{Z})\eqdef(\set{X}\times \set{Y} \times \probmeas{\alg{X}}\times \signmeaszp{\alg{X}}, \alg{X}\tensprod\alg{Y}\tensprod\probmeasalg{\alg{X}}\tensprod\signmeaszpalg{\alg{X}} ),
\end{equation}
where $(\signmeaszp{\alg{X}}, \signmeaszpalg{\alg{X}})$ is the $p$-fold product of $(\signmeasz{\alg{X}}, \signmeaszalg{\alg{X}})$, the latter being the measurable space of finite signed measures $\tilde{\mu}$ such that $\int\tilde{\mu}(dx)=0$. In fact, 
since we will be dealing exclusively with tangent filter measures, this property is always fulfilled. Note that for all $\mu\in\probmeas{\alg{X}}$, $\mutilde\in\signmeaszp{\alg{X}}$, and $y\in\set{Y}$, it is easily checked that $\rect{\thp}(\mu,\mutilde,y)\1{\set{X}}=0$; thus the tangent filter recursion is zero-mean preserving.

We begin by assuming that the law of the complete data is governed by an unspecified SSM, as stated below.
\begin{assumption}\label{assum:ssm_unspec}
	The observed data stream $(Y_t)_{t \in \nset}$ is the output of an SSM $(X_t, Y_t)_{t \in \nset}$ on $(\set{X} \times \set{Y}, \alg{X} \tensprod \alg{Y})$ with some state and observation transition kernels $\hidkertrue(x,dx')$ and $\emkertrue(x,dy)$, respectively. These kernels have transition densities $\hiddenstrue(x,x')$ and $\emdenstrue(x,y)$ with respect to $\refm$ and $\refg$, respectively. Furthermore, we let $\skertrue:(\set{X}\times\set{Y})\times (\alg{X}\tensprod\alg{Y})\to (0,1)$ be the product kernel $\hidker{}\tensprod\emker{}$, defined in the same way as its parametric counterpart in \eqref{def:sker}. The latent transition density satisfies $\eq\le \hiddenstrue(x,x')\le\eq^{-1}$ for all $(x,x')\in\set{X}^2$, where $\eq$ is the same as in Assumption~\ref{assum:eq}.
\end{assumption}
Now we introduce the Markov kernel of the extended chain, which is, for every $z= (x_1, y_1, \mu,\mutilde)\in\set{Z}$ and $f\in\bmf{\alg{Z}}$, given by
\begin{equation}
	\tz{\thp}f(z) = \int f(z') \,\tz{\thp}(z,dz')= \int f(x_2, y_2, \recf{\thp}(\mu,y_1), \rect{\thp}(\mu,\mutilde,y_1))\,\skertrue((x_1,y_1),(dx_2,dy_2)).
\end{equation}
For $t\in \nsetpos$, we let $\tz{\thp}^t$ be the $t$-skeleton, \ie, $\tz{\thp}^1=\tz{\thp}$ and, recursively, $\tz{\thp}^{t+1}f(z)=\int \tz{\thp}^tf(z') \,\tz{\thp}(z,dz')$ for $z\in\set{Z}$ and $f\in\lip{\alg{Z}}$. By convention, we let $\tz{\thp}^0f(z)=\delta_zf= f(z)$. It follows that
\begin{equation}
	\tz{\thp}^tf(z)=\int\cdots\int f(x_{t+1}, y_{t+1}, \recf{\thp}^{t}(\mu,y_{1:t}), \rect{\thp}^{t}(\mu,\mutilde,y_{1:t}))\,\prod_{s=1}^{t} \skertrue((x_{s}, y_{s}),(dx_{s+1},dy_{s+1})).
\end{equation}
In the following we will denote by $(\z{t}{\thp})_{t\in\nset}$ the extended Markov chain governed by the kernel $\tz{\thp}$, where $\z{t}{\thp}= (X_{t+1},Y_{t+1}, \recf{\thp}^{t}(\filt{\thp}{0},Y_{1:t}),\rect{\thp}^{t}(\filt{\thp}{0},\tang{\thp}{0},Y_{1:t}))$; note that $\filt{\thp}{0}$ and $\tang{\thp}{0}$ both depend on the initial observation $Y_0$. In the following assumption, we formalise how the chain is initialised.
\begin{assumption}\label{assum:init}
	The extended chain is initialised by applying a kernel $\boldsymbol{\chi}_\thp:(\set{X}\times\set{Y})\times \alg{Z}\to [0,1]$ given by, for $(x_0,y_0)\in\set{X}\times\set{Y}$ and $f\in\bmf{\alg{Z}}$,
	\begin{equation}
		\boldsymbol{\chi}_\thp f (x_0,y_0)\eqdef\int f(x_1,y_1,\filt{\thp}{0},\tang{\thp}{0})\,\skertrue((x_0,y_0),(dx_1,dy_1)),
	\end{equation}
	recalling that $\filt{\thp}{0}$ and $\tang{\thp}{0}$ are deterministic maps defined in \eqref{eq:filt_def}. The initial data $(X_0,Y_0)$ is distributed according to $\xinit\tensprod\emker{}$.
\end{assumption}
\begin{remark}
	If we at any point in time $t\in\nsetpos$ take the conditional expectation of the Markov chain w.r.t. the initial state $\z{0}{\thp}$, the latter includes information about the first two observed data points $Y_{0:1}$.
\end{remark}

Before we can establish the ergodicity of the extended chain, we need the following lemma, which establishes the ergodicity of the data-generating process.
\begin{lemma}\label{lemma:erg_sker}
	Let Assumption~\ref{assum:ssm_unspec} hold. Then there exists $\sigma\in\probmeas{\alg{X}\tensprod\alg{Y}}$ such that for all $(x,y)\in\set{X}\times\set{Y}$ and $t\in \nsetpos$, 
	\begin{equation}
		\tvnorm{\skertrue^t((x,y),\cdot)-\sigma }\le(1-\eq)^t.
	\end{equation}
\end{lemma}
\begin{proof}
	We first note that that the state space $\set{X}\times \set{Y}$ is $\nu_1$-small, where $\nu_1\in\meas{\alg{X}\tensprod\alg{Y}}$ is defined as $\nu_1(dx,dy) \eqdef \eq(\refm(dx) \tensprod \emkertrue(x,dy))$. In fact, for all $B\in\alg{X}\tensprod\alg{Y}$,
	\begin{equation}
		\int \1{B}(x',y')\,\skertrue((x,y),(dx',dy'))= \int \1{B}(x',y')\,\hidkertrue(x,dx')\,\emkertrue(x',dy')
		\ge  \eq\int \1{B}(x',y') \,\refm(dx')\,\emkertrue(x',dy').
	\end{equation}
	Then, by \citet[Theorem 16.2.4(\emph{v})]{meyn:tweedie:2009} it follows that for all $(x,y)\in \set{X}\times \set{Y}$,
	\begin{equation}
		\tvnorm{\skertrue^t((x,y),\cdot)-\sigma }\le (1-\nu_1(\set{X}\times\set{Y}))^t=(1-\eq)^t.
	\end{equation}
\end{proof}
We now establish a form of uniform geometric ergodicity of $(\z{t}{\thp})_{t\in\nset}$ for a certain class of measurable functions on $\set{Z}$, which are Lipschitz in the arguments $\mu$ and $\mutilde$. 
\begin{definition}\label{def:lipz}
	Let $\lip{\alg{Z}}$ be the set of vector-valued measurable functions on $\set{Z}$, for which there exists a positive constant $\lipz{f}$ such that for all $x\in\set{X}$, $y\in\set{Y}$, $(\mu,\mu')\in\probmeas{\alg{X}}^2$, and $(\mutilde,\mutilde')\in\signmeaszp{\alg{X}}^2$,
	\begin{itemize}
		\item[(i)] $\norm{f(x, y,  \mu,\mutilde)}
		\le  \lipz{f}(1+\tvnorm{\mutilde})$,
		\item[(ii)] $\norm{f(x, y,  \mu,\mutilde)-f(x, y,  \mu',\mutilde', )}\le  \lipz{f}\left(\tvnorm{\mutilde-\mutilde'}+(1+\tvnorm{\mutilde}+\tvnorm{\mutilde'})\tvnorm{\mu-\mu'}\right)$,
	\end{itemize}
\end{definition}
\begin{proposition}[Uniform ergodicity of the extended Markov chain]\label{prop:ergod}
	Let Assumptions~\ref{assum:eq}, \ref{assum:bg}, and \ref{assum:ssm_unspec} hold. Then there exist constants $\cq>0$ and $\cqt>0$, depending on $\eq,\bg$ only, such that for all
	$t\in\nsetpos$, $\thp\in\parspace$, $(z,z')\in\set{Z}^2$, and $f\in\lip{\alg{Z}}$,
	\begin{equation}
		\norm{\tz{\thp}^t f(z)-\tz{\thp}^tf(z')}\le \cq\lipz{f}(\tvnorm{\mutilde}+\tvnorm{\mutilde'}+1)\rhoq^{t/2}
	\end{equation}
	and 
	\begin{equation}
		\norm{\tz{\thp}^{t+1} f(z)-\tz{\thp}^tf(z)}\le \cqt(1-\rhoq^{1/2})\lipz{f}(\tvnorm{\mutilde}+1)\rhoq^{t/2}.
	\end{equation}
Moreover, there exists a kernel $\tconst{\thp}:\set{Z}\times\alg{Z}\to [0,1]$ such that for every $f\in\lip{\alg{Z}}$, $f_\thp\eqdef\tconst{\thp}f(z)$ is a constant and it holds that 
	\begin{equation}
		\norm{\tz{\thp}^t f(z)-f_\thp}\le \cqt\lipz{f}(\tvnorm{\mutilde}+1)\rhoq^{t/2}.
	\end{equation}
\end{proposition}
\begin{proof}
	Let $z=(x_1,y_1,\mu,\mutilde)$ and $z'=(x_1',y_1',\mu',\mutilde')$. For every $f\in\lip{\alg{Z}}$, we write
	\begin{multline}\label{eq:t_diff}
		\tz{\thp}^t f(z)-\tz{\thp}^t f(z') =\int\cdots\int  \left( f(x_{t+1}, y_{t+1}, \recf{\thp}^{t}(\mu,y_{1:t}), \rect{\thp}^{t}(\mu,\mutilde,y_{1:t}))\right.
		\\ \left.- f(x_{t+1}, y_{t+1}, \recf{\thp}^{t}(\mu',y_{1:t}), \rect{\thp}^{t}(\mu',\mutilde',y_{1:t}))\right)\prod_{s=1}^{t} \skertrue((x_{s}, y_{s}),(dx_{s+1},dy_{s+1}))
		\\ +\int f(x_{t+1}, y_{t+1}, \recf{\thp}^{t}(\mu',y_{1:t}), \rect{\thp}^{t}(\mu',\mutilde',y_{1:t}))
		\\\times\left(\skertrue((x_{1}, y_{1}),(dx_{2},dy_{2}))-\skertrue((x_{1}', y_{1}'),(dx_2,dy_2))\right)\prod_{s=2}^{t}\skertrue((x_{s}, y_{s}),(dx_{s+1},dy_{s+1})).
	\end{multline}
	Let us focus on the first term in \eqref{eq:t_diff}: by Definition~\ref{def:lipz}, Lemma~\ref{lemma:bound_tang}, Proposition~\ref{prop:forg_filt}, and Proposition~\ref{prop:forg_tang}, it holds that 
	\begin{align}
		&\hspace{-1cm}\norm{f(x_{t+1}, y_{t+1}, \recf{\thp}^{t}(\mu,y_{1:t}), \rect{\thp}^{t}(\mu,\mutilde,y_{1:t}))- f(x_{t+1}, y_{t+1}, \recf{\thp}^{t}(\mu',y_{1:t}), \rect{\thp}^{t}(\mu',\mutilde',y_{1:t}))}
		\\&\le \lipz{f}\tvnorm{\rect{\thp}^{t}(\mu,\mutilde,y_{1:t})-\rect{\thp}^{t}(\mu',\mutilde',y_{1:t})} \\&\hspace{1cm}+\lipz{f}(1+\tvnorm{\rect{\thp}^{t}(\mu,\mutilde,y_{1:t})}+\tvnorm{\rect{\thp}^{t}(\mu',\mutilde',y_{1:t})})\tvnorm{\recf{\thp}^{t}(\mu,y_{1:t})-\recf{\thp}^{t}(\mu',y_{1:t})}
		\\&\le \lipz{f}(\eq^{-4}\kq\rhoq^{t}\tvnorm{\mutilde-\mutilde'}+\kqt \rhoq^{t} \tvnorm{\mu-\mu'}(\tvnorm{\mutilde'}+1)t)
		\\&\hspace{1cm}+\lipz{f}(1+\crect(\rhoq^t\tvnorm{\mutilde}+1)+\crect(\rhoq^t\tvnorm{\mutilde'}+1))\kq\rhoq^{t}\tvnorm{\mu-\mu'}
		\\&\le \lipz{f}t\rhoq^t\left(2\kqt+(1+2\crect)2\kq +(\eq^{-4}\kq+2\kq\crect\rhoq^t)\tvnorm{\mutilde}+(2\kq \crect\rhoq^t+\eq^{-4}\kq+2\kqt)\tvnorm{\mutilde'}\right)
		\\&\le (2\kqt+4\crect\kq +2\eq^{-4}\kq)(1+\tvnorm{\mutilde}+\tvnorm{\mutilde'})\lipz{f}t\rhoq^t,
	\end{align}
	where we used that $\tvnorm{\mu-\mu'}\le 2$. Now we return to the main expression \eqref{eq:t_diff} and write, using Lemma~\ref{lemma:erg_sker},
	\begin{align}
		\norm{\tz{\thp}^t f(z)-\tz{\thp}^t f(z')}
		&\le (2\kqt+4\crect\kq +2\eq^{-4}\kq)(1+\tvnorm{\mutilde}+\tvnorm{\mutilde'})\lipz{f}t\rhoq^t+ \lipz{f}(1+\tvnorm{\rect{\thp}^{t}(\mu',\mutilde',y_{1:t})})
		\\
        &\quad \times\left(\int \abs{\skertrue^t-\sigma}((x_1,dy_1),(x_{t+1},dy_{t+1}))+\int\abs{\skertrue^t-\sigma}((x_1',y_1'),(dx_{t+1},dy_{t+1}))\right) \\
        &\le  (2\kqt+4\crect\kq +2\eq^{-4}\kq)(1+\tvnorm{\mutilde}+\tvnorm{\mutilde'})\lipz{f}t\rhoq^t
		+2\lipz{f}\crect(\rhoq^t\tvnorm{\mutilde'}+1)(1-\eq)^t \\
        &\le \lipz{f}(1+\tvnorm{\mutilde}+\tvnorm{\mutilde'})\left((2\kqt+4\crect\kq +2\eq^{-4}\kq)\sup_{t'\in \nsetpos} t'\rhoq^{t'/2} +2\crect\right)\rhoq^{t/2},
	\end{align}
	since $\rhoq^{1/2}\ge(1-\eq)$ and $\sup_{t\in \nsetpos} t\rhoq^{t/2} <\infty$. Letting $\cq\eqdef(2\kqt+4\crect\kq +2\eq^{-4}\kq)\sup_{t'\in \nsetpos} t'\rhoq^{t'/2} +2 \crect$ we finally obtain
	\begin{equation}
		\norm{\tz{\thp}^t f(z)-\tz{\thp}^t f(z')}\le \cq\lipz{f}(\tvnorm{\mutilde}+\tvnorm{\mutilde'}+1)\rhoq^{t/2}.
	\end{equation}
	Now, to prove the second claim, we note that for $t\in \nsetpos$,
	\begin{multline}
		\norm{\tz{\thp}^{t+1}f(z)-\tz{\thp}^{t}f(z)}\le\int\norm{\tz{\thp}^{t}f(z')-\tz{\thp}^{t}f(z)}\,\tz{\thp}(z,dz')
		\le \cq\lipz{f}(\tvnorm{\mutilde}+\int\tvnorm{\mutilde'}\,\tz{\thp}(z,dz')+1)\rhoq^{t/2}
		\\=
		\cq\lipz{f}\left(\tvnorm{\mutilde}+\tvnorm{\rect{\thp}(\mu,\mutilde, y_1)}+1\right)\rhoq^{t/2}.
	\end{multline}
	By Lemma~\ref{lemma:bound_tang}
	\begin{equation}
		\tvnorm{\rect{\thp}(\mu,\mutilde, y_1)}
	\le \crect\left(\tvnorm{\mutilde}+1\right), 
	\end{equation}
	which implies that 
	\begin{equation}
		\norm{\tz{\thp}^{t+1}f(z)-\tz{\thp}^{t}f(z)}\le 2\cq\lipz{f}\crect( \tvnorm{\mutilde}+1)\rhoq^{t/2}.
	\end{equation}
	We now define, for $f\in\lip{\alg{Z}}$, the kernel 
	\begin{equation}\label{eq:def_flim}
		\tconst{\thp}f(z)=\delta_zf+\sum_{t=0}^{\infty}\left(\tz{\thp}^{t+1}f(z)-\tz{\thp}^{t}f(z)\right),
	\end{equation}
	for which 
	\begin{multline}
		\norm{\tz{\thp}^{t}f(z)-\tconst{\thp}f(z)}= \norm{-\sum_{s=t}^{\infty}\left(\tz{\thp}^{s+1}f(z)-\tz{\thp}^{s}f(z)\right)}\le \sum_{s=t}^{\infty}\norm{\tz{\thp}^{s+1}f(z)-\tz{\thp}^{s}f(z)}
		\\\le 2\cq\lipz{f}\crect(\tvnorm{\mutilde}+1)\sum_{s=t}^{\infty}\rhoq^{s/2}=\cqt\lipz{f}(\tvnorm{\mutilde}+1)\rhoq^{t/2}, 
	\end{multline}
	where we have denoted $\cqt\eqdef(1-\rhoq^{1/2})^{-1} 2\crect\cq.$ Finally, it remains to prove that $\tconst{\thp}f(z)$ is constant in $z$. For this purpose, write 
	\begin{align}
		\norm{\tconst{\thp}f(z)-\tconst{\thp}f(z')}
		&\le\inf_{t\in \nsetpos}\{ \norm{\tz{\thp}^t f(z)-\tz{\thp}^t f(z')}+\norm{\tconst{\thp}f(z)-\tz{\thp}^{t}f(z)}+\norm{\tz{\thp}^{t}f(z')-\tconst{\thp}f(z')}\}
		\\
        &\le \lipz{f}(\cq(\tvnorm{\mutilde}+\tvnorm{\mutilde'}+1)+\cqt(\tvnorm{\mutilde}+1)+\cqt(\tvnorm{\mutilde'}+1))\inf_{t\in \nsetpos} \rhoq^{t/2} \\
        &=0,
	\end{align}
	which implies that there exists $f_\thp\in\rset^p$ such that $\tconst{\thp}f(z)=f_\thp$ for all $z\in \set{Z}$.
\end{proof}

Using the previous result, we now establish a strong law of large numbers (LLN) for the given Markov chain.
\begin{proposition}
\label{prop:slln}
	Let Assumptions~\ref{assum:eq}, \ref{assum:bg}, \ref{assum:ssm_unspec}, and \ref{assum:init} hold. Then for every $\thp\in\parspace$ and $f\in\lip{\alg{Z}}$, 
	\begin{equation}
	\lim_{t\to \infty}\frac{1}{T}\sum_{t=0}^{T-1}	\E\left[f(\z{t}{\thp})\right]= f_\thp
	\end{equation}
	and, $\prob$-a.s., 
	\begin{equation}
		\lim_{t\to \infty}\frac{1}{T}\sum_{t=0}^{T-1}f(\z{t}{\thp})=f_\thp. 
	\end{equation}
\end{proposition}
\begin{proof}
	We proceed as in \citet{breiman:1960}. First, note that for all $(t,t_0)\in\nset^2$, $\E\left[f(\z{t_0+t}{\thp})\middle\vert \z{t_0}{\thp}\right]= \tz{\thp}^{t} f(\z{t_0}{\thp})$. By  Lemma~\ref{lemma:bound_tang},
	\begin{multline}
		\norm{\tz{\thp}^{t}f(\z{t_0}{\thp})}\le \int \norm{f(z)}\,\tz{\thp}^{t}(\z{t_0}{\thp},dz)\le \lipz{f}\int (1+\tvnorm{\mutilde}) \,\tz{\thp}^{t}(\z{t_0}{\thp},dz)
		\\= \lipz{f}\int \left(1+\tvnorm{\rect{\thp}^{t_0+t}(\filt{\thp}{0},\tang{\thp}{0},(Y_{1:t_0+1},y_{t+2:t_0+t}))}\right)\,\skertrue((X_{t_0+1}, Y_{t_0+1}),(dx_{t_0+2},dy_{t_0+2}))
		\\\times \prod_{s=2}^{t} \skertrue((x_{t_0+s}, y_{t_0+s}),(dx_{t_0+s+1},dy_{t_0+s+1}))\le \lipz{f}(1+\ctang)<\infty. 
	\end{multline}
	Now, by Proposition~\ref{prop:ergod}, for all $\thp\in\parspace$ and $(s_0,t,T)\in\nset^3$,
	\begin{multline}\label{eq:conv_exp}
		\norm{\frac{1}{T}\sum_{s=0}^{T-1}\tz{\thp}^{s+s_0} f(\z{t}{\thp})-f_\thp}\le \frac{1}{T}\sum_{s=0}^{T-1} \norm{\tz{\thp}^{s+s_0} f(\z{t}{\thp})-f_\thp}
		\\\le\frac{1}{T}\cqt\lipz{f} \left(\tvnorm{\tang{\thp}{t}}+1\right)\rhoq^{s_0/2}\sum_{s=0}^{T-1}\rhoq^{s/2}\le \frac{1}{T}\cqt\lipz{f}(\ctang+1)\rhoq^{s_0/2}\frac{1-\rhoq^{T/2}}{1-\rhoq^{1/2}}, 
	\end{multline}
	where the right-hand-side tends to zero as $T$ tends to infinity. Here we used again Lemma~\ref{lemma:bound_tang} to bound the total variation of $\tang{\thp}{t}$. This implies that 
	\begin{equation}
		\lim_{T \to \infty} \frac{1}{T}\sum_{s=0}^{T-1}\E\left[f(\z{s}{\thp})\right]=\lim_{T \to \infty}\E\left[\frac{1}{T}\sum_{s=0}^{T-1}\tz{\thp}^s f( \z{0}{\thp})\right] = f_\thp
	\end{equation}
	uniformly in $\thp$. Now, define, for $0\le s<t$,
	\begin{equation}
		\mds{t}{s}=\tz{\thp}^{s}f(\z{t-s}{\thp})-\tz{\thp}^{s+1}f(\z{t-s-1}{\thp})
	\end{equation}
	and $\mds{t}{s}=0$ for $s\ge t$; then note that
	\begin{multline}
		\E\left[\mds{t}{s}\mid \mds{t-1}{s},\dots, \mds{1}{s}\right]
		=\E\left[\E\left[\mds{t}{s}\mid \z{t-s-1}{\thp},\z{t-s-2}{\thp},\dots,\z{0}{\thp}\right]\mid \mds{t-1}{s},\dots, \mds{1}{s}\right]
		\\=\E\left[\E\left[\mds{t}{s}\mid \z{t-s-1}{\thp}\right]\mid \mds{t-1}{s},\dots, \mds{1}{s}\right]=0,
	\end{multline}
	where we used, first, the tower property, second, that $(\z{t}{\thp})_{t\in\nset}$ is a Markov chain, and, third, the fact that 
	\begin{equation}
		\E\left[\mds{t}{s}\mid \z{t-s-1}{\thp}\right]=\tz{\thp}\tz{\thp}^{s}f(\z{t-s-1}{\thp})-\tz{\thp}^{s+1}f(\z{t-s-1}{\thp})=0.
	\end{equation}
	This implies that $(\mds{t}{s})_{t\in\nset}$ is a martingale difference sequence, which, since  
	\begin{equation}
		\E\left[\lVert\mds{t}{s}\rVert\right]\le\E\left[\norm{\tz{\thp}^{s}f(\z{t-s}{\thp})}+\norm{\tz{\thp}^{s+1}f(\z{t-s-1}{\thp})}\right]\le 2\lipz{f}(1+\ctang),
	\end{equation}
	is uniformly integrable. Then by Lemma~\ref{lemma:martdiffseq}, for all $s\in\nset$, 
	\begin{equation}\label{eq:conv_mds}
		\lim_{T\to\infty}\frac{1}{T}\sum_{t=0}^{T-1}\mds{t}{s}=0, \quad\prob\text{-a.s.}
	\end{equation}
	 Now write
	\begin{equation}
		f(\z{t}{\thp})-\tz{\thp}^{s+1}f(\z{t-s-1}{\thp})=\sum_{s'=0}^{s}\tz{\thp}^{s'}f(\z{t-s'}{\thp})-\tz{\thp}^{s'+1}f(\z{t-s'-1}{\thp})=\sum_{s'=0}^{s}\mds{t}{s'},
	\end{equation}
	so that for every $s\in\nset$, 
	\begin{align}
		\lefteqn{\norm{\frac{1}{T}\sum_{t=0}^{T-1}f(\z{t}{\thp})-\frac{1}{T}\sum_{t=0}^{T-1}\tz{\thp}^{s+1}f(\z{t}{\thp})}} \hspace{10mm}
		\\
        &=\norm{\frac{1}{T}\sum_{t=0}^{s}f(\z{t}{\thp})+\frac{1}{T}\sum_{t=s+1}^{T-1}f(\z{t}{\thp})-\frac{1}{T}\sum_{t=s+1}^{T-1}\tz{\thp}^{s+1}f(\z{t-s-1}{\thp})-\frac{1}{T}\sum_{t=T}^{T+s}\tz{\thp}^{s+1}f(\z{t-s-1}{\thp})}
		\\
        &\le \frac{1}{T}\sum_{t=0}^{s}\norm{f(\z{t}{\thp})}+\sum_{s'=0}^{s}\norm{\frac{1}{T}\sum_{t=s+1}^{T-1}\mds{t}{s'}}+\frac{1}{T}\sum_{t=T}^{T+s}\norm{\tz{\thp}^{s+1}f(\z{t-s-1}{\thp})}
		\\
        &\le \frac{s+1}{T}\lipz{f}(1+\ctang)+\sum_{s'=0}^{s}\norm{\frac{1}{T}\sum_{t=s+1}^{T-1}\mds{t}{s'}}+\frac{s+1}{T}\lipz{f}(1+\ctang).
	\end{align}
	Letting $T\to \infty$ and using \eqref{eq:conv_mds}, we obtain
	\begin{equation}
		\norm{\frac{1}{T}\sum_{t=0}^{T-1}f(\z{t}{\thp})-\frac{1}{T}\sum_{t=0}^{T-1}\tz{\thp}^{s+1}f(\z{t}{\thp})}\to 0, \quad\prob\text{-a.s.} 
	\end{equation}
	Since the previous limit holds for every $s\in\nset$, it holds for the average of $S\in\nsetpos$ elements, \ie,
	\begin{equation}
		\lim_{T\to\infty}\norm{\frac{1}{T}\sum_{t=0}^{T-1}f(\z{t}{\thp})-\frac{1}{T}\sum_{t=0}^{T-1}\frac{1}{S}\sum_{s=0}^{S-1}\tz{\thp}^{s+1}f(\z{t}{\thp})}\to 0, \quad\prob\text{-a.s.}
	\end{equation}
	Finally, for every $S\in\nsetpos$, we may write
	\begin{equation}
		\norm{\frac{1}{T}\sum_{t=0}^{T-1}f(\z{t}{\thp})-f_\thp}\le \norm{\frac{1}{T}\sum_{t=0}^{T-1}f(\z{t}{\thp})-\frac{1}{T}\sum_{t=0}^{T-1}\frac{1}{S}\sum_{s=0}^{S-1}\tz{\thp}^{s+1}f(\z{t}{\thp})}+\frac{1}{T}\sum_{t=0}^{T-1}\norm{\frac{1}{S}\sum_{s=0}^{S-1}\tz{\thp}^{s+1}f(\z{t}{\thp})-f_\thp},
	\end{equation}
	and by \eqref{eq:conv_exp}, we can, for every $\varepsilon>0$, chose $S$ so large that the right-hand term is smaller than $\varepsilon$. This proves that, $\prob$-a.s.,
	\begin{equation}
		\lim_{T\to\infty}\frac{1}{T}\sum_{t=0}^{T-1}f(\z{t}{\thp})=f_\thp.
	\end{equation}
	
\end{proof}

\subsection{Existence of the \colbo}\label{subsec:mean_field}
In this section we establish the existence of the \colbo, the contrast function, and their gradients. We begin by letting $z=(x,y,\mu,\mutilde)$ and redefining $\vfunc{\thp},\grad{\thp},\likfunc{\thp},$ and $\gradlik{\thp}$ as measurable functions on $\set{Z}$:
\begin{align}
	\vfunc{\thp}(z)&=\vfunc{\thp}(y,\mu)
	\eqdef\E_{(\mu\tensprod\indmeas)^{\tensprod \M}}\left[\log \left(\frac{1}{\M}\sum_{i=1}^{\M}\wgtfunc{\thp}(X^i,y,\auxrv^i)\right)\right], 
	\\\grad{\thp}(z)&=	\grad{\thp}(y,\mu,\mutilde)
	\eqdef
	\E_{(\mu\tensprod\indmeas)^{\tensprod \M}}\left[\frac{\sum_{i=1}^{\M}\nabla_\thp\wgtfunc{\thp}(X^i,y,\auxrv^i)}{\sum_{i=1}^{\M}\wgtfunc{\thp}(X^i,y,\auxrv^i)}\right]
	\\&\hspace{3cm}+\M\int \E_{(\mu\tensprod\indmeas)^{\tensprod (\M-1)}}\left[\log \left(\frac{1}{\M}\wgtfunc{\thp}(x,y,\auxrvb)+\frac{1}{\M}\sum_{i=1}^{\M-1}\wgtfunc{\thp}(X^i,y,\auxrv^i)\right)\right]\,\indmeas(d\auxrvb)\,\mutilde(dx), 
	\\\likfunc{\thp}(z)&=\likfunc{\thp}(y,\mu)\eqdef\log\int \emdens{\thp}( x',y)\,\hidker{\thp}( x,dx')\,\mu(dx),\label{eq:deflikfunc}
	\\\gradlik{\thp}(z)&=\gradlik{\thp}(y,\mu,\mutilde)
	\eqdef\frac{\int\nabla_{\thp} \{\emdens{\thp}(x',y)\hiddens{\thp}(x,x')\}\,\refm(dx')\,\mu(dx)+\int\emdens{\thp}(x',y)\,\hidker{\thp}(x,dx')\,\mutilde(dx)}{\int \emdens{\thp}(x',y)\,\hidker{\thp}(x,dx')\,\mu(dx)}.\label{eq:defgradlik}
\end{align}

In the following we establish strong law of large numbers for path averages of $\vfunc{\thp},\grad{\thp},\likfunc{\thp},$ and $\gradlik{\thp}$.
This result follows directly from Proposition~\ref{prop:slln} if we are able to show that the functions defined above are in $\lip{\alg{Z}}$, which requires some additional assumptions listed below. 
\begin{assumption}\label{assum:bounds_r_grad}
	The constants $\eq,\bg$ already defined in Assumptions~\ref{assum:eq}--\ref{assum:bg} satisfy the following additional properties: for all $\thp \in \parspace$, $(x,x')\in\set{X}^2$, $y\in\set{Y}$, and $\auxrvb\in\set{E}$, 
	$$
	\eq\le\propdens{\thp}(x,x',y)\le\eq^{-1},
	$$
	$$
	\max \left\{\norm{\nabla_\thp \hiddens{\thp}(x,\repfunc{\thp}(x,y,\auxrvb))},\norm{\nabla_\thp \emdens{\thp}(\repfunc{\thp}(x,y,\auxrvb),y)},\norm{\nabla_{\thp}\propdens{\thp}(x,\repfunc{\thp}(x,y,\auxrvb),y)}\right\}\le\bg.
	$$
\end{assumption}
The following lemma extends these uniform bounds to the weight function.
\begin{lemma}\label{lemma:boundw}
	Let Assumptions \ref{assum:eq}--\ref{assum:bg} and \ref{assum:bounds_r_grad} hold. Then $\eq^3\le\wgtfunc{\thp}(x,y,\auxrvb)\le\eq^{-3}$ and there exists $\bgw\in[1,\infty)$ such that or all $\thp \in \parspace$, $x\in\set{X}$, $y\in\set{Y}$, and $\auxrvb\in\set{E}$, 
	\begin{equation}\norm{\nabla_{\thp}\wgtfunc{\thp}(x,y,\auxrvb)}\le\bgw. 
	\end{equation}
\end{lemma}
\begin{proof}
	The first bound follows immediately from the definition of $\wgtfunc{\thp}$ and the assumed bounds on $\hiddens{\thp}$, $\emdens{\thp}$, and $\propdens{\thp}$. For the latter we write,
	\begin{multline}
		\norm{\nabla_{\thp}\wgtfunc{\thp}(x,y,\auxrvb)}
		\le\frac{\emdens{\thp}(\repfunc{\thp}(x,y,\auxrvb),y)\norm{\nabla_\thp\hiddens{\thp}(x,\repfunc{\thp}(x,y,\auxrvb))}+\hiddens{\thp}(x,\repfunc{\thp}(x,y,\auxrvb))\norm{\nabla_{\thp}\emdens{\thp}(\repfunc{\thp}(x,y,\auxrvb),y)}}{\propdens{\thp}(x,\repfunc{\thp}(x,y,\auxrvb),y)}
		\\+\frac{\hiddens{\thp}(x,\repfunc{\thp}(x,y,\auxrvb))\emdens{\thp}(\repfunc{\thp}(x,y,\auxrvb),y)}{\propdens{\thp}(x,\repfunc{\thp}(x,y,\auxrvb),y)^2}\norm{\nabla_{\thp}\propdens{\thp}(x,\repfunc{\thp}(x,y,\auxrvb),y)}\le 2\eq^{-2}\bg+\eq^{-4}\bg\eqqcolon\bgw.
	\end{multline}
\end{proof}
We are now ready to prove that $\vfunc{\thp},\grad{\thp},\likfunc{\thp},$ and $\gradlik{\thp}$ are all in $\lip{\alg{Z}}$.
\begin{lemma}\label{lemma:func_in _lipz}
	Let Assumptions \ref{assum:eq}, \ref{assum:bg}, and \ref{assum:bounds_r_grad} hold. Then for every $\thp\in\parspace$ and $\M\ge 2$ it holds that $\vfunc{\thp},\grad{\thp},\likfunc{\thp},$ and $\gradlik{\thp}$ are all in $\lip{\alg{Z}}$.
\end{lemma}
\begin{proof}
	We begin with $\vfunc{\thp}$, noting that
	\begin{equation}
		\abs{\vfunc{\thp}(z)}\le\E_{(\mu\tensprod\indmeas)^{\tensprod \M}}\left[\abs{\log \left(\frac{1}{\M}\sum_{i=1}^{\M}\wgtfunc{\thp}(X^i,y,\auxrv^i)\right)}\right]\le \log \eq^{-3}\le \log\eq^{-3}(1+\tvnorm{\mutilde}),
	\end{equation}
	since by Lemma~\ref{lemma:boundw}, $\eq^3\le \M^{-1}\sum_{i=1}^{\M}\wgtfunc{\thp}(X^i,y,\auxrv^i)\le\eq^{-3}$. To check condition (ii) of Definition~\ref{def:lipz}, we write
	\begin{align}
		\abs{\vfunc{\thp}(x,y,\mu,\mutilde)-\vfunc{\thp}(x,y,\mu',\mutilde')}
		&\le\int\abs{\log \left(\frac{1}{\M}\sum_{i=1}^{\M}\wgtfunc{\thp}(x^i,y,\auxrvb^i)\right)}\indmeas^{\tensprod\M}(d\auxrvb^{1:\M})\abs{\mu^{\tensprod\M}-\mu'^{\tensprod\M}}(dx^{1:\M})
		\\
        &\le \log\eq^{-3}\int\sum_{i=1}^\M\abs{\mu-\mu'}(dx^i)\prod_{j=1}^{i-1}\mu(dx^j)\prod_{j=i+1}^{\M}\mu(dx^j)
		\\
        &\le \M \log\eq^{-3}\tvnorm{\mu-\mu'}.
	\end{align}
	We continue with $\grad{\thp}$, focusing first on the first term of its definition. Using twice Lemma~\ref{lemma:boundw},
	\begin{equation}\norm{\E_{(\mu\tensprod\indmeas)^{\tensprod \M}}\left[\frac{\sum_{i=1}^{\M}\nabla_\thp\wgtfunc{\thp}(X^i,y,\auxrv^i)}{\sum_{i=1}^{\M}\wgtfunc{\thp}(X^i,y,\auxrv^i)}\right]}\le \E_{(\mu\tensprod\indmeas)^{\tensprod \M}}\left[\frac{\sum_{i=1}^{\M}\norm{\nabla_\thp\wgtfunc{\thp}(X^i,y,\auxrv^i)}}{\sum_{i=1}^{\M}\wgtfunc{\thp}(X^i,y,\auxrv^i)}\right]\le \eq^{-3}\bgw
	\end{equation}
	and  
	\begin{multline}\norm{\E_{(\mu\tensprod\indmeas)^{\tensprod \M}}\left[\frac{\sum_{i=1}^{\M}\nabla_\thp\wgtfunc{\thp}(X^i,y,\auxrv^i)}{\sum_{i=1}^{\M}\wgtfunc{\thp}(X^i,y,\auxrv^i)}\right]-\E_{(\mu'\tensprod\indmeas)^{\tensprod \M}}\left[\frac{\sum_{i=1}^{\M}\nabla_\thp\wgtfunc{\thp}(X^i,y,\auxrv^i)}{\sum_{i=1}^{\M}\wgtfunc{\thp}(X^i,y,\auxrv^i)}\right]}
		\\\le \int \frac{\sum_{i=1}^{\M}\norm{\nabla_\thp\wgtfunc{\thp}(x^i,y,\auxrvb^i)}}{\sum_{i=1}^{\M}\wgtfunc{\thp}(x^i,y,\auxrvb^i)}\prod_{i=1}^{\M}\indmeas(d\auxrvb^i)\sum_{j=1}^{\M}\abs{\mu-\mu'}(dx^j)\prod_{i'=1}^{j-1}\mu(dx^{i'})\prod_{i''=j+1}^{\M}\mu'(dx^{i''})
		\le \M \eq^{-3}\bgw\tvnorm{\mu-\mu'}.
	\end{multline}
	For the second term of $\grad{\thp}$ we note that, since $\int \mutilde(dx)=0$, for $\M\ge2$, 
	\begin{multline}
		\M\int \E_{(\mu\tensprod\indmeas)^{\tensprod (\M-1)}}\left[\log \left(\frac{1}{\M}\wgtfunc{\thp}(x,y,\auxrvb)+\frac{1}{\M}\sum_{i=1}^{\M-1}\wgtfunc{\thp}(X^i,y,\auxrv^i)\right)\right]\,\indmeas(d\auxrvb)\,\mutilde(dx)
		\\=	\M\int \E_{(\mu\tensprod\indmeas)^{\tensprod (\M-1)}}\left[\log \left(1+\frac{\wgtfunc{\thp}(x,y,\auxrvb)}{\sum_{i=1}^{\M-1}\wgtfunc{\thp}(X^i,y,\auxrv^i)}\right)\right]\,\indmeas(d\auxrvb)\,\mutilde(dx)
		\\+\M \E_{(\mu\tensprod\indmeas)^{\tensprod (\M-1)}}\left[\log \left(\frac{1}{\M}\sum_{i=1}^{\M-1}\wgtfunc{\thp}(X^i,y,\auxrv^i)\right)\right]\int\indmeas(d\auxrvb)\int\mutilde(dx),
	\end{multline}
	where the second term vanishes.  Then taking the norm and using again Lemma~\ref{lemma:boundw}, we obtain 
	\begin{multline}
		\norm{\M\int \E_{(\mu\tensprod\indmeas)^{\tensprod (\M-1)}}\left[\log \left(1+\frac{\wgtfunc{\thp}(x,y,\auxrvb)}{\sum_{i=1}^{\M-1}\wgtfunc{\thp}(X^i,y,\auxrv^i)}\right)\right]\,\indmeas(d\auxrvb)\,\mutilde(dx)}
		\\\le \M\log\left(\left(1+\frac{\eq^{-6}}{\M-1}\right)\right)\tvnorm{\mutilde}\le \frac{\M\eq^{-6}}{\M-1}\tvnorm{\mutilde}\le 2\eq^{-6}\tvnorm{\mutilde}.
	\end{multline}
	It remains to prove that part (ii) of Definition~\ref{def:lipz} is satisfied. For this purpose, write 
	\begin{multline}
		\left \| \M\int \E_{(\mu\tensprod\indmeas)^{\tensprod (\M-1)}}\left[\log \left(\frac{1}{\M}\wgtfunc{\thp}(x,y,\auxrvb)+\frac{1}{\M}\sum_{i=1}^{\M-1}\wgtfunc{\thp}(X^i,y,\auxrv^i)\right)\right]\,\indmeas(d\auxrvb)\,\mutilde(dx)\right.
		\\\left.-\M\int \E_{(\mu'\tensprod\indmeas)^{\tensprod (\M-1)}}\left[\log \left(\frac{1}{\M}\wgtfunc{\thp}(x,y,\auxrvb)+\frac{1}{\M}\sum_{i=1}^{\M-1}\wgtfunc{\thp}(X^i,y,\auxrv^i)\right)\right]\,\indmeas(d\auxrvb)\,\mutilde'(dx) \right\|
		\\\le \norm{\M\int \E_{(\mu\tensprod\indmeas)^{\tensprod (\M-1)}}\left[\log \left(1+\frac{\wgtfunc{\thp}(x,y,\auxrvb)}{\sum_{i=1}^{\M-1}\wgtfunc{\thp}(X^i,y,\auxrv^i)}\right)\right]\,\indmeas(d\auxrvb)\,(\mutilde-\mutilde')(dx)}
		\\+\int \left|\M\E_{(\mu\tensprod\indmeas)^{\tensprod (\M-1)}}\left[\log \left(1+\frac{\wgtfunc{\thp}(x,y,\auxrvb)}{\sum_{i=1}^{\M-1}\wgtfunc{\thp}(X^i,y,\auxrv^i)}\right)\right]\right.
		\\\left.-\M\E_{(\mu'\tensprod\indmeas)^{\tensprod (\M-1)}}\left[\log \left(1+\frac{\wgtfunc{\thp}(x,y,\auxrvb)}{\sum_{i=1}^{\M-1}\wgtfunc{\thp}(X^i,y,\auxrv^i)}\right)\right]\right|\,\indmeas(d\auxrvb)\,\abs{\mutilde'}(dx)
		\\\le 2\eq^{-6}\tvnorm{\mutilde-\mutilde'}+2\eq^{-6}\M \tvnorm{\mu-\mu'}\tvnorm{\mutilde'}.
	\end{multline}
	Finally, we have
	\begin{equation}\label{eq:boundH}
		\norm{\grad{\thp}(x, y,  \mu,\mutilde)}\le \eq^{-3}\bgw +2\eq^{-6}\tvnorm{\mutilde}\le 2\bgw\eq^{-6}(1+\tvnorm{\mutilde})
	\end{equation}
	and
	\begin{align}
		\norm{\grad{\thp}(x, y,  \mu,\mutilde)-\grad{\thp}(x, y,  \mu',\mutilde')}
		&\le  \M \eq^{-3}\bgw\tvnorm{\mu-\mu'} +2\eq^{-6}\tvnorm{\mutilde-\mutilde'}+2\eq^{-6}\M \tvnorm{\mu-\mu'}\tvnorm{\mutilde}
		\\&\le 2\eq^{-6}\tvnorm{\mutilde-\mutilde'}+2\M\eq^{-6}\bgw (1+\tvnorm{\mutilde'})\tvnorm{\mu-\mu'}
		\\&\le 2\M\eq^{-6}\bgw\left(\tvnorm{\mutilde-\mutilde'}+(1+\tvnorm{\mutilde}+\tvnorm{\mutilde'})\tvnorm{\mu-\mu'}\right).
	\end{align}
	For $\likfunc{\thp}$, we observe that $\abs{\likfunc{\thp}(x,y,\mu,\mutilde)}\le \log\eq^{-1}\le\log \eq^{-1}(1+\tvnorm{\mutilde})$ and 
	\begin{equation}
		\abs{\likfunc{\thp}(x,y,\mu,\mutilde)-\likfunc{\thp}(x,y,\mu',\mutilde')}
		\le \eq^{-1}\int \emdens{\thp}( x',y)\,\hidker{\thp}( x,dx')\,\abs{\mu-\mu'}(dx)\le \eq^{-2}\tvnorm{\mu-\mu'}.
	\end{equation}
	It remains to show that $\gradlik{\thp}\in\lipz{\alg{Z}}$. Indeed,  
	\begin{multline}
		\norm{\gradlik{\thp}(x,y,\mu,\mutilde)}
		\le\left(\int \emdens{\thp}(x',y)\,\hidker{\thp}(x,dx')\,\mu(dx)\right)^{-1}\left(\int\emdens{\thp}(x',y)\,\hidker{\thp}(x,dx')\,\abs{\mutilde}(dx)\right.
		\\\left.+\int \left(\emdens{\thp}(x',y)\norm{\nabla_\thp\hiddens{\thp}(x,x')}+\hiddens{\thp}(x,x')\norm{\nabla_\thp\emdens{\thp}(x',y)}\right)\,\refm(dx')\,\mu(dx)\right)
		\\\le \eq^{-1}\left(\eq^{-1}\tvnorm{\mutilde}+\eq^{-1}\bg+\bg\right)\le 2\eq^{-2}\bg(1+\tvnorm{\mutilde})
	\end{multline}
	and
	\begin{multline}
		\norm{\gradlik{\thp}(x,y,\mu,\mutilde)-\gradlik{\thp}(x,y,\mu',\mutilde')}
		\le \frac{\int\norm{\nabla_{\thp} \left(\emdens{\thp}(x',y)\hiddens{\thp}(x,x')\right)}\,\refm(dx')\,\abs{\mu-\mu'}(dx)}{\int \emdens{\thp}(x',y)\,\hidker{\thp}(x,dx')\,\mu(dx)}
		\\+\frac{\int\norm{\nabla_{\thp} \left(\emdens{\thp}(x',y)\hiddens{\thp}(x,x')\right)}\,\refm(dx')\,\mu'(dx)\int\emdens{\thp}(x',y)\,\hidker{\thp}(x,dx')\,\abs{\mu'-\mu}(dx)}{\int \emdens{\thp}(x',y)\,\hidker{\thp}(x,dx')\,\mu'(dx)\int \emdens{\thp}(x',y)\,\hidker{\thp}(x,dx')\,\mu(dx)}
		\\+ \frac{\int\emdens{\thp}(x',y)\,\hidker{\thp}(x,dx')\,\abs{\mutilde-\mutilde'}(dx)}{\int \emdens{\thp}(x',y)\,\hidker{\thp}(x,dx')\,\mu(dx)}
		+\frac{\int\emdens{\thp}(x',y)\,\hidker{\thp}(x,dx')\,\abs{\mutilde'}(dx)\int\emdens{\thp}(x',y)\,\hidker{\thp}(x,dx')\,\abs{\mu'-\mu}(dx)}{\int \emdens{\thp}(x',y)\,\hidker{\thp}(x,dx')\,\mu'(dx)\int \emdens{\thp}(x',y)\,\hidker{\thp}(x,dx')\,\mu(dx)}
		\\\le2\eq^{-2}\bg \tvnorm{\mu-\mu'}+2\eq^{-4}\bg\tvnorm{\mu-\mu'}+ \eq^{-2}\tvnorm{\mutilde-\mutilde'}+\eq^{-4}\tvnorm{\mutilde'}\tvnorm{\mu-\mu'}
		\\\le 2\eq^{-4}\bg\left(\tvnorm{\mutilde-\mutilde'}+(1+\tvnorm{\mutilde}+\tvnorm{\mutilde'})\tvnorm{\mu-\mu'}\right).
	\end{multline}
	This completes the proof. 
\end{proof}	

	We are finally ready to prove the existence of the contrast function $\contr(\thp)$ and the COLBO $\lyap(\thp)$ as well as their gradients as $\prob$-a.s. limits.
	\begin{proposition}\label{prop:all_conv_slln}
		Let Assumptions \ref{assum:eq}, \ref{assum:bg}, \ref{assum:ssm_unspec}, \ref{assum:bounds_r_grad}, and \ref{assum:init} hold. Then there exist two real-valued differentiable functions $\lyap$ and $\contr$ on $\parspace$ such that for every $\thp\in\parspace$ and $\M\ge2$, $\prob$-a.s., 
		\begin{align}
			&\lim_{T \to \infty}\frac{1}{T}\sum_{t=0}^{T-1}\vfunc{\thp}(\z{t}{\thp})=\contr^\M(\thp),& &\lim_{T \to \infty}\frac{1}{T}\sum_{t=0}^{T-1}\grad{\thp}(\z{t}{\thp})=\nabla_\thp\contr^\M(\thp),
			\\
        &\lim_{T \to \infty}\frac{1}{T}\sum_{t=0}^{T-1}\likfunc{\thp}(\z{t}{\thp})=\contr(\thp),& &\lim_{T \to \infty}\frac{1}{T}\sum_{t=0}^{T-1}\gradlik{\thp}(\z{t}{\thp})=\nabla_\thp\contr(\thp).
		\end{align}
		Moreover, the same limits hold when the terms of each sum are replaced by their expectations. 
		\end{proposition}
	\begin{proof}
		The limits follow from Proposition~\ref{prop:ergod} and Proposition~\ref{prop:slln}, respectively, since all  functions are in $\lip{\alg{Z}}$, as shown in Lemma~\ref{lemma:func_in _lipz}. It remains to show that the limits for $\grad{\thp}$ and $\gradlik{\thp}$ are the gradients of $\lyap(\thp)$ and $\contr(\thp)$, respectively. Indeed, we have
		\begin{equation}\label{eq:grad_vfunc}
			\frac{1}{T}\sum_{t=0}^{T-1}\nabla_\thp\E[\vfunc{\thp}(\z{t}{\thp})]=\frac{1}{T}\sum_{t=0}^{T-1}\nabla_\thp\E[\tz{\thp}^t\vfunc{\thp}(\z{0}{\thp})]
			=\frac{1}{T}\sum_{t=0}^{T-1}\E[\tz{\thp}^t\grad{\thp}(\z{0}{\thp})]=\frac{1}{T}\sum_{t=0}^{T-1}\E[\grad{\thp}(\z{t}{\thp})],
		\end{equation}
		which converges uniformly in $\thp$ as $T\to \infty$ by Proposition~\ref{prop:slln}. The second equality follows from
		\begin{multline}
			\nabla_\thp\E[\tz{\thp}^t\vfunc{\thp}(\z{0}{\thp})]= \nabla_{\thp}\int \vfunc{\thp}(x_{t+1},y_{t+1}, \recf{\thp}^{t}(\filt{\thp}{0},y_{1:t}),\rect{\thp}^{t}(\filt{\thp}{0},\tang{\thp}{0},y_{1:t}))
			\\\times\xinit(dx_0) \, \emker{}(x_0,dy_0)\prod_{s=0}^{t} \skertrue((x_{s}, y_{s}),(dx_{s+1},dy_{s+1}))
			\\=\int \grad{\thp}(x_{t+1},y_{t+1}, \recf{\thp}^{t}(\filt{\thp}{0},y_{1:t}),\rect{\thp}^{t}(\filt{\thp}{0},\tang{\thp}{0},y_{1:t}))\,\xinit(dx_0) \, \emker{}(x_0,dy_0)\prod_{s=0}^{t} \skertrue((x_{s}, y_{s}),(dx_{s+1},dy_{s+1}))
			\\=\E[\tz{\thp}^t\grad{\thp}(\z{0}{\thp})].
		\end{multline}
		Then, since
		\begin{equation}
			\lim_{T\to\infty}\frac{1}{T}\sum_{t=0}^{T-1}\E[\vfunc{\thp}(\z{t}{\thp})]=\contr^\M(\thp)
		\end{equation}
		uniformly in $\thp$ by Proposition~\ref{prop:slln}, the uniform convergence theorem states that the limit of \eqref{eq:grad_vfunc} is $\nabla_\thp\lyap(\thp)$. The same argument can be applied to $\gradlik{\thp}$, which concludes the proof.
	\end{proof}
	
	\subsection{Bias study}\label{subsec:bias}
	In this section we study the bias of the {\colbo} w.r.t. the contrast function and between the gradient functions $\grad{\thp}$ and $\gradlik{\thp}$.
	\begin{proposition}\label{prop:biascontr}
		Let Assumptions \ref{assum:eq}, \ref{assum:bg}, \ref{assum:ssm_unspec}, \ref{assum:bounds_r_grad}, and \ref{assum:init} hold. 
		Then for all $\thp\in\parspace$ and $\M\in\nsetpos$, $\contr(\thp)\ge \contr^{\M+1}(\thp)\ge \contr^\M(\thp)$. Moreover, 
		\begin{equation}
		\contr(\thp)-\lyap(\thp)=\frac{\eq^{-8}}{2\M}+\mathcal{O}\left(\frac{1}{\M^2}\right).
		\end{equation}
	\end{proposition}
	\begin{proof}
	To prove this we use the results of \citet{nowozin:2018}. First, we rewrite $\likfunc{\thp}$ as
	\begin{multline}
		\likfunc{\thp}(z)=\likfunc{\thp}(x,y,\mu,\mutilde)=\log\int \emdens{\thp}( x',y)\hiddens{\thp}( x,x')\,\refm(dx')\,\mu(dx)
		\\=\log\int \wgtfunc{\thp}( x,y,\auxrvb)\,\indmeas(d\auxrvb)\,\mu(dx)=\log\E_{\mu\tensprod\indmeas}\left[\wgtfunc{\thp}( X,y,\auxrv)\right],
	\end{multline}
	so that we may express
	\begin{equation}
		\likfunc{\thp}(z)-\vfunc{\thp}(z)=\log\E_{\mu\tensprod\indmeas}\left[\wgtfunc{\thp}( X,y,\auxrv)\right]-\E_{(\mu\tensprod\indmeas)^{\tensprod \M}}\left[\log \left(\frac{1}{\M}\sum_{i=1}^{\M}\wgtfunc{\thp}(X^i,y,\auxrv^i)\right)\right].
	\end{equation}
	Now, since all the moments of $\wgtfunc{\thp}( X,y,\auxrv)$ are finite, being $\wgtfunc{\thp}$ bounded, then we may apply \citet[Proposition~1]{nowozin:2018}. Thus, for all $\thp\in\parspace$ and $z\in\set{Z}$, 
	\begin{equation}
		\likfunc{\thp}(z)-\vfunc{\thp}(z)=\frac{\E_{\mu\tensprod\indmeas}[\left(\wgtfunc{\thp}(X,y,\auxrv)\right)^2]}{2\M\left(\E_{\mu\tensprod\indmeas}\left[\wgtfunc{\thp}( X,y,\auxrv)\right]\right)^2}+\mathcal{O}\left(\frac{1}{\M^2}\right)\le \frac{\eq^{-8}}{2\M}+\mathcal{O}\left(\frac{1}{\M^2}\right).
	\end{equation}
	Then, by Proposition~\ref{prop:all_conv_slln} we have
	\begin{equation}\label{eq:bias_vfunc}
		\contr(\thp)-\lyap(\thp)
		= \lim_{t\to\infty}\frac{1}{t}\sum_{s=0}^{t-1}\E\left[\likfunc{\thp}(\z{s}{\thp})-\vfunc{\thp}(\z{s}{\thp})\right]
		\le \frac{\eq^{-8}}{2\M}+\mathcal{O}\left(\frac{1}{\M^2}\right).
	\end{equation}
	The monotonicity in $\M$ follows from \citep[Theorem~1]{burda:2016}, which states that 
	\begin{equation}
		\likfunc{\thp}(z)\ge \likfunc{\thp}^{\M+1}(z)\ge \vfunc{\thp}(z)
	\end{equation}
	for all $z\in\set{Z}$ and $\M\in\nsetpos$; then the claim is proven by expressing $\contr(\thp)-\lyap(\thp)$ and $\contr^{\M+1}(\thp)-\lyap(\thp)$ in the same way as in \eqref{eq:bias_vfunc} and using that nonnegative sequences have nonnegative limits. 
 This concludes the proof.
	\end{proof}
	
	\begin{theorem}\label{thm:bias_grad_gen}
		Let Assumptions \ref{assum:eq}, \ref{assum:bg}, and \ref{assum:bounds_r_grad} hold. 
		Then, for all $\thp\in\parspace$, $z\in\set{Z}$, and $\M> \eq^{-6}+1$, 
	\begin{equation}
		\norm{\grad{\thp}(z)-\gradlik{\thp}(z)}\le \frac{\bq_1}{\M}		+\frac{(\bq_2\eq^{-3}+2\eq^{-6})\tvnorm{\mutilde}}{\M-1}+2\eq^{-6}\tvnorm{\mutilde}\sum_{j=2}^{\infty} \frac{1}{j+1}\left(\frac{ \eq^{-6}}{M-1}\right)^{j},
	\end{equation}
	where $\bq_1$ and $\bq_2$ are the constants provided in Lemma~\ref{lemma:bias_gen}, depending only on $\eq$ and $\bgw$.
	\end{theorem}
	\begin{proof}
		By Lemma~\ref{lemma:taylor_log} and Fubini's Theorem we may write
		\begin{multline}\label{eq:grad_pelbo}
			\grad{\thp}(z)=\grad{\thp}(y,\mu,\mutilde)
			\\=\E_{(\mu\tensprod\indmeas)^{\tensprod \M}}\left[\frac{\sum_{i=1}^{\M}\nabla_\thp\wgtfunc{\thp}(X^i,y,\auxrv^i)/\M}{\sum_{i=1}^{\M}\wgtfunc{\thp}(X^i,y,\auxrv^i)/\M}\right]
			+\int \E_{(\mu\tensprod\indmeas)^{\tensprod (\M-1)}}\left[\frac{\M\wgtfunc{\thp}(x,y,\auxrvb)}{\sum_{i=1}^{\M-1}\wgtfunc{\thp}(X^i,y,\auxrv^i)}\right]\indmeas(d\auxrvb)\,\mutilde(dx)
			\\+\int \M\sum_{j=2}^{\infty}\frac{(-1)^{j+1}}{j}\E_{(\mu\tensprod\indmeas)^{\tensprod (\M-1)}}\left[\left(\frac{\wgtfunc{\thp}(x,y,\auxrvb)}{\sum_{i=1}^{\M-1}\wgtfunc{\thp}(X^i,y,\auxrv^i)}\right)^j\right]\indmeas(d\auxrvb)\,\mutilde(dx). 
		\end{multline}
		In order to be able to compare $\gradlik{\thp}$ to $\grad{\thp}$, we express the former as 
		\begin{multline}\label{eq:grad_predlik}
			\gradlik{\thp}(z)=\gradlik{\thp}(x,y,\mu,\mutilde)
			=\frac{\nabla_\thp\int\emdens{\thp}(x',y)\,\hidker{\thp}(x,dx')\,\mu(dx)+\int\emdens{\thp}(x',y)\,\hidker{\thp}(x,dx')\,\mutilde(dx)}{\int \emdens{\thp}(x',y)\,\hidker{\thp}(x,dx')\,\mu(dx)}
			\\=\frac{\int\nabla_\thp\wgtfunc{\thp}(x,y,\auxrvb)\,\indmeas(d\auxrvb)\,\mu(dx)+\int\wgtfunc{\thp}(x,y,\auxrvb)\,\indmeas(d\auxrvb)\,\mutilde(dx)}{\int \wgtfunc{\thp}(x,y,\auxrvb)\,\indmeas(d\auxrvb)\,\mu(dx)}
			\\=\frac{\E_{\mu\tensprod\indmeas}\left[\nabla_\thp\wgtfunc{\thp}(X,y,\auxrv)\right]}{\E_{\mu\tensprod\indmeas}\left[\wgtfunc{\thp}(X,y,\auxrv)\right]}+\frac{\int  \wgtfunc{\thp}(x,y,\auxrvb)\indmeas(d\auxrvb)\,\mutilde(dx)}{\E_{\mu\tensprod\indmeas}\left[\wgtfunc{\thp}(X,y,\auxrv)\right]}.
		\end{multline}
		To bound the bias, we begin with the first terms of \eqref{eq:grad_pelbo} and \eqref{eq:grad_predlik}, and by Lemma~\ref{lemma:bias_gen}, 
		\begin{equation}\label{eq:bias1}
			\norm{\frac{\E_{\mu\tensprod\indmeas}\left[\nabla_\thp\wgtfunc{\thp}(X,y,\auxrv)\right]}{\E_{\mu\tensprod\indmeas}\left[\wgtfunc{\thp}(X,y,\auxrv)\right]}-\E_{(\mu\tensprod\indmeas)^{\tensprod \M}}\left[\frac{\sum_{i=1}^{\M}\nabla_\thp\wgtfunc{\thp}(X^i,y,\auxrv^i)/\M}{\sum_{i=1}^{\M}\wgtfunc{\thp}(X^i,y,\auxrv^i)/\M}\right]}\le\frac{\bq_1}{\M}.
		\end{equation}
		The second term of \eqref{eq:grad_pelbo} is compared to the second term of \eqref{eq:grad_predlik}, and again by Lemma~\ref{lemma:bias_gen},
		\begin{multline}\label{eq:bias3}
			\norm{\int \left(\E_{(\mu\tensprod\indmeas)^{\tensprod (\M-1)}}\left[\frac{\wgtfunc{\thp}(x,y,\auxrvb)\M}{\sum_{i=1}^{\M-1}\wgtfunc{\thp}(X^i,y,\auxrv^i)}\right]-\frac{\wgtfunc{\thp}(x,y,\auxrvb)}{\E_{\mu\tensprod\indmeas}\left[\wgtfunc{\thp}(X,y,\auxrv)\right]}\right)\,\indmeas(d\auxrvb)\,\mutilde(dx)}
			\\\le\abs{\frac{1}{\E_{\mu\tensprod\indmeas}\left[\wgtfunc{\thp}(X,y,\auxrv)\right]}-\E_{(\mu\tensprod\indmeas)^{\tensprod (\M-1)}}\left[\frac{\M}{\sum_{i=1}^{\M-1}\wgtfunc{\thp}(X^i,y,\auxrv^i)}\right]}
			\norm{\int\wgtfunc{\thp}(x,y,\auxrvb)\,\indmeas(d\auxrvb)\,\mutilde(dx)}
			\\\le \frac{\bq_2\eq^{-3}\tvnorm{\mutilde}}{\M-1}.
		\end{multline}
			Finally, we turn to the last term of \eqref{eq:grad_pelbo} and proceed like 
		\begin{multline}\label{eq:bias2}
			\norm{\int\M\sum_{j=2}^{\infty}\frac{(-1)^{j+1}}{j}\E_{(\mu\tensprod\indmeas)^{\tensprod (\M-1)}}\left[\left(\frac{\wgtfunc{\thp}(x,y,\auxrvb)}{\sum_{i=1}^{\M-1}\wgtfunc{\thp}(X^i,y,\auxrv^i)}\right)^j\right]\indmeas(d\auxrvb)\,\mutilde(dx)}
			\\\le\int\abs{\M\sum_{j=2}^{\infty}\frac{(-1)^{j+1}}{j}\E_{(\mu\tensprod\indmeas)^{\tensprod (\M-1)}}\left[\left(\frac{\wgtfunc{\thp}(x,y,\auxrvb)}{\sum_{i=1}^{\M-1}\wgtfunc{\thp}(X^i,y,\auxrv^i)}\right)^j\right]}\indmeas(d\auxrvb)\,\abs{\mutilde}(dx)
			\\\le \tvnorm{\mutilde}\M\sum_{j=2}^{\infty} \frac{1}{j}\left(\frac{ \eq^{-6}}{\M-1}\right)^j\le2\eq^{-6}\tvnorm{\mutilde}\sum_{j=2}^{\infty} \frac{1}{j}\left(\frac{ \eq^{-6}}{M-1}\right)^{j-1}.
		\end{multline}
		Combining \eqref{eq:bias1},  \eqref{eq:bias3}, and \eqref{eq:bias2} we obtain
		\begin{align}
			\norm{\grad{\thp}(z)-\gradlik{\thp}(z)}
			&\le \norm{\E_{(\mu\tensprod\indmeas)^{\tensprod \M}}\left[\frac{\sum_{i=1}^{\M}\nabla_\thp\wgtfunc{\thp}(X^i,y,\auxrv^i)}{\sum_{i=1}^{\M}\wgtfunc{\thp}(X^i,y,\auxrv^i)}\right]-\frac{\E_{\mu\tensprod\indmeas}\left[\nabla_\thp\wgtfunc{\thp}(X,y,\auxrv)\right]}{\E_{\mu\tensprod\indmeas}\left[\wgtfunc{\thp}(X,y,\auxrv)\right]}} \\
			&\quad +\norm{\int \left(\E_{(\mu\tensprod\indmeas)^{\tensprod (\M-1)}}\left[\frac{\M\wgtfunc{\thp}(x,y,\auxrvb)}{\sum_{i=1}^{\M-1}\wgtfunc{\thp}(X^i,y,\auxrv^i)}\right]-\frac{\wgtfunc{\thp}(x,y,\auxrvb)}{\E_{\mu\tensprod\indmeas}\left[\wgtfunc{\thp}(X,y,\auxrv)\right]}\right)\indmeas(d\auxrvb)\,\mutilde(dx)}
			\\
            &\quad +\norm{\int \M\sum_{j=2}^{\infty}\frac{(-1)^{j+1}}{j}\left(\E_{(\mu\tensprod\indmeas)^{\tensprod (\M-1)}}\left[\frac{\wgtfunc{\thp}(x,y,\auxrvb)}{\sum_{i=1}^{\M-1}\wgtfunc{\thp}(X^i,y,\auxrv^i)}\right]\right)^j\,\indmeas(d\auxrvb)\,\mutilde(dx)}
			\\
            &\le \frac{\bq_1}{\M}
			+\frac{(\bq_2\eq^{-3}+2\eq^{-6})\tvnorm{\mutilde}}{\M-1}+2\eq^{-6}\tvnorm{\mutilde}\sum_{j=2}^{\infty} \frac{1}{j+1}\left(\frac{ \eq^{-6}}{M-1}\right)^{j},
		\end{align}
		which concludes the proof.
	\end{proof}
	\begin{corollary}\label{cor:biasmf}
		Let Assumptions \ref{assum:eq}, \ref{assum:bg}, \ref{assum:ssm_unspec}, \ref{assum:bounds_r_grad}, and \ref{assum:init} hold. 
		Then there exists a function $\biasfunc:\nset_{\ge2}\to \rsetpos$
    such that
		$$\biasfunc(\M)=\mathcal{O}\left(\frac{1}{\M-1}\right),$$
		and which satisfies, for all $\M\ge2$, $\thp\in\parspace$, $t\in \nset$, and $y_{1:t + 1} \in \set{Y}^{t + 1}$, 
		\begin{equation}
				\norm{\grad{\thp}(y_{t+1},\filt{\thp}{t},\tang{\thp}{t})-\gradlik{\thp}(y_{t+1},\filt{\thp}{t},\tang{\thp}{t})}\le\biasfunc(\M).
			\end{equation}
		Consequently, it also holds that
		\begin{equation}
				\norm{\nabla_\thp\contr(\thp)-\nabla_\thp\lyap(\thp)}\le\biasfunc(\M).
			\end{equation}
	\end{corollary}
	\begin{proof}
		By Theorem~\ref{thm:bias_grad_gen}, if $\M> \eq^{-6}+1$,
		\begin{equation}
			\norm{\grad{\thp}(y_{t+1},\filt{\thp}{t},\tang{\thp}{t})-\gradlik{\thp}(y_{t+1},\filt{\thp}{t},\tang{\thp}{t})}
			\le \frac{\bq_1}{\M}		+\frac{(\bq_2\eq^{-3}+2\eq^{-6})\tvnorm{\tang{\thp}{t}}}{\M-1}+2\eq^{-6}\tvnorm{\tang{\thp}{t}}\sum_{j=2}^{\infty} \frac{1}{j+1}\left(\frac{ \eq^{-6}}{M-1}\right)^{j},
		\end{equation}
        otherwise, inspecting the proof of Lemma~\ref{lemma:func_in _lipz},
        \begin{equation}
			\norm{\grad{\thp}(y_{t+1},\filt{\thp}{t},\tang{\thp}{t})-\gradlik{\thp}(y_{t+1},\filt{\thp}{t},\tang{\thp}{t})}
			\le (2\bgw\eq^{-6}+2\eq^{-2}\bg)(1+\tvnorm{\tang{\thp}{t}}).
		\end{equation}
		Note that by Lemma~\ref{lemma:bound_tang}, $\tvnorm{\tang{\thp}{t}}\le\ctang$. Hence, we may define
		\begin{multline}
			\biasfunc(\M)\eqdef 	(\bq_1+\bq_2+2)\ctang\frac{\eq^{-6}}{\M-1}+2\eq^{-6}\ctang\sum_{j=2}^{\infty} \frac{1}{j+1}\left(\frac{ \eq^{-6}}{M-1}\right)^{j}\1{\{\M>\eq^{-6}+1\}}
   \\+(2\bgw\eq^{-6}+2\eq^{-2}\bg)(1+\ctang)\1{\{\M\le\eq^{-6}+1\}},
		\end{multline}
		which is clearly $\mathcal{O}(1/(\M-1))$, proving that the first claim holds true. Finally, by applying Proposition~\ref{prop:all_conv_slln} we obtain that 
		\begin{multline}
			\norm{\nabla_{\thp}\contr(\thp)-\nabla_\thp\lyap(\thp)}\le\inf_{t\in \nsetpos}\left\{\norm{\nabla_{\thp}\contr(\thp)-\frac{1}{t}\sum_{s=0}^{t-1}\E\left[\gradlik{\thp}(\z{s}{\thp})\right]}\right.
			\\+ \norm{\frac{1}{t}\sum_{s=0}^{t-1}\E\left[\grad{\thp}(\z{s}{\thp})\right]- \nabla_\thp\lyap(\thp)}
			\left.+\frac{1}{t}\sum_{s=0}^{t-1}\E\left[\norm{\gradlik{\thp}(\z{s}{\thp})-\grad{\thp}(\z{s}{\thp})}\right]\right\}\le \biasfunc(\M),
		\end{multline}
		which concludes the proof.
	\end{proof}

	\subsection{Auxiliary results}\label{subsec:aux}
	\begin{lemma}[Convergence of martingale difference sequences]\label{lemma:martdiffseq}
		Let $(N_t)_{t\in\nsetpos}$ be a martingale difference sequence, \ie an adapted stochastic process such that $(N_t)_{t\in\nsetpos}$ is uniformly integrable and $\E[N_{t+1}\mid N_{t},\dots,N_0]=0$ for all $t\in\nset$. Then
		\begin{equation}
			\lim_{T\to\infty}\frac{1}{T}\sum_{t=1}^{T}N_t=0,\quad\prob\text{-a.s.}
		\end{equation}
	\end{lemma}
	\begin{proof}
		For every $t \in \nsetpos$, let $M_t\eqdef\sum_{s=1}^{t}N_s/s$. It is easy to check that $(M_t)_{t\in\nsetpos}$ is a martingale. Then by the \emph{martingale convergence theorem} it has a limit $\prob$-almost surely. Finally, by Kronecker's lemma, since $\sum_{t=1}^{T}N_t/t$ converges $\prob$-a.s. as $T\to \infty$, then 
		\begin{equation}
			\lim_{T\to\infty}\frac{1}{T}\sum_{t=1}^{T}N_t= 0,\quad \prob\text{-a.s.}
		\end{equation}
	\end{proof}
	\begin{lemma}\label{lemma:taylor_log}
		Let Assumptions \ref{assum:eq}--\ref{assum:bg}, and \ref{assum:bounds_r_grad} hold and let $\M>\eq^{-6}+1$. Then for all $z=(x,y,\mu,\mutilde)\in\set{Z}$ it holds
		\begin{multline}
			\int\M\log \left(1+\frac{\wgtfunc{\thp}(x,y,\auxrvb)}{\sum_{i=1}^{\M-1}\wgtfunc{\thp}(x^i,y,\auxrvb^i)}\right)\,\indmeas(d\auxrvb)\,\mutilde(dx)
			\\= \int  \M\sum_{k=1}^{\infty}\frac{(-1)^{k+1}}{k}\left(\frac{\wgtfunc{\thp}(x,y,\auxrvb)}{\sum_{i=1}^{\M-1}\wgtfunc{\thp}(x^i,y,\auxrvb^i)}\right)^k\,\indmeas(d\auxrvb)\,\mutilde(dx).
		\end{multline}
	\end{lemma}
	\begin{proof}
		The claim follows immediately from a Taylor expansion, $\log(1+x)$ being analytic when $\abs{x}<1$ and
		\begin{equation}
			\frac{\wgtfunc{\thp}(x,y,\auxrvb)}{\sum_{i=1}^{\M-1}\wgtfunc{\thp}(x^i,y,\auxrvb^i)}\le\frac{ \eq^{-6}}{\M-1}<1
		\end{equation}
        for $\M>\eq^{-6}+1$.
	\end{proof}
	\begin{lemma}\label{lemma:bias_gen}
		Let Assumptions \ref{assum:eq}, \ref{assum:bg}, and \ref{assum:bounds_r_grad} hold. Then for every $\mu\in\probmeas{\alg{X}}$, $y\in\set{Y}$, and $\M\ge2$ there exist constants $\bq_1 > 0$ and $\bq_2>0$, depending only on $\eq$ and $\bgw$, such that
		\begin{equation}
			\norm{\E_{(\mu\tensprod\indmeas)^{\tensprod \M}}\left[\frac{\sum_{i=1}^{\M}\nabla_\thp\wgtfunc{\thp}(X^i,y,\auxrv^i)}{\sum_{i=1}^{\M}\wgtfunc{\thp}(X^i,y,\auxrv^i)}\right]-\frac{\E_{\mu\tensprod\indmeas}\left[\nabla_\thp\wgtfunc{\thp}(X,y,\auxrv)\right]}{\E_{\mu\tensprod\indmeas}\left[\wgtfunc{\thp}(X,y,\auxrv)\right]}}\le\frac{\bq_1}{\M}
		\end{equation}
		and 
		\begin{equation}
			\abs{\frac{1}{\E_{\mu\tensprod\indmeas}\left[\wgtfunc{\thp}(X,y,\auxrv)\right]}-\E_{(\mu\tensprod\indmeas)^{\tensprod (\M-1)}}\left[\frac{\M}{\sum_{i=1}^{\M-1}\wgtfunc{\thp}(X^i,y,\auxrv^i)}\right]}\le \frac{\bq_2}{\M-1}.
		\end{equation}
	\end{lemma}
	\begin{proof}
		We apply the identity $a/b-c/d=1/d(a/b(d-b)+(a-c))$ and write
		\begin{multline}
			\frac{\E_{\mu\tensprod\indmeas}\left[\nabla_\thp\wgtfunc{\thp}(X,y,\auxrv)\right]}{\E_{\mu\tensprod\indmeas}\left[\wgtfunc{\thp}(X,y,\auxrv)\right]}-\E_{(\mu\tensprod\indmeas)^{\tensprod \M}}\left[\frac{\sum_{i=1}^{\M}\nabla_\thp\wgtfunc{\thp}(X^i,y,\auxrv^i)/\M}{\sum_{i=1}^{\M}\wgtfunc{\thp}(X^i,y,\auxrv^i)/\M}\right]
			\\= \E_{(\mu\tensprod\indmeas)^{\tensprod \M}}\left[\frac{1}{\sum_{i=1}^{\M}\wgtfunc{\thp}(X^i,y,\auxrv^i)/\M}\left(\frac{\E_{\mu\tensprod\indmeas}\left[\nabla_{\thp}\wgtfunc{\thp}(X,y,\auxrv)\right]}{\E_{\mu\tensprod\indmeas}\left[\wgtfunc{\thp}(X,y,\auxrv)\right]}\right.\right.
			\\\left.\left.\left(\frac{1}{\M}\sum_{i=1}^{\M}\wgtfunc{\thp}(X^i,y,\auxrv^i)-\E_{\mu\tensprod\indmeas}\left[\wgtfunc{\thp}(X,y,\auxrv)\right]\right)
			+\left(\E_{\mu\tensprod\indmeas}\left[\nabla_\thp\wgtfunc{\thp}(X,y,\auxrv)\right]-\frac{1}{\M}\sum_{i=1}^{\M}\nabla_\thp\wgtfunc{\thp}(X^i,y,\auxrv^i)\right)\right)\right]
			\\=\E_{(\mu\tensprod\indmeas)^{\tensprod \M}}\left[\left(\frac{1}{\sum_{i=1}^{\M}\wgtfunc{\thp}(X^i,y,\auxrv^i)/\M}-\frac{1}{\E_{\mu\tensprod\indmeas}\left[\wgtfunc{\thp}(X,y,\auxrv)\right]}\right)\right.
			\\\times\left(\frac{\E_{\mu\tensprod\indmeas}\left[\nabla_{\thp}\wgtfunc{\thp}(X,y,\auxrv)\right]}{\E_{\mu\tensprod\indmeas}\left[\wgtfunc{\thp}(X,y,\auxrv)\right]}\frac{1}{\M}\sum_{i=1}^{\M}\left(\wgtfunc{\thp}(X^i,y,\auxrv^i)-\E_{\mu\tensprod\indmeas}\left[\wgtfunc{\thp}(X,y,\auxrv)\right]\right)\right.
			\\\left.\left.+\frac{1}{\M}\sum_{i=1}^{\M}\left(\E_{\mu\tensprod\indmeas}\left[\nabla_{\thp}\wgtfunc{\thp}(X,y,\auxrv)\right]-\nabla_{\thp}\wgtfunc{\thp}(X^i,y,\auxrv^i)\right)\right)\right],
		\end{multline}
		where we used $\E_{\mu\tensprod\indmeas}\left[\wgtfunc{\thp}(X,y,\auxrv)-\E_{\mu\tensprod\indmeas}\left[\wgtfunc{\thp}(X,y,\auxrv)\right]\right]=0$ and $\E_{\mu\tensprod\indmeas}\left[\nabla_{\thp}\wgtfunc{\thp}(X,y,\auxrv)-\E_{\mu\tensprod\indmeas}\left[\nabla_{\thp}\wgtfunc{\thp}(X,y,\auxrv)\right]\right]=0$ in the last equality. Using Lemma~\ref{lemma:boundw} and the Cauchy--Schwarz inequality,
		\begin{align}
			&\hspace{-1cm}\norm{\frac{\E_{\mu\tensprod\indmeas}\left[\nabla_\thp\wgtfunc{\thp}(X,y,\auxrv)\right]}{\E_{\mu\tensprod\indmeas}\left[\wgtfunc{\thp}(X,y,\auxrv)\right]}-\E_{(\mu\tensprod\indmeas)^{\tensprod \M}}\left[\frac{\sum_{i=1}^{\M}\nabla_\thp\wgtfunc{\thp}(X^i,y,\auxrv^i)/\M}{\sum_{i=1}^{\M}\wgtfunc{\thp}(X^i,y,\auxrv^i)/\M}\right]}
			\\&\le\E_{(\mu\tensprod\indmeas)^{\tensprod \M}}\left[\frac{\abs{\frac{1}{\M}\sum_{i=1}^{\M}\wgtfunc{\thp}(X^i,y,\auxrv^i)-\E_{\mu\tensprod\indmeas}\left[\wgtfunc{\thp}(X,y,\auxrv)\right]}}{	\E_{\mu\tensprod\indmeas}\left[\wgtfunc{\thp}(X,y,\auxrv)\right]\sum_{i=1}^{\M}\wgtfunc{\thp}(X^i,y,\auxrv^i)/\M}\right.
			\\&\hspace{1cm}\times\left.\frac{\E_{\mu\tensprod\indmeas}\left[\norm{\nabla_{\thp}\wgtfunc{\thp}(X,y,\auxrv)}\right]}{\E_{\mu\tensprod\indmeas}\left[\wgtfunc{\thp}(X,y,\auxrv)\right]}\abs{\frac{1}{\M}\sum_{i=1}^{\M}\wgtfunc{\thp}(X^i,y,\auxrv^i)-\E_{\mu\tensprod\indmeas}\left[\wgtfunc{\thp}(X,y,\auxrv)\right]}\right]
			\\&\hspace{1cm}+\E_{(\mu\tensprod\indmeas)^{\tensprod \M}}\left[\frac{\abs{\frac{1}{\M}\sum_{i=1}^{\M}\wgtfunc{\thp}(X^i,y,\auxrv^i)-\E_{\mu\tensprod\indmeas}\left[\wgtfunc{\thp}(X,y,\auxrv)\right]}}{	\E_{\mu\tensprod\indmeas}\left[\wgtfunc{\thp}(X,y,\auxrv)\right]\sum_{i=1}^{\M}\wgtfunc{\thp}(X^i,y,\auxrv^i)/\M}\right.
			\\&\hspace{1cm}\times\left.			\norm{\frac{1}{\M}\sum_{i=1}^{\M}\nabla_{\thp}\wgtfunc{\thp}(X^i,y,\auxrv^i)-\E_{\mu\tensprod\indmeas}\left[\nabla_{\thp}\wgtfunc{\thp}(X,y,\auxrv)\right]}\right]
			\\&\le \eq^{-9}\bgw\E_{(\mu\tensprod\indmeas)^{\tensprod \M}}\left[\left(\frac{1}{\M}\sum_{i=1}^{\M}\wgtfunc{\thp}(X^i,y,\auxrv^i)-\E_{\mu\tensprod\indmeas}\left[\wgtfunc{\thp}(X,y,\auxrv)\right]\right)^2\right]
			\\&\hspace{1cm}+\eq^{-6}\E_{(\mu\tensprod\indmeas)^{\tensprod \M}}\left[\left(\frac{1}{\M}\sum_{i=1}^{\M}\wgtfunc{\thp}(X^i,y,\auxrv^i)-\E_{\mu\tensprod\indmeas}\left[\wgtfunc{\thp}(X,y,\auxrv)\right]\right)^2\right]^{1/2}
			\\&\hspace{1cm}\times\E_{(\mu\tensprod\indmeas)^{\tensprod \M}}\left[\norm{\frac{1}{\M}\sum_{i=1}^{\M}\nabla_{\thp}\wgtfunc{\thp}(X^i,y,\auxrv^i)-\E_{\mu\tensprod\indmeas}\left[\nabla_\thp\wgtfunc{\thp}(X,y,\auxrv)\right]}^2\right]^{1/2}
			\\&\le \eq^{-9}\bgw\frac{1}{\M}\V_{\mu\tensprod\indmeas}(\wgtfunc{\thp}(X,y,\auxrv))+\eq^{-6}\frac{1}{\M}\V_{\mu\tensprod\indmeas}(\wgtfunc{\thp}(X,y,\auxrv))^{1/2}
			\\&\hspace{1cm}\times\E_{\mu\tensprod\indmeas}\left[\norm{\nabla_{\thp}\wgtfunc{\thp}(X,y,\auxrv)-\E_{\mu\tensprod\indmeas}\left[\nabla_\thp\wgtfunc{\thp}(X,y,\auxrv)\right]}^2\right]^{1/2}\le3 \eq^{-9}\bgw\frac{1}{\M}\eqqcolon\frac{\bq_1}{\M}.
		\end{align}
		Similarly, to establish the second claim of the lemma we proceed like
			\begin{align}
			&\hspace{-1cm}\abs{\frac{1}{\E_{\mu\tensprod\indmeas}\left[\wgtfunc{\thp}(X,y,\auxrv)\right]}-\E_{(\mu\tensprod\indmeas)^{\tensprod (\M-1)}}\left[\frac{\M/(\M-1)}{\sum_{i=1}^{\M-1}\wgtfunc{\thp}(X^i,y,\auxrv^i)/(\M-1)}\right]}
			\\&=\E_{(\mu\tensprod\indmeas)^{\tensprod (\M-1)}}\left[\left( \frac{1}{\sum_{i=1}^{\M-1}\wgtfunc{\thp}(X^i,y,\auxrv^i)/(\M-1)}-\frac{1}{\E_{\mu\tensprod\indmeas}\left[\wgtfunc{\thp}(X,y,\auxrv)\right]}\right)\right.
			\\&\hspace{1cm}\left.\times\frac{1}{\E_{\mu\tensprod\indmeas}\left[\wgtfunc{\thp}(X,y,\auxrv)\right]}\left(\frac{1}{\M-1}\sum_{i=1}^{\M-1}\wgtfunc{\thp}(X^i,y,\auxrv^i)-\E_{\mu\tensprod\indmeas}\left[\wgtfunc{\thp}(X,y,\auxrv)\right]\right) \right]
			\\&\hspace{1cm}+\E_{(\mu\tensprod\indmeas)^{\tensprod (\M-1)}}\left[\frac{1-\M/(\M-1)}{\sum_{i=1}^{\M-1}\wgtfunc{\thp}(X^i,y,\auxrv^i)/(\M-1)}\right]
			\\&\le\E_{(\mu\tensprod\indmeas)^{\tensprod (\M-1)}}\left[ \frac{\abs{\E_{\mu\tensprod\indmeas}\left[\wgtfunc{\thp}(X,y,\auxrv)\right]-\sum_{i=1}^{\M-1}\wgtfunc{\thp}(X^i,y,\auxrv^i)/(\M-1)}}{\E_{\mu\tensprod\indmeas}\left[\wgtfunc{\thp}(X,y,\auxrv)\right]\sum_{i=1}^{\M-1}\wgtfunc{\thp}(X^i,y,\auxrv^i)/(\M-1)}\right.
			\\&\hspace{1cm}\left.\times\frac{1}{\E_{\mu\tensprod\indmeas}\left[\wgtfunc{\thp}(X,y,\auxrv)\right]}\abs{\frac{1}{\M-1}\sum_{i=1}^{\M-1}\wgtfunc{\thp}(X^i,y,\auxrv^i)-\E_{\mu\tensprod\indmeas}\left[\wgtfunc{\thp}(X,y,\auxrv)\right]}\right] 
			\\&\hspace{1cm}+\E_{(\mu\tensprod\indmeas)^{\tensprod (\M-1)}}\left[\frac{1/(\M-1)}{\sum_{i=1}^{\M-1}\wgtfunc{\thp}(X^i,y,\auxrv^i)/(\M-1)}\right]
			\\&\le \eq^{-9}\E_{(\mu\tensprod\indmeas)^{\tensprod (\M-1)}}\left[\left(\frac{1}{\M-1}\sum_{i=1}^{\M-1}\wgtfunc{\thp}(X^i,y,\auxrv^i)-\E_{\mu\tensprod\indmeas}\left[\wgtfunc{\thp}(X,y,\auxrv)\right]\right)^2\right]
			\\&\hspace{1cm}+\eq^{-3}\frac{1}{\M-1}\le (\eq^{-15}+\eq^{-3})\frac{1}{\M-1}\eqqcolon\frac{\bq_2}{\M-1}.
		\end{align}
	\end{proof}

\section{Proof of Lemma~\ref{lemma:tfunc}}\label{app:tfunc}

\begin{proof}
	First note that the complete-data score function is additive, \ie it satisfies
	$$
	\nabla_\thp\log p_\thp(x_{0:t},y_{0:t})=\nabla_\thp\log\emdens{\thp}(y_0\mid x_0)+\sum_{s=0}^{t-1}\nabla_\thp\log\hiddens{\thp}(x_{s+1}\mid x_{s})+\nabla_\thp\log\emdens{\thp}(y_{s+1}\mid x_{s+1})
	$$
	for all $t\in\nset$, $x_{0:t}\in\set{X}^{t+1}$ and $y_{0:t}\in\set{Y}^{t+1}$. Now, for $t\in\nsetpos$ we write
	\begin{align}
		\tfunc{\thp}{t+1}(x_{t+1})&=\int \nabla_\thp\log p_\thp(x_{0:t+1},y_{0:t+1})p_\thp(x_{0:t}\mid y_{0:t},x_{t+1})\,dx_{0:t}
		\\&=\int \nabla_\thp\log p_\thp(x_{0:t+1},y_{0:t+1})p_\thp(x_{0:t-1}\mid y_{0:t},x_{t:t+1})p_\thp(x_{t}\mid y_{0:t},x_{t+1})\,dx_{0:t}
		\\&=\iint \nabla_\thp\log p_\thp(x_{0:t},y_{0:t})p_\thp(x_{0:t-1}\mid y_{0:t-1},x_{t})\,dx_{0:t-1}\,p_\thp(x_{t}\mid y_{0:t},x_{t+1})\,dx_t
		\\&\hspace{5mm}+\int\nabla_\thp\log\{\hiddens{\thp}(x_{t+1}\mid x_{t})\emdens{\thp}(y_{t+1}\mid x_{t+1})\}\int p_\thp(x_{0:t-1}\mid y_{0:t-1},x_{t})\,dx_{0:t-1}\,p_\thp(x_{t}\mid y_{0:t},x_{t+1})\,dx_{t}
		\\&=\int\left( \tfunc{\thp}{t}(x_t)+\nabla_\thp\log\{\hiddens{\thp}(x_{t+1}\mid x_{t})\emdens{\thp}(y_{t+1}\mid x_{t+1})\}\right)\frac{\filt{\thp}{t}(x_{t})\hiddens{\thp}(x_{t+1}\mid x_t)}{\int\filt{\thp}{t}(x)\hiddens{\thp}(x_{t+1}\mid x)\,dx}\,dx_{t}.
	\end{align}
	Hence, given $(\tfunc{\thp}{t},\filt{\thp}{t},y_{t+1})$, we define the recursive update
	\begin{equation}
		\rectsmooth{\thp}(\tfunc{\thp}{t},\filt{\thp}{t},y_{t+1})(x_{t+1})
		\eqdef\int\left( \tfunc{\thp}{t}(x_t)+\nabla_\thp\log\{\hiddens{\thp}(x_{t+1}\mid x_{t})\emdens{\thp}(y_{t+1}\mid x_{t+1})\}\right)\frac{\filt{\thp}{t}(x_{t})\hiddens{\thp}(x_{t+1}\mid x_t)}{\int\filt{\thp}{t}(x)\hiddens{\thp}(x_{t+1}\mid x)\,dx}\,dx_{t},
	\end{equation}
	which satisfies $\tfunc{\thp}{t+1}(x_{t+1})=\rectsmooth{\thp}(\tfunc{\thp}{t},\filt{\thp}{t},y_{t+1})(x_{t+1})$ for all $x_{t+1}\in\set{X}$. This proves (i).
	
	To prove (ii), we first note that for every $t\in\nset$ and every measurable function $f:\set{X}\to \rset$, 
	\begin{multline}
		\int f(x_t)\tang{\thp}{t}(x_{t})\,dx_t=\int f(x_t)\nabla_{\thp}\filt{\thp}{t}(x_t)\,dx_t= \int f(x_t) \nabla_\thp\frac{p_\thp(x_{0:t}, y_{0:t})}{p_\thp( y_{0:t})}\, dx_{0:t}
		\\=\int f(x_t) \nabla_\thp \log p_\thp(x_{0:t}, y_{0:t}) p_\thp(x_{0:t}\mid y_{0:t}) \, dx_{0:t}-\frac{\nabla_\thp p_\thp(y_{0:t})}{p_\thp(y_{0:t})} \int f(x_t)\filt{\thp}{t}(x_t)\,dx_t
		\\=\E_{\filt{\thp}{0:t}}[f(X_t)(\nabla_\thp\log p_\thp(X_{0:t}, y_{0:t})-\E_{\filt{\thp}{0:t}}[\nabla_\thp\log p_\thp(X_{0:t}, y_{0:t})])],
	\end{multline}
	which implies 
	\begin{multline}
		\grad{\thp}(y_{t+1},\filt{\thp}{t},\tang{\thp}{t})=\E_{(\filt{\thp}{t})^{\varotimes \M-1}\tensprod\filt{\thp}{0:t}\tensprod\indmeas^{\varotimes \M}}  \left[\frac{\sum_{i=1}^{\M}\nabla_\thp\wgtfunc{\thp}(X^i, y_{t+1}, \auxrv^i)}{\sum_{i'=1}^{\M}\wgtfunc{\thp}(X^{i'}, y_{t+1}, \auxrv^{i'})}\right.
		\\+\left.\M\log\left(\frac{1}{\M}\sum_{i=1}^{\M}\wgtfunc{\thp}(X^i, y_{t+1}, \auxrv^i)\right)\left(\nabla_\thp\log p_\thp(X_{0:t}^\M,y_{0:t})-\E_{\filt{\thp}{0:t}} [\nabla_\thp\log p_\thp(X_{0:t},y_{0:t})]\right)\right],
	\end{multline}
	where $(X^i)_{i=1}^{\M-1}$ are i.i.d. draws from $\filt{\thp}{t}$, the trajectory $X_{0:t}^\M$ is drawn from $\filt{\thp}{0:t}$, and the auxiliary variables $(\auxrv^i)_{i=1}^\M$ are i.i.d. draws from $\indmeas$. Now, by the definition of $\tfunc{\thp}{t}$, we have
	\begin{align}
		\E_{\filt{\thp}{0:t}}\left[f(X_{t})\nabla_\thp\log p_\thp(X_{0:t},y_{0:t})\right]
		&=\int f(x_{t})\nabla_\thp\log p_\thp(x_{0:t},y_{0:t})\filt{\thp}{0:t}(x_{0:t})\,dx_{0:t}
		\\
        &=\int f(x_{t})\nabla_\thp\log p_\thp(x_{0:t},y_{0:t})p_\thp(x_{0:t-1}\mid y_{0:t},x_{t})p_{\thp}(x_{t}\mid y_{0:t})\,dx_{0:t}
		\\
  &=\int f(x_{t})\int\nabla_\thp\log p_\thp(x_{0:t},y_{0:t})p_\thp(x_{0:t-1}\mid y_{0:t-1},x_{t})\,dx_{0:t-1}\,\filt{\thp}{t}(x_{t})\,dx_{t} \\
  &=\E_{\filt{\thp}{t}}[f(X_t)\tfunc{\thp}{t}(X_t)].
	\end{align}
	Thus, it finally holds that 
	\begin{multline}
		\grad{\thp}(y_{t+1},\filt{\thp}{t},\tang{\thp}{t})= \E_{(\filt{\thp}{t}\tensprod\indmeas)^{\varotimes \M}}  \left[\frac{\sum_{i=1}^{\M}\nabla_\thp\wgtfunc{\thp}(X^i, y_{t+1}, \auxrv^i)}{\sum_{i'=1}^{\M}\wgtfunc{\thp}(X^{i'}, y_{t+1}, \auxrv^{i'})}\right.
		\\+\left.\M\log\left(\frac{1}{\M}\sum_{i=1}^{\M}\wgtfunc{\thp}(X^i, y_{t+1}, \auxrv^i)\right)\left(\tfunc{\thp}{t}(X^\M)-\E_{\filt{\thp}{t}}[\tfunc{\thp}{t}(X)]\right)\right]=\gradt{\thp}(y_{t+1},\filt{\thp}{t},\tfunc{\thp}{t}).
	\end{multline}
\end{proof}

\numberwithin{algorithm}{section}
\section{The \texttt{AdaSmooth} algorithm}\label{app:adasmooth}
In this appendix we discuss how to produce recursively by means of SMC methods \citep{gordon:salmond:smith:1993,doucet:defreitas:gordon:2001,chopin:papaspiliopoulos:2020} the sample $\{(\epart{t+1}{i},\tstat{t+1}{i},\wgt{t+1}{i})\}_{i=1}^N$, given $\{(\epart{t}{i},\tstat{t}{i},\wgt{t}{i})\}_{i=1}^N$ together with the new observation $Y_{t+1}$ and the updated parameter $\thp_{t+1}$.
In a standard particle filter, the propagation of the particles is carried out by first resampling, with replacement, the particles $(\epart{t}{i})_{i=1}^N$ proportionally to the weights $(\wgt{t}{i})_{i=1}^N$ (the so-called \emph{selection} step) and then \emph{mutating} the resampled particles using some proposal distribution. Here we let the proposal be  $\propdens{\thp}$, which is progressively adapted over the iterations, but also other choices are possible. Then, after the recalculation of the importance weights, it remains to propagate the terms $(\tstat{t}{i})_{i=1}^N$. Several works on particle-based online additive smoothing have presented strategies for updating these statistics by approximating the recursion $\rectsmooth{\thp}$. In fact, this requires the computation of an expectation with respect to the so-called \emph{backward kernel}, whose density is proportional to $\filt{\thp}{t}(x_{t})\hiddens{\thp}(x_{t+1}\mid x_t)$; see Appendix~\ref{app:tfunc} for details. For each propagated particle $\epart{t+1}{i}$, $i \in \{1,\dots, N\}$, the backward kernel is translated into a categorical distribution with support on the previous particle cloud $(\epart{t}{j})_{j=1}^N$ and probabilities proportional to $\{\wgt{t}{j}\hiddens{\thp}(\epart{t+1}{i}\mid\epart{t}{j})\}_{j=1}^N$. \cite{delmoral:doucet:singh:2010} evaluate these expectations exactly, while the \texttt{PaRIS} algorithm \citep{olsson:westerborn:2017} employs a Monte Carlo approximation based on two samples at least (sufficient to guarantee long-term stability) from $\catdist(\{\wgt{t}{j}\hiddens{\thp}(\epart{t+1}{i}\mid\epart{t}{j})\}_{j=1}^N)$, for each $\epart{t+1}{i}$, $i \in \{1,\dots,N\}$. Both of these approaches have complexity $\mathcal{O}(N^2)$ per time step, due to the computation of the normalising constants, which is not desirable; therefore, the latter implements the backward sampling according to the accept-reject alternative suggested by \citet{douc:garivier:moulines:olsson:2009}, where the complexity of each draw does not depend on $N$. Still, \citet{dau:chopin:2023}, showed that in many realistic cases, the expected time to acceptance is infinite, which forces to have an early stopping rule for the rejection sampler and obtain the remaining draws from the exact distributions. Therefore, since backward sampling remains the bottleneck of this class of methods, the \texttt{AdaSmooth} algorithm  \citep{mastrototaro:olsson:alenlov:2021} proposes to reduce the frequency of this operation by combining a fast but unstable (as the asymptotic variance grows quadratically in time rather than linearly) naive forward smoother with the \texttt{PaRIS}. This is accomplished by applying sparsely the resampling and the backward sampling operations according to a schedule which is governed by two sequences $(\res_t)_{t\in\nset}\in\{0,1\}$ and $(\bs_t)_{t\in\nset}\in\{0,1\}$, indicating whether to perform resampling or backward sampling at each iteration, respectively. These sequences may be deterministic or adapted to the random variables generated by the algorithm, \eg, by monitoring the weights and the particle-path degeneracy, respectively. If the values equal to one appear regularly in the sequences, then the algorithm is proved to be stable \citep[see][for more details]{mastrototaro:olsson:alenlov:2021}. Algorithm~\ref{algo:adasmooth} illustrates the \texttt{AdaSmooth} update in the special case where the additive functional is the complete-data score of the SSM. 

\begin{algorithm}[htb]
	\caption{\texttt{AdaSmooth}}\label{algo:adasmooth}
	\begin{algorithmic}[1]
		\REQUIRE $\{(\epart{t}{i},\tstat{t}{i},\wgt{t}{i})\}_{i=1}^N$, $Y_{t+1}$, $\thp_{t+1}$.
		\FOR{$i \gets 1,\dots, N$}
		\IF{$\res_t$}
		\STATE draw $\I{t+1}{i}\sim\catdist(\{\wgt{t}{j}\}_{j=1}^N)$;
		\ELSE
		\STATE set $\I{t+1}{i}\gets i$;
		\ENDIF
		\STATE draw $\epart{t+1}{i}\sim \propdens{\thp_{t+1}}(\cdot\mid \epart{t}{\I{t+1}{i}},Y_{t+1})$;
		\STATE set $\wgt{t+1}{i}\gets (\wgt{t}{i})^{1-\res_t}\, \dfrac{\hiddens{\thp_{t+1}}(\epart{t+1}{i}\mid \epart{t}{\I{t+1}{i}})\emdens{\thp_{t+1}}(Y_{t+1}\mid \epart{t+1}{i})}{\propdens{\thp_{t+1}}(\epart{t+1}{i}\mid \epart{t}{\I{t+1}{i}},Y_{t+1})}$;
		\IF{$\res_t$ \AND $\bs_t$}
		\STATE draw $J_{t+1}^{i}\sim\catdist(\{\wgt{t}{j}\hiddens{\thp_{t+1}}(\epart{t+1}{i}\mid\epart{t}{j})\}_{j=1}^N)$;
		\STATE set $\begin{aligned}[t]
			\tstat{t+1}{i}\gets \frac{1}{2}&\big(\tstat{t}{\I{t+1}{i}}+\nabla_\thp\log\hiddens{\thp_{t+1}}(\epart{t+1}{i}\mid \epart{t}{\I{t+1}{i}})+\tstat{t}{J_{t+1}^{i}}+\nabla_\thp\log\hiddens{\thp_{t+1}}(\epart{t+1}{i}\mid \epart{t}{J_{t+1}^{i}})\big)
			\\&+\nabla_\thp\log\emdens{\thp_{t+1}}(Y_{t+1}\mid \epart{t+1}{i});
		\end{aligned}$
		\ELSE
		\STATE set $\tstat{t+1}{i}\gets\tstat{t}{\I{t+1}{i}}+\nabla_\thp\log\hiddens{\thp_{t+1}}(\epart{t+1}{i}\mid \epart{t}{\I{t+1}{i}})+\nabla_\thp\log\emdens{\thp_{t+1}}(Y_{t+1}\mid \epart{t+1}{i})$;
		\ENDIF
		\ENDFOR
		\RETURN $\{(\epart{t+1}{i},\tstat{t+1}{i},\wgt{t+1}{i})\}_{i=1}^N$.
	\end{algorithmic}
\end{algorithm}

\section{Additional numerical experiments}
In this appendix, we present additional results from the numerical experiments conducted on the SLAM model to further highlight the advantages of the {\OSIWAE} algorithm over {\RML} and {\OVSMC}. Figure~\ref{fig:slam_landmark} illustrates the Mean Absolute Error (MAE) averaged over 10 runs for the estimated positions of all landmarks. The dashed lines represent the minimum and maximum MAE observed across all runs at each timestep. As discussed in the main paper, we observe that after the initial phase where the proposal distribution is being learned, {\OSIWAE} achieves a lower MAE than both other algorithms. 

\begin{figure}[H]
    \centering   
    \includegraphics[width=1.\columnwidth]{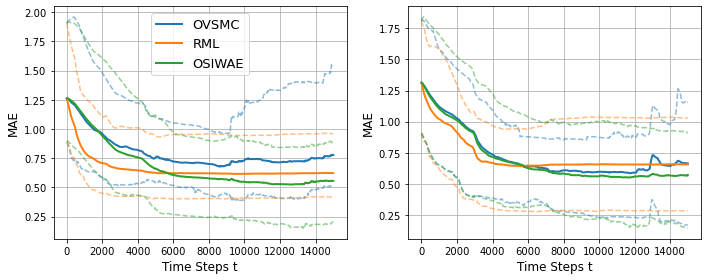}
    \caption{Mean absolute errors (MAEs) averaged over 10 runs for the estimated positions of $L = 8$ landmarks over time using {\OSIWAE}, {\RML}, and {\OVSMC} in a SLAM scenario with motion noise variance \( \sigma_{\text{motion}}^2 = 0.2 \) and observation noise variance \( \sigma_{\text{obs}}^2 = 0.1 \). The dashed lines indicate the minimum and maximum MAE across all runs.
    The proposal distribution \( r_\theta(\cdot \mid x_t, y_{t+1}) \) in both {\OSIWAE} and {\OVSMC} is learned using two distinct neural networks, each with one hidden layer of 12 nodes. All three methods employ \( N = 1000 \) particles, and {\OSIWAE} uses \( M = 1000 \) samples.
    Left panel: All three algorithms are run on the same data without any prior learning. 
    Right panel: A training run is first performed using {\smcsiwae} on a different dataset to learn the proposal distribution; subsequently, all three algorithms are applied to the same data. Each of the 10 runs is performed with the same observations and true landmark positions but with different initial landmark estimates.
    }
    \label{fig:slam_landmark}
\end{figure}

\end{document}